\documentclass{article}

\usepackage[preprint,nonatbib]{neurips_2026}

\usepackage{subcaption}
\usepackage{amsmath}
\usepackage{amssymb}
\usepackage{amsthm}
\usepackage{booktabs}
\usepackage{tikz}
\usetikzlibrary{arrows.meta} 
\usepackage[many]{tcolorbox}
\tcbuselibrary{theorems,skins,breakable}

\newtcbtheorem{Summary}{\bfseries Summary}{
  enhanced,
  breakable,
  sharp corners,
  boxrule=0.4pt,
  colback=gray!2,
  colframe=black!15,
  coltitle=black,
  fonttitle=\bfseries,
  fontupper=\small,
  boxsep=1.2mm,
  left=1.5mm,
  right=1.5mm,
  top=1.2mm,
  bottom=1.2mm,
  drop fuzzy shadow=black!12,
  attach boxed title to top left={
    xshift=2mm,
    yshift=-0.5mm
  },
  boxed title style={
    enhanced,
    sharp corners,
    colback=black!6,
    colframe=black!10,
    boxrule=0pt,
    boxsep=0.8mm,
    left=1mm,
    right=1mm,
    top=0.4mm,
    bottom=0.4mm,
  },
  before skip=8pt,
  after skip=8pt,
}{
  summary
}

\newtheoremstyle{postulatestyle} 
  {3pt} 
  {3pt} 
  {\itshape} 
  {} 
  {\bfseries} 
  {.} 
  {.5em} 
  {} 

\theoremstyle{postulatestyle}
\newtheorem{postulate}{Postulate}
\usepackage{amsthm}

\newtheorem{remark}{Remark}

\newtheorem{theorem}{Theorem}
\newtheorem*{theorem*}{Theorem}

\theoremstyle{definition}
\newtheorem{defn}[theorem]{Definition} 

\newtheoremstyle{lemmastyle} 
  {3pt} 
  {3pt} 
  {\itshape} 
  {} 
  {\bfseries} 
  {.} 
  {.5em} 
  {} 

\theoremstyle{lemmastyle}
\newtheorem{lemma}{Lemma}

\usepackage{graphicx}                           
\usepackage{rotating}                           
\usepackage{pdflscape}                          
\usepackage{verbatim}                           
\usepackage{url}                                
\usepackage{indentfirst}                        
\usepackage{float}                              
\usepackage{breakcites}
\usepackage{amsmath}                            
\usepackage{changepage}                         
\usepackage{cancel}
\usepackage{listings}
\usepackage{wrapfig}
\usepackage{enumitem}                           
\usepackage{array}                              
\usepackage{makecell}                           
\usepackage{amssymb}                            
\usepackage[table]{xcolor}
\usepackage{mathtools}
\usepackage{setspace}
\usepackage{multirow}
\usepackage{multicol}
\usepackage{hhline}
\usepackage{dirtytalk}                          
\usepackage{algpseudocode}                      
\usepackage{algorithm}                          
\usepackage{diagbox}                            
\usepackage{bm}
\usepackage{pifont}

\usepackage{amsthm}

\newtheorem{definition}[theorem]{Definition}

\usepackage[utf8]{inputenc} 
\usepackage[T1]{fontenc}    
\usepackage{url}            
\usepackage{booktabs}       
\usepackage{amsfonts}       
\usepackage{nicefrac}       
\usepackage{microtype}      
\usepackage{xcolor}         

\usepackage[colorlinks=true, linkcolor=blue, citecolor=blue, urlcolor=blue]{hyperref}
\title{Axiomatizing Neural Networks via Pursuit of Subspaces}

\author{
  Mehmet Yamaç$^{1}$ \And
  Mert Duman$^{1}$ \And
  Ugur Akpinar$^{1}$ \And
  Felix Rojas Casadiego$^{1}$ \And
  Serkan Kiranyaz$^{2}$ \And
  Marcel van Gerven$^{3}$ \And
  Moncef Gabbouj$^{1}$ \\
  \\
  $^{1}$Tampere University, Faculty of ITC, Finland \\
  $^{2}$Department of Electrical Engineering, Qatar University, Qatar \\
  $^{3}$Donders Institute, Radboud University, The Netherlands
}

\begin{document}

\maketitle

\begin{abstract}
 While deep neural networks have achieved remarkable success across a wide range of domains, their underlying mechanisms remain poorly understood, and they are often regarded as black boxes. This gap between empirical performance and theoretical understanding poses a challenge analogous to the pre-axiomatic stage of classical geometry. In this work, we introduce the Pursuit of Subspaces (PoS) hypothesis, an axiomatic framework that formulates neural network behavior through a set of geometric postulates. These axioms, together with their derived consequences, provide a unified perspective on representation, computation, and generalization in both shallow and deep architectures. We show that this framework yields geometric explanations for fundamental questions in deep learning, including representation structure, architectural mechanisms, and generalization behavior, offering a principled step toward a coherent theoretical foundation.
\end{abstract}

\section{Introduction}

While deep neural networks have demonstrated remarkable success across a broad spectrum of applications from computer vision and natural language processing to scientific discovery they nonetheless remain fundamentally opaque. Despite their empirical effectiveness, these models are often regarded as "black boxes" because our understanding of the internal principles that govern their learning dynamics and decision‑making processes is still limited. This opacity raises the following critical questions at the intersection of theory and practice:
(i) Why are neural networks so successful~\cite{sejnowski2020unreasonable}? More precisely, what foundational mathematical or statistical properties underlie their superior performance compared to traditional machine‑learning methods? Can we derive rigorous theoretical guarantees that explain their expressive power, optimization behavior, and robustness?
(ii) How do neural networks internally implement computational mechanisms that lead to specific outputs ~\cite{spline1, spline2,prince2023understanding}? Understanding the representational transformations that emerge across layers how abstract concepts form, evolve, and interact remains an active area of research, involving interpretability, mechanistic analysis, and the study of implicit modularity within learned features.
(iii) What principles enable neural networks to generalize effectively to unseen data, even when they possess far more parameters than training samples \cite{zhang2016understanding, belkin2019reconciling}? This surprising ability challenges classical frameworks such as VC dimension and traditional bias–variance trade‑offs, suggesting that new theoretical paradigms are needed to explain the implicit regularization induced by optimization algorithms and data structure.
(iv) When and why do neural networks produce erroneous, inconsistent, or fabricated outputs often referred to as “hallucinations” and how can such failures be detected, mitigated, or prevented? Understanding the conditions that lead to these deviations, including distribution shifts, overconfidence, and miscalibrated uncertainty, is crucial for deploying neural networks in safety‑critical and high‑stakes environments.
Collectively, these questions highlight the gap between the empirical capabilities of deep neural networks and our theoretical understanding of them. Bridging this gap is essential not only for advancing the scientific foundations of artificial intelligence but also for ensuring the reliability, transparency, and responsible use of these increasingly pervasive systems.

Several approaches have been developed to advance deep learning theory. 
\textbf{Spline theory}~\cite{spline1, spline2}
 interprets ReLU networks as continuous piecewise-linear (CPWL) systems, offering so-called \textbf{``global explanations''}. While mathematically elegant, it is limited to ReLU activations and does not generalize to smooth or non-standard operators, with only limited extensions~\cite{goujon2024towards}.  
The \textbf{Folding Hypothesis}~\cite{folding2, hauser2017principles} proposes that layers progressively fold affine regions, but this remains local and experimentally driven.  
The \textbf{Manifold Hypothesis}~\cite{cohen2020separability} suggests that high-dimensional data reside on low-dimensional manifolds, providing representational insight without offering an axiomatic framework.  
\textbf{Geometric Deep Learning (GDL)}~\cite{bronstein2021geometric, cohen2016group, kondor2018generalization, cohen2019gauge}
 exploits symmetries to design equivariant architectures, but its focus is generalization rather than interpretability. {Together, these approaches advance interpretability and architectural design but remain fragmented, lacking an axiomatic foundation that unifies interpretability, geometry, and cognition.

In this work, we introduce the \textbf{Pursuit of Subspaces (PoS)} framework, motivated by our aim to establish an axiomatic language for neural networks. PoS introduces four geometric axioms that describe how deep networks learn compact data representations and thus enable mathematically grounded explanations of generalization, hallucination control, and stability. PoS provides rigorous 
explanations for existing neural architectures by generalizing \textbf{Sparse Representation (SR)} into \textbf{differential 
geometry}, extending single-layer CPWL models to hierarchical curved-space representations observed in 
modern DNNs. While both GDL and PoS adopt differential-geometric perspectives, GDL originated from graph 
theory with a focus on symmetry and efficiency, whereas \textbf{PoS extends SR theory into a complete geometric foundation} for both \textbf{demistfying existing neural networks} and designing novel neural architectures that are \emph{explainable by design}.

The logical flow of the paper is as follows: We develop a geometric theory of deep learning in which representations are explicitly modeled as unions of low-dimensional smooth submanifolds, each equipped with structured projection operators onto these submanifolds. Within this framework, learning naturally induces a transversal decomposition of the corresponding tangent spaces into shared and residual components, leading to a first key insight: \textit{orthogonality and disentanglement are not imposed design choices, but emerge as necessary conditions for stable and unique projection onto learned manifolds.}

This geometric structure provides a unified interpretation of standard architectural mechanisms. 
\textit{Nonlinear activations such as ReLU act as angular selectors on local tangent subspaces}, enabling the transition from a single subspace to a union of subspaces by suppressing cross-subspace components. 
\textit{Residual connections implement annihilation of normal-space components}, refining representations by removing components orthogonal to the learned manifold, and when composed across multiple modules, can yield recursive refinement across layers. 
\textit{Structured perturbations such as masking and noise injection activate latent residual directions}, allowing the network to identify and suppress them, thereby promoting compact and disentangled representations. 
More generally, \textit{attention mechanisms can be interpreted as collaborative residual-removal procedures}, aligning multiple representations through shared geometric structure. 
Together, these mechanisms realize a common principle: selective suppression of residual directions to enforce compact representations.

Building on this foundation, we show that learned transformations generate entire families of manifolds through symmetry, yielding a group-action view in which representations are organized as orbits of a canonical manifold. This leads to a second key insight: \textit{the union-of-submanifolds model is a special case of a symmetry-generated structure, and what appears as manifold reorientation from an external viewpoint is equivalent to coordinate expansion and orthogonalization in latent space}. 

This perspective culminates in our main theoretical result: \textit{deep architectures reduce the sample complexity of learning transformation families from multiplicative to additive scaling by hierarchically composing transformations whose learned action generalizes across components via equivariance}. Finally, we validate these predictions experimentally across multiple regimes: through zero-shot anomaly detection enabled by compact manifold representations, through PoS Former as an architectural realization of structured projection and residual suppression, and through image restoration where learned mappings approximate isometric transformations between corrupted and clean data manifolds. These results collectively demonstrate how geometric principles translate into both architectural design and transformation-aware learning.


\section{Preliminaries}
\subsection{Definitions and Notation}
A \textbf{mapping} between two sets, \(\mathcal{X} \) and \(\mathcal{Y} \) is defined as a rule that assigns a \( \mathbf{y} \in \mathcal{Y}  \) for each \( \mathbf{x} \in \mathcal{X}  \), i.e., \(F\colon \mathcal{X} \to \mathcal{Y}\). \(\mathcal{X} \) and \(\mathcal{Y} \)  are called the \textbf{domain} and \textbf{codomain} of \(F\), respectively. The \textbf{image (or range)} of \(F\), \(\operatorname{im} \mathcal{F}
\), is defined as \(F(\mathcal{X}) = \left\{  y \in \mathcal{Y} \mid \mathbf{y}=  F\left ( \mathcal{X} \right ) \text{for some } \mathbf{x} \in \mathcal{X}  \right\} \subset \mathcal{Y}\). For a map \(F\colon \mathcal{X} \to \mathcal{Y}\) and a subset \(\mathcal{A} \subseteq \mathcal{X}\), the \textbf{restriction} of \(F\) to \(\mathcal{A}\), denoted \(F|_{\mathcal{A}}\), is the map from \(\mathcal{A}\) to \(\mathcal{Y}\) defined by \(F|_{\mathcal{A}}(\mathbf{x}) = F(\mathbf{x})\) for all \(\mathbf{x} \in \mathcal{A}\). A map \(F\colon \mathcal{X} \to \mathcal{Y}\) is called \textbf{injective} (one-to-one) if \(F(\mathbf{x}_1) = F(\mathbf{x}_2)\) implies \(\mathbf{x}_1 = \mathbf{x}_2\), meaning distinct elements in \(\mathcal{X}\) map to distinct elements in \(\mathcal{Y}\). It is called \textbf{surjective} (onto) if, for every \(\mathbf{y} \in \mathcal{Y}\), there exists \(\mathbf{x} \in \mathcal{X}\) such that \(F(\mathbf{x}) = \mathbf{y}\), ensuring that the entire codomain is covered. A map that is both injective and surjective is called \textbf{bijective}, establishing a one-to-one correspondence between \(\mathcal{X}\) and \(\mathcal{Y}\).
We use standard notation for norms, sparsity, and support sets. 
A summary is provided in Appendix~\ref{app:Notation}.

\subsection{Manifolds}

The concept of a \textbf{manifold} extends the familiar idea of a surface to spaces that locally resemble Euclidean space across any dimension. For instance, while surfaces like the skin of a sphere or the surface of a torus are examples of 2-dimensional manifolds, manifolds can also be 1-dimensional (such as curves) or extend into higher dimensions. One-dimensional manifolds consist of only two fundamental types: the line and the circle. In two dimensions, manifolds can be classified based on their number of holes, formally known as genus. The sphere represents genus 0, the torus represents genus 1, and this classification continues in complexity for higher dimensions. Despite becoming more intricate in higher dimensions, the essential notion remains that a manifold is a topological space that resembles a flat space in the vicinity of each point. More formally, a manifold \(\mathcal{M} \subset \mathbb{R}^n\) is defined as a topological space that, in the vicinity of each point, resembles \(\mathbb{R}^k\), i.e., it exhibits a structure that is locally \textbf{homeomorphic} (i.e., a one-to-one, onto map that is continuous and has a continuous inverse) to Euclidean space \(\mathbb{R}^k\), where typically \(n \ge k\) and \(\mathbb{R}^n\) is called the \textbf{ambient space} that \(\mathcal{M}\) lives in. In what follows, we provide a brief overview of manifolds to build geometric intuition; for a more formal and in-depth treatment, we refer interested readers to standard texts in differential topology~\cite{guillemin2025differential, lee2000introduction, dundas2018short}. We can also view a \(k\)-dimensional manifold as the set of solutions of the form \(\{ \mathbf{x} \in \mathcal{X} \mid F(\mathbf{x}) = \mathbf{y} \in \mathcal{Y} \} \subset \mathcal{X}\), which has \(k\) degrees of freedom. In this setup, the map \(F\colon \mathcal{X} \to \mathcal{Y}\) is defined as \(F(\mathbf{x}) = [f_1(\mathbf{x}), f_2(\mathbf{x}), \ldots, f_{n-k}(\mathbf{x})]\), where each \(f_i\) is a scalar-valued function, and each constraint \(f_i(\mathbf{x}) = b_i\) reduces the dimension of the manifold by one,  assuming \(\mathbf{x}\) is \(n\)-dimensional. Consider the unit sphere \(S^1 = \{\mathbf{x} \in \mathbb{R}^2 \colon \|\mathbf{x}\|_2^2 = 1\} \subset \mathbb{R}^2\). This manifold is of dimension \(k = 1\) locally, yet it is embedded within a higher-dimensional \(n = 2\)-dimensional Euclidean space, and described using the \textbf{coordinates} \(x_1\) and \(x_2\). When examining the structure more closely, as illustrated in Figure~\ref{fig:circle}, each segment of the sphere may be perceived as resembling curved straight line segments. Indeed, for each open set \(\mathcal{U}_i\), there exists a homeomorphic mapping \(\varphi_i \colon \mathcal{U}_i \to (-1, 1) \subset \mathbb{R}\), where:
\[
\begin{aligned}
&\mathcal{U}_1 = \left\{ \mathbf{x} \in S^1 \mid x_2 > 0 \right\}, \quad
\mathcal{U}_2 = \left\{ \mathbf{x} \in S^1 \mid x_2 < 0 \right\}, \\
&\mathcal{U}_3 = \left\{ \mathbf{x} \in S^1 \mid x_1 > 0 \right\},\quad 
\mathcal{U}_4 = \left\{ \mathbf{x} \in S^1 \mid x_1 < 0 \right\},
\end{aligned}
\]
with corresponding mappings: \(\varphi_1(\mathbf{x}) = x_1\), \(\varphi_2(\mathbf{x}) = x_1\), \(\varphi_3(\mathbf{x}) = x_2\), \(\varphi_4(\mathbf{x}) = x_2\), as shown in Figure 1. Each pair \((\mathcal{U}_i, \varphi_i)\) is called a \textbf{chart}, and the collection of such charts that together cover the manifold forms an \textbf{atlas} \(\{(\mathcal{U}_i, \varphi_i)\}\). For a \(k\)-dimensional manifold, the \(k\)-tuples \(\mathbf{x} \in \mathbb{R}^k\), provided by a chart \(\varphi\), are called \textbf{coordinates}; if a single chart covers the manifold, they are global, otherwise they are \textbf{local coordinates}, as in our example where four charts were required to describe \(S^1\), i.e., \(S^1 = \bigcup_i \mathcal{U}_i\). For this reason, a point may have two or more local coordinate representations, and the transition from one coordinate system to another must be smooth; that is, on the overlaps \(\mathcal{U}_i \cap \mathcal{U}_j \neq \emptyset\), the \textbf{transition maps} \(\varphi_{ij} = \varphi_i \circ \varphi_j^{-1}\) should be differentiable. A manifold equipped with such an atlas is called a \textbf{smooth manifold}. Moreover, the inverse of each coordinate map, \(\varphi_i^{-1}\colon V_i \subset \mathbb{R}^k \to \mathcal{U}_i \subset \mathbb{R}^n\), where both \(V_i\) and \(\mathcal{U}_i\) are open sets, is called a \textbf{local parametrization}. For instance, in our example one has \(\varphi_1^{-1}(x_1) = (x_1, \sqrt{1-x_1^2})\), \(\varphi_2^{-1}(x_1) = (x_1, -\sqrt{1-x_1^2})\), \(\varphi_3^{-1}(x_2) = (\sqrt{1-x_2^2}, x_2)\), and \(\varphi_4^{-1}(x_2) = (-\sqrt{1-x_2^2}, x_2)\). As an illustration, on the overlap \(\mathcal{U}_1 \cap \mathcal{U}_3\), the transition map is \(\varphi_{13}(x_2) = \sqrt{1-x_2^2}\), and on \(\mathcal{U}_2 \cap \mathcal{U}_3\) it is \(\varphi_{23}(x_2) = -\sqrt{1-x_2^2}\). Each of these transition maps is smooth, and therefore \(S^1\) is a smooth manifold.

\begin{wrapfigure}{l}{0.45\linewidth}
\vspace{-10pt} 
\centering
\includegraphics[width=\linewidth]{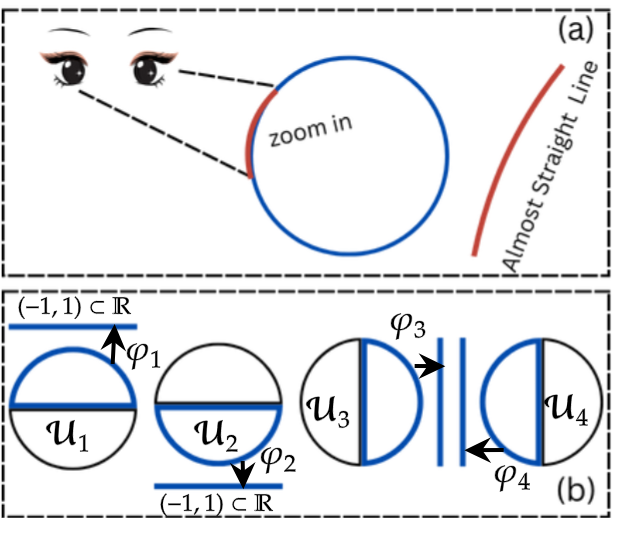}
\caption{An example manifold and local coordinates.}
\label{fig:circle}
\vspace{-10pt} 
\end{wrapfigure}

Beyond these basic examples, new manifolds can be obtained by forming Cartesian products, which are known as \textbf{product manifolds}. If \(\mathcal{M}_1\) and \(\mathcal{M}_2\) are manifolds of dimensions \(k_1\) and \(k_2\), then their product \(\mathcal{M} = \mathcal{M}_1 \times \mathcal{M}_2\) is itself a manifold of dimension \(k_1 + k_2\). The local charts of \(\mathcal{M}\) are given by combining charts from \(\mathcal{M}_1\) and \(\mathcal{M}_2\). A fundamental example is the \(n\)-torus, defined as the product \(\mathbb{T}^n = S^1 \times \cdots \times S^1\), which is an \(n\)-dimensional smooth manifold.

A \textbf{{submanifold}} of a smooth \(k\)-dimensional manifold \(\mathcal{M}\) is
a subset \(\mathcal{N} \subset \mathcal{M}\) that is itself a smooth manifold of
dimension \(\ell < k\). The codimension of \(\mathcal{N}\) in \(\mathcal{M}\) is
defined as \(\operatorname{codim} \mathcal{N} = k - \ell\). More precisely,
\(\mathcal{N}\) is called an \(\ell\)-dimensional submanifold if for every point
\(p \in \mathcal{N}\), there exists a chart \((U,\varphi)\) of \(\mathcal{M}\)
with \(p \in U\) such that \(\varphi(U \cap \mathcal{N}) \subset
\mathbb{R}^{\ell} \times \{0\} \subset \mathbb{R}^{k}\). As an explicit example,
in the torus \(\mathbb{T}^{2} = S^{1} \times S^{1}\), the subset
\(S^{1} \times \{p\}\) for a fixed point \(p \in S^{1}\) is a 1-dimensional
submanifold of the 2-dimensional torus. As another important example, the open
sets \(U_i\) used in the example of the circle \(S^{1}\) are 1-dimensional
submanifolds of \(S^{1}\). Indeed, any open subset \(U \subset \mathcal{M}\) of
a smooth manifold is automatically a submanifold of \(\mathcal{M}\) of the same
dimension~\cite[p.~100]{tu2010introduction}.

The \textbf{tangent space} of a $k$-dimensional manifold $\mathcal{M} \subset \mathbb{R}^n$
at a point $x \in \mathcal{M}$, denoted $T_x \mathcal{M}$, is the $k$-dimensional
vector space spanned by the partial derivatives of a local parametrization at that
point. More precisely, let
\(
\varphi^{-1} \colon V \subset \mathbb{R}^k \to U \subset \mathcal{M}
\)
be a local parametrization around $x$, and let
\(
u_0 \coloneq \varphi(x) \in V
\)
be the coordinate of $x$ in the chart (so that $\varphi^{-1}(u_0)=x$).
The Jacobian matrix (also called the dictionary matrix) is
\begin{equation}
D(u_0)
= 
\begin{bmatrix}
\frac{\partial \varphi^{-1}}{\partial u_1}(u_0) &
\cdots &
\frac{\partial \varphi^{-1}}{\partial u_k}(u_0)
\end{bmatrix}
\in \mathbb{R}^{n \times k}.
\end{equation}
The columns of $D(u_0)$ are tangent vectors, and their span equals the tangent space:
\begin{equation}    
T_x \mathcal{M} = \operatorname{span}(D(u_0)).
\end{equation}
Intuitively, the tangent space provides the best linear approximation of the manifold near \(x\): just as a tangent line approximates a curve or a tangent plane approximates a surface, \(T_x \mathcal{M}\) gives the linear subspace of \(\mathbb{R}^n\) that locally models the geometry of \(\mathcal{M}\) at \(x\).


Up to this point, manifolds have been introduced as spaces that locally resemble
Euclidean space and admit smooth coordinate charts. However, many geometric
problems require additional structure: one must be able to measure lengths,
angles, and distances between points or tangent vectors. The appropriate
framework is that of a \textbf{Riemannian manifold} $(\mathcal M,g)$, in which
each point $x \in \mathcal M$ is equipped with an inner product 
$g_x(\cdot,\cdot)$ on the tangent space $T_x\mathcal M$, depending smoothly on $x$
in local coordinates. This Riemannian metric allows one to measure lengths of
curves and thus induces a notion of distance generalizing the Euclidean metric. Once a metric is available, one may
speak of geometric notions invariant under isometries (distance-preserving mappings), such as orthogonality or
projections onto submanifolds.

The \textbf{normal space} at a point \(x \in \mathcal{M} \subset \mathbb{R}^n\), denoted \(N_x \mathcal{M}\), is the subspace of \(\mathbb{R}^n\) consisting of all vectors orthogonal to the tangent space \(T_x \mathcal{M}\). Formally,  
\[
N_x \mathcal{M} \coloneq \left\{ \mathbf{n} \in \mathbb{R}^n \;\mid\; \langle \mathbf{t}, \mathbf{n} \rangle = 0 \;\; \forall \mathbf{t} \in T_x \mathcal{M} \right\},
\]  
where \(\langle \cdot , \cdot \rangle\) denotes the Euclidean inner product. Since \(\dim(T_x \mathcal{M}) = k\), it follows that \(\dim(N_x \mathcal{M}) = n-k\), and we have the orthogonal decomposition \(\mathbb{R}^n = T_x \mathcal{M} \oplus N_x \mathcal{M}\) with $\oplus$ the direct sum. As a simple example, for the circle \(S^1 \subset \mathbb{R}^2\), the tangent space \(T_x S^1\) at a point \(x\) is the tangent line to the circle, while the normal space \(N_x S^1\) is the radial line through \(x\), perpendicular to the tangent.

\textbf{Transversality and Preimages: }

\begin{wrapfigure}{l}{0.45\linewidth}
\vspace{-10pt}
\centering
\includegraphics[width=\linewidth]{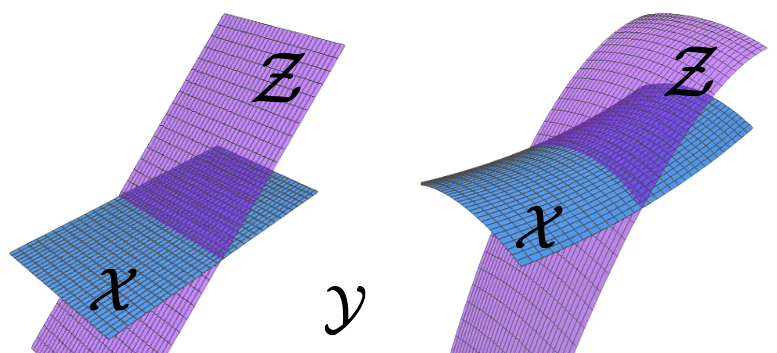}
\caption{Transversal intersections. (a) Linear subspaces. (b) Curved submanifolds.}
\label{fig:transversality}
\vspace{4pt}
\end{wrapfigure}

Let \(F\colon \mathcal{X} \to \mathcal{Y}\) be a smooth map between manifolds, and let \(\mathcal{Z} \subset \mathcal{Y}\) be a smooth submanifold. We say that \(F\) is \textbf{transversal} to \(\mathcal{Z}\), written \(F \pitchfork \mathcal{Z}\), if for every point \(x \in \mathcal{X}\) with \(F(x) \in \mathcal{Z}\), we have
\begin{equation}
\mathrm{Im}(dF_x) + T_{F(x)} \mathcal{Z} = T_{F(x)} \mathcal{Y}.
\end{equation}
If this condition holds, then the preimage $F^{-1}(\mathcal{Z})$ is a smooth 
submanifold of $\mathcal{X}$. Moreover, the codimension of 
$F^{-1}(\mathcal{Z})$ in $\mathcal{X}$ equals the codimension of 
$\mathcal{Z}$ in $\mathcal{Y}$. Let the map 
$i \colon \mathcal{X} \hookrightarrow \mathcal{Y}$ denote the inclusion, that is, 
$i(x)=x$ for all $x \in \mathcal{X}$. In this case, the preimage 
$i^{-1}(\mathcal{Z})$ is simply the intersection 
$\mathcal{X} \cap \mathcal{Z}$. Since the derivative of the inclusion is the 
natural embedding $T_x\mathcal{X} \subset T_x\mathcal{Y}$, the transversality 
condition reduces to
\begin{equation}
    T_x \mathcal{X} + T_x \mathcal{Z} = T_x \mathcal{Y},
    \qquad x \in \mathcal{X} \cap \mathcal{Z}.
\end{equation}
If this holds, then 
$\mathrm{codim}_{\mathcal{Y}}(\mathcal{X} \cap \mathcal{Z})
  = \mathrm{codim}_{\mathcal{Y}} (\mathcal{X})
  + \mathrm{codim}_{\mathcal{Y}} (\mathcal{Z})$.
See Figure~\ref{fig:transversality} for a geometric visualization of transversal 
intersections between flat and curved submanifolds.

\subsection{Nonlinear Orthogonal Projection}
\label{sec:nonlinear-projection}

Let \( Z \) be any metric space equipped with a distance metric  \( d \). Then, we borrow the well-defined \textbf{non-linear orthogonal projection} concept from the literature~\cite{dudek1994nonlinear}: The nonlinear orthogonal projection onto a non-empty subset \( \mathcal{D}\) of \(Z\) is a mapping (i.e., relation) \( \mathcal{P} \subset Z \times \mathcal{D} \)
with domain
\begin{align}
\text{dom}(\mathcal{P}) \coloneq \left\{ z \in Z \mid \exists! z' \in \mathcal{D}, d(z, z') = \rho(z, \mathcal{D}) \right\},
\end{align}
where \( \rho(z, \mathcal{D}) \coloneq \inf_{x \in \mathcal{D}} d(z, x) \). For instance, let us consider the set (e.g., manifold in this example) \(\mathcal{M} \coloneq \{(x, x^2) \mid x \in \mathbb{R}\}\), where \(\mathcal{M} \subset Z\). In this example, if we select any point in \(\{(0,t) \mid t > 1/2\}\), the projection assigns this point to two points on the manifold as shown in Figure \ref{ex1}. The domain of the orthogonal projection \(\mathcal{P}\) onto \(\mathcal{M}\) is \(
\mathrm{dom}(\mathcal{P}) = \mathbb{R}^2 \setminus \{(0,t) \mid t > 1/2\} \).

\begin{figure}[h]
\centering

\begin{subfigure}{0.45\linewidth}
\centering
\begin{tikzpicture}[scale=1.0, >=latex]

  \draw[domain=-1.2:1.2, smooth, variable=\x] plot ({\x},{(abs(\x))^2});
  \node[right] at (1.3,1.6) {$\mathcal{M}$};

  \filldraw (-0.3,1) circle [radius=0.06];
  \node[right] at (-0.3,1) {$\mathbf{z}$};

  \draw[line width=1.4] (-0.3,1) -- (-0.83,0.69);
  \filldraw (-0.83,0.69) circle [radius=0.06];
  \node[left] at (-0.83,0.69) {$\mathbf{z}'$};

\end{tikzpicture}
\caption{}
\end{subfigure}
\hfill
\begin{subfigure}{0.45\linewidth}
\centering
\begin{tikzpicture}[scale=1.0, >=latex]

  \draw[->] (-2.0,0) -- (2.0,0);
  \draw[->] (0,-0.1) -- (0,2.4);
  \draw[->, line width=1.4pt] (0,0.55) -- (0,2.4);

  \draw[domain=-1.2:1.2, smooth, variable=\x] plot ({\x},{(abs(\x))^2});
  \node[right] at (1.3,1.6) {$\mathcal{M}$};

  \draw (0,0.5) circle [radius=0.06];
  \node[left] at (0,0.5) {\(\frac{1}{2}\)};

  \filldraw[red] (0,1) circle [radius=0.06];
  \node[left] at (0,1) {1};

  \draw[line width=1.4, red] (0,1) -- (-0.7071,0.5);
  \draw[line width=1.4, red] (0,1) -- (0.7071,0.5);

  \draw[->] (0.2,1.9) -- (0.05,1.3);
  \node[right] at (0.05,2.0) {\(\mathbb{R}^2 \setminus \mathrm{dom}\,\mathcal{P}\)};

\end{tikzpicture}
\caption{}
\end{subfigure}

\caption{Nonlinear orthogonal projection onto a manifold. 
(a) $z \in \operatorname{dom}(\mathcal{P})$ has a unique closest point $z' \in \mathcal{M}$. 
(b) Points outside $\operatorname{dom}(\mathcal{P})$ have multiple closest points on $\mathcal{M}$.}
\label{ex1}

\end{figure}
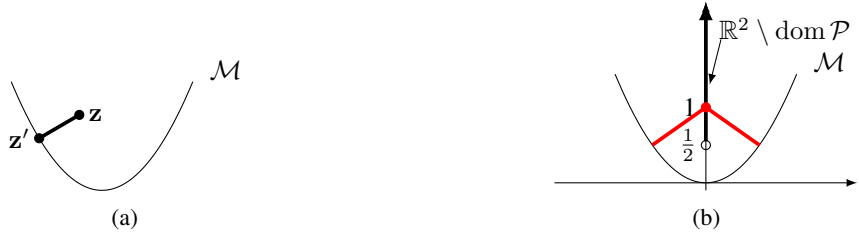

\subsection{Sparse Representation}
\label{appendix_sparse_representation}
Sparse representation (SR), also known as Sparse Coding (SC), refers to the idea of 
expressing a signal $\mathbf{s} \in \mathbb{R}^n$ using a linear combination of only 
a small number of atoms selected from a larger collection in a dictionary 
$\mathbf{D} \in \mathbb{R}^{n \times N}$. A typical example is Compressive Sensing 
(CS)~\cite{CS1,CS2}, which can be viewed as a particular case of this framework: 
if a signal $\mathbf{s} \in \mathbb{R}^{d}$ admits a sparse coefficient vector 
$\mathbf{x} \in \mathbb{R}^{N}$ with respect to a dictionary or basis 
$\mathbf{\Phi} \in \mathbb{R}^{d \times N}$, then it can be acquired in a 
reduced-dimensional form through a linear measurement operator 
$\mathbf{A} \in \mathbb{R}^{n \times d}$. As a result, the measurement vector can 
be represented using an alternative dictionary
\(
    \mathbf{D} = \mathbf{A}\mathbf{\Phi} \in \mathbb{R}^{n \times N}.
\)
Within the SC literature, \textbf{signal synthesis} describes the process 
of generating a signal $\mathbf{s} = \mathbf{D}\mathbf{x} \in \mathbb{R}^n$ from a 
sparse code $\mathbf{x} \in \mathbb{R}^N$ using a given dictionary $\mathbf{D}$. 
In contrast, \textbf{signal analysis} focuses on estimating the SR 
coefficients $\mathbf{x}$ from the observed signal $\mathbf{s}$. A typical signal 
analysis procedure first estimates the locations of the nonzero coefficients of 
$\mathbf{x}$, $\Lambda = \{ i \mid x_i \neq 0 \}$; this process is known as 
sparse support estimation~\cite{csen, osen}. Then the nonzero coefficients can be estimated using 
the ordinary least-squares solution
\(
    \mathbf{x}_\Lambda 
    = (\mathbf{D}_\Lambda^\top \mathbf{D}_\Lambda)^{-1}
      \mathbf{D}_\Lambda^\top \mathbf{s},
\)
which is equivalent to computing the orthogonal projection of $\mathbf{s}$ onto 
the subspace $\mathrm{span}(\mathbf{D}_\Lambda)$, i.e.,
\(
    \widehat{\mathbf{s}} 
    = \mathbf{D}_\Lambda
      (\mathbf{D}_\Lambda^\top \mathbf{D}_\Lambda)^{-1}
      \mathbf{D}_\Lambda^\top \mathbf{s}.
\)

\begin{wrapfigure}{r}{0.48\linewidth}
\vspace{-10pt}
\centering
\includegraphics[width=\linewidth]{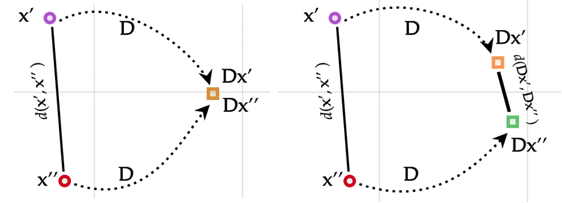}
\caption{Uniqueness vs.\ stability. (a) Null-space uniqueness avoids collisions. (b) RIP preserves distances and ensures stability.}
\label{fig:rip_graph}
\vspace{-10pt}
\end{wrapfigure}
SR can also be viewed through the \textbf{union-of-subspaces} model~\cite{blumensath2009sampling}, where each signal 
$\mathbf{s}$ lies in one element of a finite collection of low-dimensional linear 
subspaces determined by support sets. Let the dictionary 
$\mathbf{D}$ be partitioned into groups of atoms 
$\mathbf{D}_{\Lambda_i} \in \mathbb{R}^{n \times k_i} $, where each $\Lambda_i \subset \{1,\dots,N\}$ indexes a 
subset of columns and the corresponding subspace 
$\mathcal{S}_i = \mathrm{span}(\mathbf{D}_{\Lambda_i})$ has intrinsic dimension 
$k_i = |\Lambda_i|$. The overall model is therefore 
\begin{equation}
    \mathcal{A} = \bigcup_{i=1}^{L} \mathcal{S}_i,
    \qquad 
    \mathcal{S}_i = 
    \left\{ \mathbf{s} = \mathbf{D}_{\Lambda_i}\mathbf{x}_{\Lambda_i} \mid
    \mathbf{x}_{\Lambda_i} \in \mathbb{R}^{k_i} \right\},
\end{equation}
where \(L\) is the number of different subspace which requires particular attention in later stage of the discussion. The natural question is when the representation $\mathbf{s}=\mathbf{D}\mathbf{x}$, 
with $\mathbf{D}\in\mathbb{R}^{n\times N}$ and $N\gg n$, is unique. 
Equivalently, we require
\(
    \mathbf{D}\mathbf{x}' \neq \mathbf{D}\mathbf{x}'' 
    \text{ for all } \mathbf{x}' \neq \mathbf{x}''.
\) Consider the unwanted case $\mathbf{D}\mathbf{x}'=\mathbf{D}\mathbf{x}''$, which 
implies $\mathbf{D}(\mathbf{x}'-\mathbf{x}'')=0$. For $k$-sparse signals, the 
difference $\mathbf{x}'-\mathbf{x}''$ is at most $2k$-sparse in the worst case. 
Thus we require every nonzero $2k$-sparse vector to lie outside $\mathrm{Null}(\mathbf{D})$, 
i.e., $\mathbf{x}''' \notin \mathrm{Null}(\mathbf{D})$. This in turn requires 
$\mathbf{D}$ to have rank at least $2k$. Since $\mathrm{rank}(\mathbf{D}) \le n$, 
uniqueness demands
\(
    n \ge 2k.
\) This condition corresponds to the worst-case scenario in which the 
$k$-dimensional subspaces determined by different supports do not intersect. 
In more realistic union-of-subspaces models, the required embedding dimension 
is governed by the maximum dimension of all secant subspaces, defined as
\(
    k_{\max}
= \max_{i\neq j} \dim\big( \mathrm{span}(\mathcal{S}_i \cup \mathcal{S}_j) \big)
,
\)
and uniqueness holds whenever $n \ge k_{\max}$~\cite{blumensath2009sampling, lu2008theory}. In other words, although exact recovery of $\mathbf{s}$ via the minimum-norm 
recovery method is impossible in general, perfect reconstruction 
becomes achievable when the signal is sparse and the embedding dimension satisfies 
$n > k_{\max}$.

When approximately sparse signals are measured in noise, a stronger condition is
required to guarantee stable recovery. This is the Restricted Isometry Property
(RIP).
\begin{defn}[Restricted Isometry Property]
A dictionary $\mathbf{D} \in \mathbb{R}^{n \times N}$ satisfies RIP of order $k$
if there exists a smallest constant $\delta_k(\mathbf{D})$ such that
\[
    (1-\delta_k)\|\mathbf{x}\|_2^2 
    \;\le\; \|\mathbf{D}\mathbf{x}\|_2^2 
    \;\le\; (1+\delta_k)\|\mathbf{x}\|_2^2,
\]
for all $k$-sparse vectors $\mathbf{x} \in \mathbb{R}^N$. The quantity 
$\delta_k(\mathbf{D})$ is the restricted isometry constant (RIC).
\end{defn}

\begin{wrapfigure}{l}{0.48\linewidth}
\vspace{-10pt}
\centering

\includegraphics[width=0.48\linewidth]{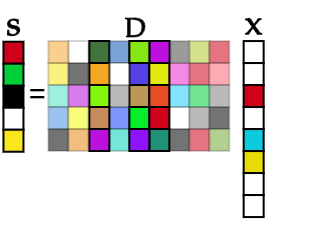}
\hfill
\includegraphics[width=0.48\linewidth]{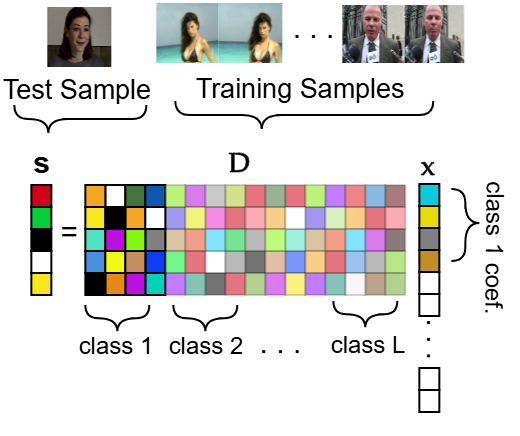}

\caption{Sparsity models. (a) Conventional sparsity. (b) Group sparsity.}
\label{fig:sparsity}
\vspace{-10pt}
\end{wrapfigure}
The recovery guarantees for $k$-sparse signals require controlling the 
$k_{\max}$-order constant (equal to $2k$ in the classical setting~\cite{RIP}), 
denoted $\delta_{k_{\max}}(\mathbf{D})$. The intuition parallels the null-space 
analysis above: to avoid $\mathbf{D}\mathbf{x}' = \mathbf{D}\mathbf{x}''$, we 
require $\mathbf{D}(\mathbf{x}' - \mathbf{x}'') \neq 0$ for all distinct 
$k$-sparse vectors, which in turn requires that $\mathbf{D}$ not annihilate any nonzero vector 
belonging to a $k_{\max}$-dimensional secant subspace.
 However, uniqueness is not 
sufficient for robustness. In addition to preventing collisions, we also need 
the images of distinct sparse vectors to remain well separated after mapping. 
This stronger requirement is captured by the RIP, which enforces that the 
distance between any two $k$-sparse vectors, \(\mathbf{x}', \mathbf{x}''\) cannot collapse under $\mathbf{D}$, 
namely,
\(
    (1-\delta_{k_{\max}})\|\mathbf{x}' - \mathbf{x}''\|_2^2
    \;\le\; 
    \|\mathbf{D}\mathbf{x}' - \mathbf{D}\mathbf{x}''\|_2^2.
\)

The proof of RIP-type results follows the same strategy commonly used in 
modern analyses of the Johnson--Lindenstrauss Lemma (JLL)~\cite{johnson1984extensions}. 
The JLL asserts that a mapping from $\mathbb{R}^N$ to $\mathbb{R}^n$ can 
approximately preserve pairwise distances for any finite subset 
$\mathcal{S} \subset \mathbb{R}^N$. The argument proceeds in two steps. 
(i) For a fixed pair $\mathbf{x}',\mathbf{x}'' \in \mathcal{S}$, define 
$\mathbf{x} = \mathbf{x}' - \mathbf{x}''$. One first controls the probability 
that $\mathbf{D}$ distorts its length, i.e., probability of failure:
\(
    \Pr\!\left(\|\mathbf{D}\mathbf{x}\|_2^2 \le (1-\gamma)\|\mathbf{x}\|_2^2\right)
    + 
    \Pr\!\left(\|\mathbf{D}\mathbf{x}\|_2^2 \ge (1+\gamma)\|\mathbf{x}\|_2^2\right)
    \le P_e.
\)
(ii) A uniform guarantee for all pairs in $\mathcal{S}$ then follows from the 
union bound:
\(
    \Pr\!\left(
        (1-\gamma)\|\mathbf{x}\|_2^2
        \le \|\mathbf{D}\mathbf{x}\|_2^2
        \le (1+\gamma)\|\mathbf{D}\mathbf{x}\|_2^2
        \;\text{for all }\mathbf{x}\in\mathcal{S}
    \right)
    \ge 1 - |\mathcal{S}|P_e.
\)
For general subsets of $\mathbb{R}^N$, the covering number scales as 
$\mathcal{O}((1/\varepsilon)^N)$. When restricted to $k$-sparse vectors, however, 
the effective size of $\mathcal{S}$ reduces dramatically: each fixed support 
defines a $k$-dimensional Euclidean subspace with covering number 
$\mathcal{O}((1/\varepsilon)^k)$, and there are $\binom{N}{k}$ such supports. 
Thus the total covering number decreases from $\mathcal{O}((1/\varepsilon)^N)$ to 
\(
   \binom{N}{k} \mathcal{O}\!\left((1/\varepsilon)^k\right).
\)
This exponential reduction ensures that an embedding satisfies the RIP 
with overwhelming probability whenever
\(
    n \ge \mathcal{O}\!\left(k \log(N/k)\right),
\)
consistent with classical sparse representations. In more structural union-of-subspace models, such as group sparsity as in Figure~\ref{fig:sparsity} and hierarchical models, the number of different subspaces $\binom{N}{k}$ can be further reduced.

\begin{Summary}{}{firstsummary}
\label{firstsummary}
\textbf{Although an isometric embedding of the entire ambient space 
$\mathbb{R}^N$ into a lower-dimensional observation space is 
infeasible, such embeddings become possible when the cardinality of the target set is reduced,} i.e.,  target set is 
restricted to a union of $k$-dimensional subspaces. The same 
realization extends to low-dimensional curved spaces: by replacing 
$\mathbb{R}^N$ with an underlying union of submanifolds, as will be discussed in the following PoS Hypothesis.
\end{Summary}

\section{PoS Hypothesis}
\label{sec:PoS_Theory}

The \textbf{Manifold Hypothesis} states that the high-dimensional real-world data live on a lower dimensional manifold~\cite{ManifoldHypothesis}. It is then possible to parametrize this surface with fewer variables as would normally be required to represent the ambient space.
Let us denote the $i$th observation in the ambient space as $\mathbf{x}^i \in \mathbb{R}^n$, and assume that for all $i$, $\mathbf{x}^i$ approximately lies on a manifold $\mathcal{M} \subset \mathbb{R}^n$, where $\mathrm{dim}(\mathcal{M}) = m < n$. 
\textbf{Manifold representation learning} aims to reconstruct $\mathcal{M}$, as well as to learn a local parametrization function, $\Phi^{-1} \colon V \rightarrow \mathcal{M}$ with $V \subset \mathbb{R}^m$. A typical example is the case $m = k$, where the best $k$-dimensional subspace approximation is given by PCA~~\cite{hotelling1933analysis}. Other examples include linear and nonlinear bottleneck autoencoders~\cite{jamshidi2011geometric}. 
Built upon the foundation of the manifold hypothesis, the \textbf{classification manifold hypothesis} further states that the samples from distinct input classes tend to cumulate on the distinct sub-manifolds of $\mathcal{M}$~\cite{TangentClassifier}.

\begin{wrapfigure}{r}{0.55\linewidth}
\vspace{-10pt}
\centering
\includegraphics[width=\linewidth]{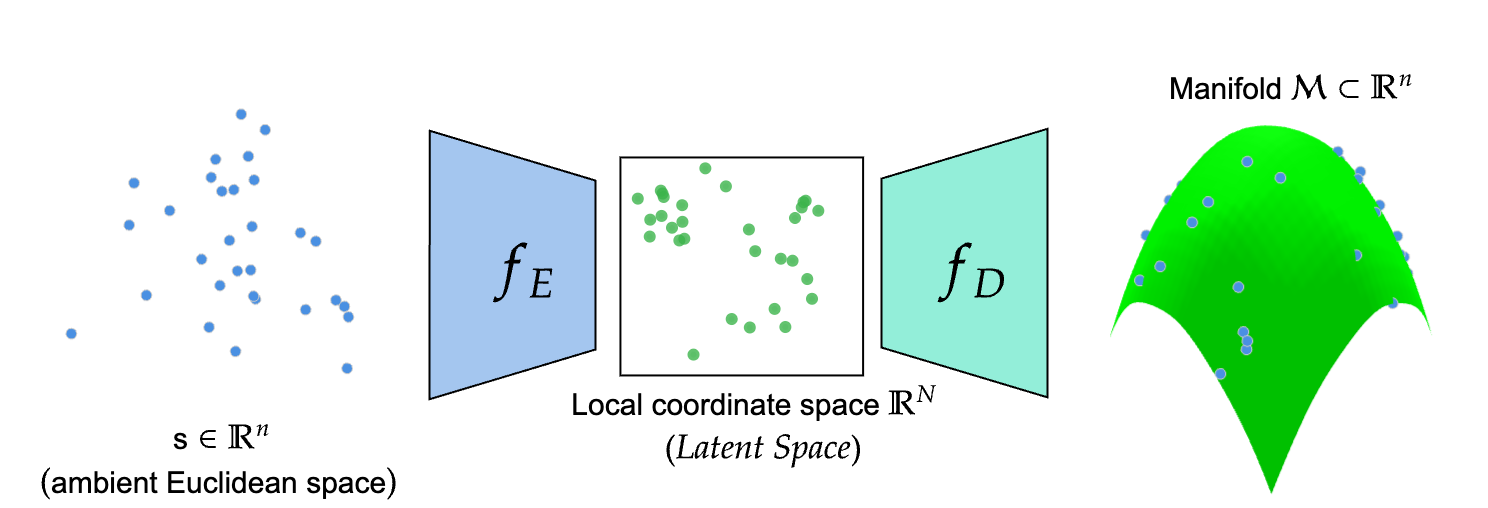}
\caption{Autoencoder as a geometric coordinate model. The encoder and decoder provide local coordinate maps for the data manifold $\mathcal{M}$, later refined to capture its structure.}
\label{fig:placeholder}
\vspace{-10pt}
\end{wrapfigure}
Let us view an autoencoder as an unsupervised mechanism for learning an 
approximation of the underlying data manifold, which can later be adapted for 
semi-supervised tasks. Consider a dataset consisting of samples drawn from an 
unknown manifold $\mathcal{M}_s \subset \mathbb{R}^n$, i.e., 
$\mathcal{D}_s = \{\mathbf{s}^{\,1}, \dots, \mathbf{s}^{\,I}\},  \mathbf{s}^{\,i} \in \mathcal{M}_s$.
The autoencoder aims to simultaneously learn the coordinate map $\Phi$ and its corresponding coordinates $\mathbf{x} \in \mathbb{R}^{N}$, such that $\Phi^{-1}(\mathbf{x}^{\,i}) = \mathbf{s}^{\,i}$. Autoencoders in general learn a non-linear \textit{encoder} function $f_E\colon \mathbb{R}^{n} \rightarrow \mathbb{R}^{N}$, to map the input-space coordinates $\mathbf{s} \in \mathbb{R}^{n}$ onto the local coordinates $\mathbf{x} \in \mathbb{R}^{N}$, also known as the latent space. The encoder is jointly learned with a \textit{decoder} $f_D\colon \mathbb{R}^{N} \rightarrow \mathbb{R}^{n}$, that maps the latent space coordinates back to the input space. This encoder-decoder architecture aims to reconstruct the data, i.e., $f_D(f_E(\mathbf{s}^i)) \approx \mathbf{s}^i$, via the loss function
\begin{equation}
    \mathcal{L}(\phi,\theta) = \sum_{\mathbf{s}^i \in \mathcal{D}_s} \|f_D(f_E(\mathbf{s}^i)) - \mathbf{s}^i\|_2^2.
\end{equation}
The encoder $f_E$ approximates the coordinate map 
$\Phi = \bigcup_i \varphi_i$, and the decoder $f_D$ approximates the 
parametrization $\Phi^{-1}$. As discussed before, when $N =m< n$, the autoencoder reduces to the 
classical bottleneck formulation, where the latent dimension provides a 
lower-dimensional representation of the data manifold. However, many 
autoencoder architectures use an overcomplete latent space with $N \ge n$. 
Examples include sparse autoencoders~\cite{SparseAutoencoders} which, in the union-of-subspaces 
interpretation discussed earlier, produce a sparse coefficient vector that 
identifies the active subspace of each sample, as well as denoising 
autoencoders~\cite{denoisedAutoencoders} and more recent masked autoencoders~\cite{he2022masked}. In all cases, the latent 
representation $\mathbf{x}^i = f_E(\mathbf{s}^i)$ lives in an abstract coordinate 
space $\mathbb{R}^N$ that is not directly observable, and whose geometric meaning 
depends on the underlying structure of the data manifold (e.g., local affine 
coordinates in a piecewise-linear approximation of a curved manifold).


\begin{wrapfigure}{r}{0.55\linewidth}
\vspace{-10pt}
\centering
\includegraphics[width=\linewidth]{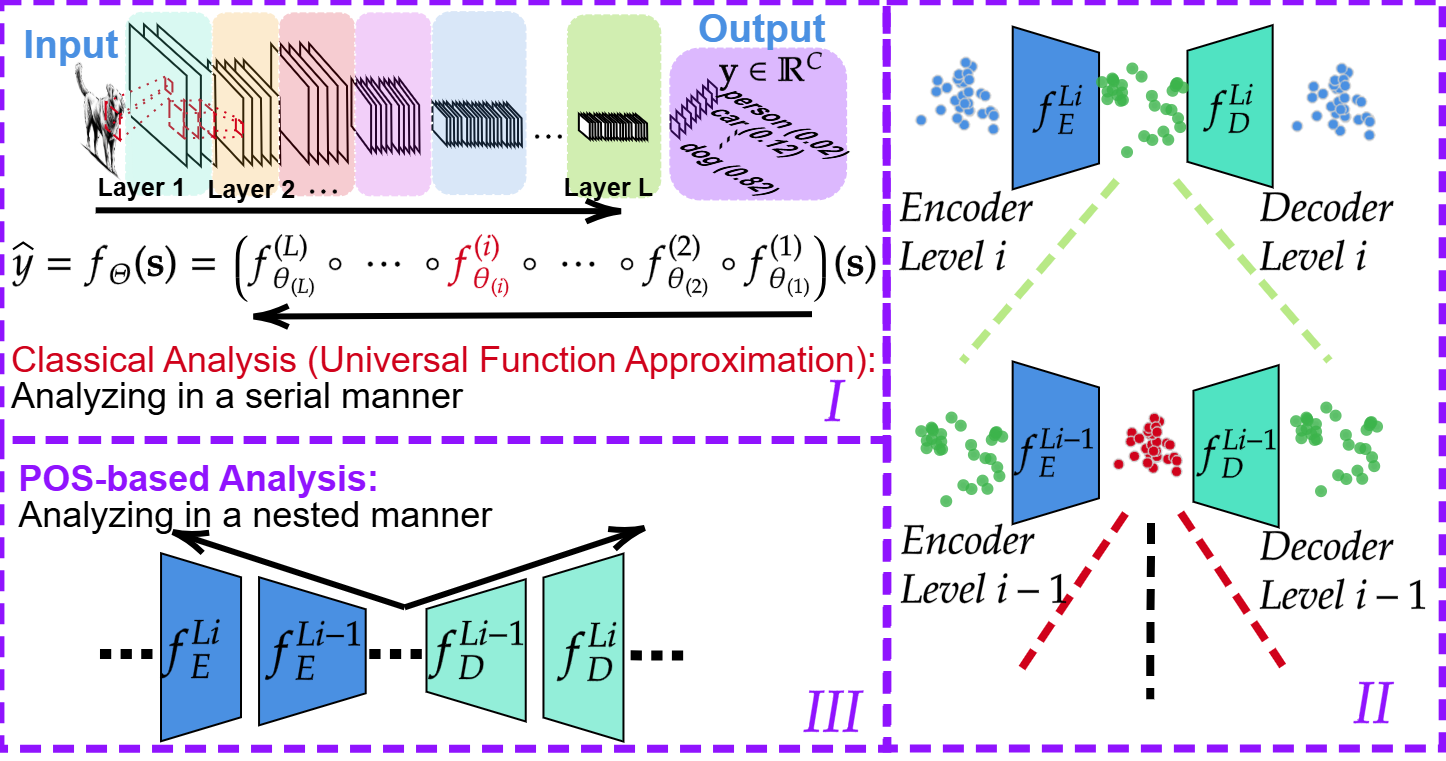}
\caption{Classical vs.\ PoS-based views of neural networks. 
(I) Flat layer-wise composition. 
(II) Nested encoder–decoder structure enabling manifold representations. 
(III) Hierarchical organization emerging from the framework.}
\label{fig:classicalvsPoSAnalysis}
\vspace{-12pt}
\end{wrapfigure}

In this work, we do not analyze autoencoders directly. Instead, as illustrated in 
Figure~\ref{fig:classicalvsPoSAnalysis}, we use the autoencoder template as a 
unifying mathematical framework.
This framework allows us to derive most modern neural network architectures. By choosing which components of the template to retain, remove, or replace according to the neural network axioms introduced in the sequel, we obtain concrete architectures as special realizations. Thus, the autoencoder serves as an abstract blueprint, while the axioms dictate how its parts are instantiated to produce the wide range of models used in practice.

\subsection{Axiomatizing Neural Networks's Operations}
The foundations of mathematics, particularly in Euclidean space, were established by formalizing simple observations into axioms. One of Euclid’s axioms, for example, posits that a straight line can be drawn between any two points, representing the shortest path between them, and that this straight line can be extended infinitely in both directions. This principle was inspired by empirical observations, such as the fact that light seemingly travels in straight lines, which were then abstracted into mathematical postulates. Similarly, we can analyze the behavior of neural networks, specifically, those that outperform others in practice, and propose theoretical frameworks that stem from fundamental axioms derived from these observations. Just as all mathematical formulas in Euclidean geometry can theoretically be derived from a limited set of basic axioms, theories of neural networks can be systematically developed and refined based on foundational principles inferred from their observed behaviors. Thus, we present the following postulates for neural network operations.


\begin{postulate}[Compactness of Data Representation]
Let $\mathcal{M}_D \subset \mathbb{R}^n$ be an $m$-dimensional submanifold 
representing a candidate target set for the data. A neural network operates 
as a mapping $P : \mathbb{R}^n \to \mathbb{R}^n$ such that 
$\hat{\mathbf{s}} = P(\mathbf{s}) \in \mathcal{M}_D$. An idealized, sufficiently expressive neural network admits mappings whose 
image is better structured as a union of lower-dimensional submanifolds 
$\bigcup_{i} \mathcal{M}_{D_i} \subseteq \mathcal{M}_D$, where each 
$\mathcal{M}_{D_i}$ has dimension $k_i < m$, i.e.,
\begin{equation}
    \mathcal{M}_D \xrightarrow{\text{Opt}} 
    \bigcup_{i=1} \mathcal{M}_{D_i}.
\end{equation}
\end{postulate}

This transition reflects a geometric compaction of the representation, in which 
the target set becomes smaller and more structured. As illustrated in 
Figure~~\ref{fig:toy}, a linear bottleneck autoencoder (top row) projects the data 
onto a full two-dimensional plane, even though the samples originate from two 
one-dimensional clusters. In contrast, the nonlinear autoencoder (bottom row) 
discovers a union of low-dimensional subspaces, each corresponding to one data 
component. Such compact target sets have exponentially fewer cardinality than 
their linear counterparts, thereby increasing the likelihood of achieving an 
isometric representation even when the reconstruction error itself does not 
decrease. This reduced target set naturally lowers the ambiguity of the mapping 
$P$, in accordance with the discussion in Summary 1 of Section \ref{appendix_sparse_representation}.

The practical implications of compact representation learning are particularly 
clear in the context of anomaly detection and hallucination, as illustrated in 
Figure~~\ref{fig:anomalies}. In the left panel, when anomaly samples fall outside 
the linear span of the learned dictionary $\mathbf{D}$, a non-compact linear 
model can identify them through large reconstruction or projection errors. 
However, in the middle panel, when anomalies lie \emph{on} the learned subspace, 
the same model fails: the anomaly becomes indistinguishable from in-distribution 
data, highlighting the inherent ambiguity of a non-compact representation. In 
contrast, the nonlinear model (right panel), which represents data as a union of 
lower-dimensional submanifolds, successfully separates the anomaly even though 
it resides on the ambient plane. We propose that this phenomenon explains a 
significant class of out-of-distribution (OOD) failures in deep learning, 
including misclassification of unseen data and hallucination effects in 
generative models.

\begin{figure}[t]
\centering

\begin{subfigure}{0.48\linewidth}
\centering
\includegraphics[width=\linewidth]{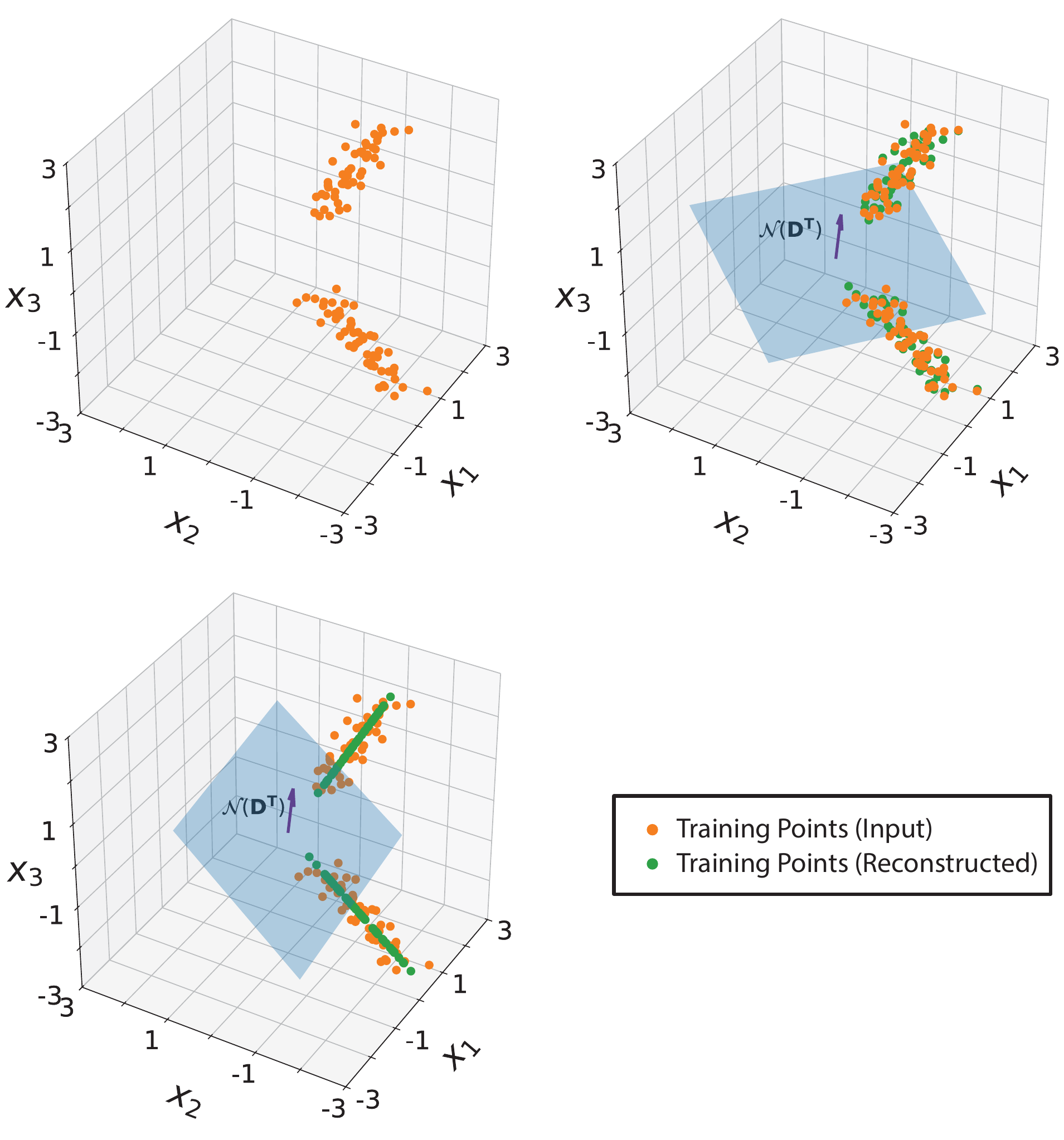}
\caption{
Toy example of representation compactness.
\textbf{Top row:} a linear bottleneck autoencoder learns a single global subspace $\mathrm{span}(D)$, which mixes distinct components and fails to preserve their intrinsic structure.
\textbf{Bottom row:} a nonlinear model learns a union of low-dimensional subspaces (i.e., flat submanifolds) $\mathrm{span}(d_1) \cup \mathrm{span}(d_2)$, preserving the structure of each component.
}
\label{fig:toy}
\end{subfigure}
\hfill
\begin{subfigure}{0.48\linewidth}
\centering
\includegraphics[width=\linewidth]{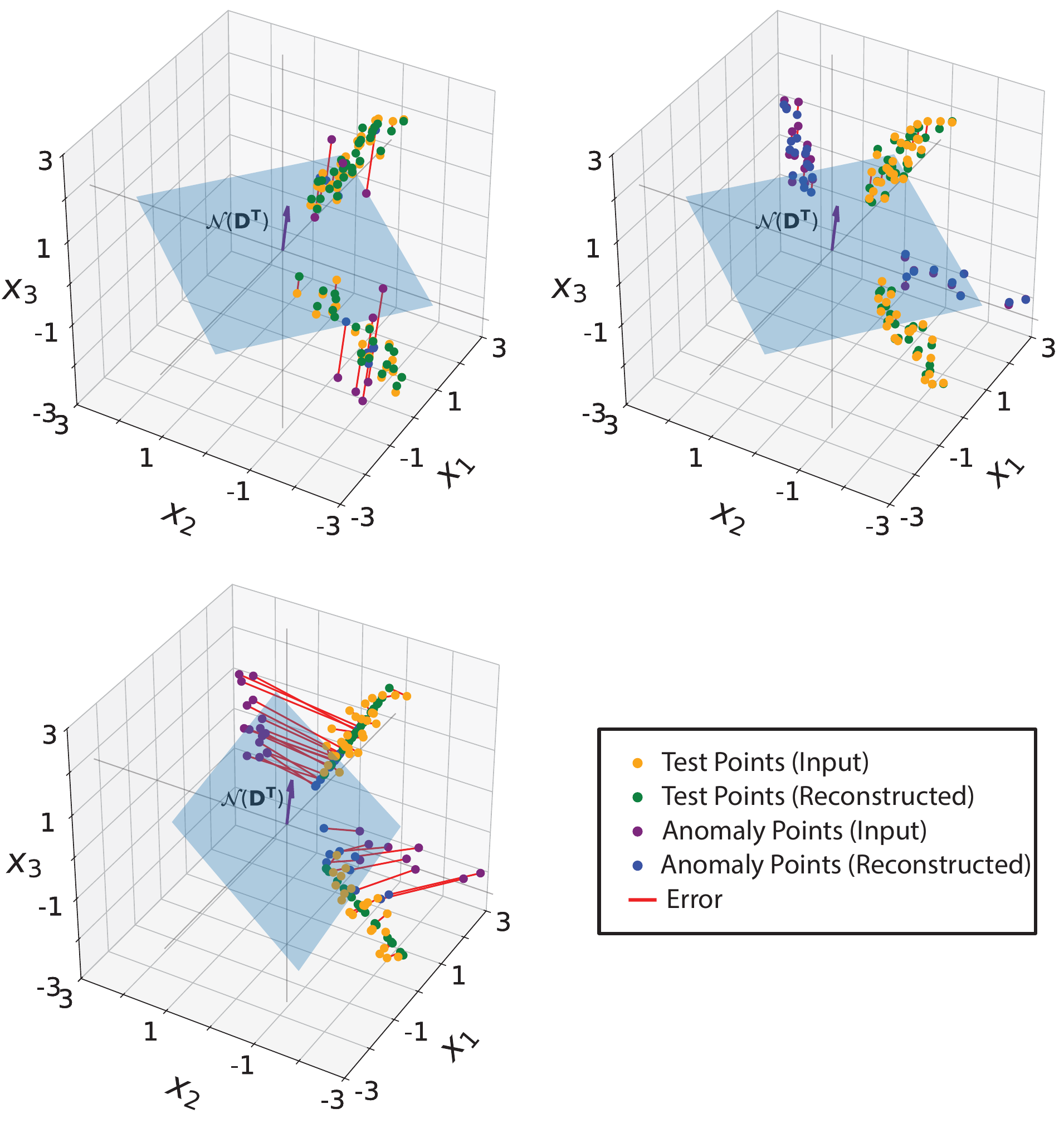}
\caption{
Effect of compact representations on anomaly detection. 
\textbf{Left:} when anomalies lie outside the learned subspace, a linear model can detect them via large Reconstruction Errors (REs). 
\textbf{Middle:} when anomalies lie within the learned subspace, the linear model fails, as anomalies yield REs comparable to in-distribution data. 
\textbf{Right:} the nonlinear model separates anomalies via projection onto a union of subspaces, where incorrect assignments yield large REs.
}
\label{fig:anomalies}
\end{subfigure}

\caption{
Compact vs. non-compact representations. 
Linear models learn a single global span, which can mix distinct structures and lead to ambiguity. 
Nonlinear models learn unions of submanifolds, producing more compact representations that preserve structure and enable robust anomaly detection under both off-subspace and on-subspace perturbations.
}
\label{fig:compactness_vs_ood}
\end{figure}

\begin{postulate}[Nonlinear Orthogonal Projection onto Submanifolds]
A neural network implements a non-linear orthogonal projection $P\colon \mathbb{R}^n \rightarrow \mathcal{M}_D$, where $\mathcal{M}_D \subseteq \mathbb{R}^n$ is the current estimation of the manifold during learning, satisfying $P^2 = P$. This projection decomposes as $P = f_D \circ f_E$, where the non-linear \emph{encoder} $f_E\colon \mathbb{R}^{n} \rightarrow \mathbb{R}^{N}$ maps the input-space to coordinates and the \emph{decoder} $f_D \colon \mathbb{R}^{N} \rightarrow \mathbb{R}^{n}$ reconstructs the input back to the ambient space $\mathbb{R}^n$. An ideal neural operator projects signals onto the union of its submanifolds rather than the manifold $\mathcal{M}_D$ itself, that is,
\(
\widehat{\mathbf{s}} \in \bigcup_{i=1} \mathcal{M}_{D_i} \quad \text{such that} \quad d(\widehat{\mathbf{s}}, \mathbf{s}) \leq \gamma,
\)
where $d$ is the ambient metric and $\gamma$ is chosen such that the metric 
projection is unique, consistent with the nonlinear projection operator defined 
in Section~\ref{sec:nonlinear-projection}.
\end{postulate}

Postulate 1 specifies the ideal structure of a learned representation: under suitable learning dynamics, the target set is expected to transition from a single manifold $\mathcal{M}_D$ to a compact union of lower-dimensional submanifolds $\mathcal{M}_{D_i}$.  Such compaction reduces the cardinality of 
the target set and thereby increases the likelihood that the learned mapping 
is injective when restricted to the data manifold. This formulation is directly motivated by the geometric insight summarized in Summary~1: reducing the effective cardinality of the target set enables stable and efficient embeddings. However, it is essential to note that Postulate~1 does not describe the \emph{nature} of the mapping that realizes this compact representation.  
It specifies where the representation should reside, but not how 
a neural network maps inputs into that structure.  Postulate~2 fills this 
conceptual gap by characterizing a neural network as a nonlinear orthogonal projection 
operator whose range is precisely the compact union of submanifolds dictated 
by Postulate~1.  Together, the two postulates state that ideal learning seeks 
a compact geometric target and that the network behaves as an nonlinear orthogonal  projection onto 
this target.

The projection nature of the learned mapping described in Postulate~2 can be 
directly observed in the toy examples of Figures~\ref{fig:toy} and 
\ref{fig:anomalies}. In the non-compact case, the neural network behaves as a piecewise orthogonal projection onto $\mathrm{span}(D)$, where $D = [d_1, d_2]$. In contrast, when the learned representation is compact, the network acts as a orthogonal projection onto a union of subspaces $\mathrm{span}(d_1) \cup \mathrm{span}(d_2)$, assigning each input to its corresponding component. Thus, the figures visually confirms that the learned operator functions 
as a projection whose range is determined precisely by the compact geometric 
structure prescribed by Postulate~1. 

\begin{postulate}[Orthogonal Complements via Residual Connection]
The projection operator $P \colon \mathbb{R}^n \rightarrow \mathcal{M}_D$ of 
Postulate~2 admits a residual decomposition
\[
    P = I - P^{\perp},
\]
where $I$ denotes the identity operator on $\mathbb{R}^n$ and 
$P^{\perp} = f_D^{\perp} \circ f_E^{\perp}$ is a learned operator approximating orthogonal projection onto the normal bundle satisfying
\[
    f_E^{\perp} \colon \mathbb{R}^n \to \mathbb{R}^{r_\perp},
    \qquad
    f_D^{\perp} \colon \mathbb{R}^{r_\perp} \to \mathbb{R}^n,
\]
with $r_\perp \ge \max_i (n - k_i)$, where $k_i = \dim(\mathcal{M}_{D_i})$ are the 
dimensions of the submanifolds described in Postulate~1. $P$ is obtained by subtracting this normal component from 
the input.
\end{postulate}

\begin{wrapfigure}{r}{0.41\linewidth}
\vspace{-10pt}
\centering
\includegraphics[width=\linewidth]{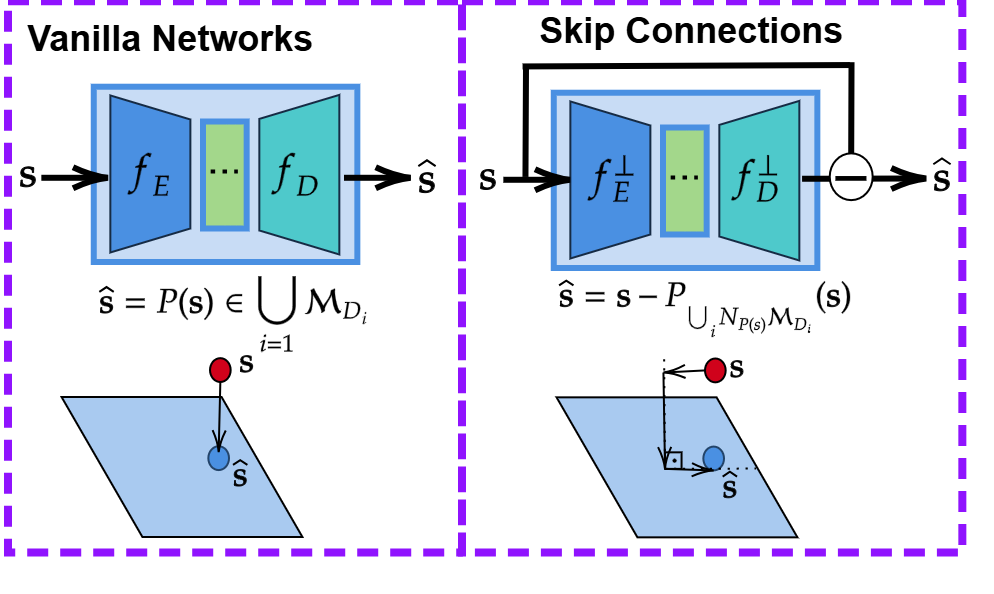}
\vspace{-14pt}
\caption{Vanilla Networks vs.\ Skip Connections.}
\label{fig:skipconnection}
\vspace{-14pt}
\end{wrapfigure}

When $\mathcal{M}_D$ is a single smooth manifold, the operator $P^{\perp}$ 
extracts components orthogonal to the tangent space $T_{P(s)}\mathcal{M}_D$, 
i.e., elements of the normal space $N_{P(s)}\mathcal{M}_D$.  In the special case 
where $\mathcal{M}_D$ is a flat $k$-dimensional subspace, this reduces to the 
classical null–space projection onto the orthogonal complement 
$\operatorname{span}(D)^\perp$, so that $P(s) = s - P^\perp(s)$ is simply the 
standard orthogonal projection onto $\operatorname{span}(D)$.  In the compact 
setting of Postulate~1, where the neural network makes an orthogonal projection to the union of 
submanifolds $\mathcal{M}_{D_i}$, the same operator can also be achieved in the residual setting as projecting 
onto the union of their normal bundles $\bigcup_i N_{P(s)}\mathcal{M}_{D_i}$.  
Thus, the residual form $P = I - P^\perp$ expresses the manifold projection by mirroring residual neural networks, where the skip connection implements the identity map and \textbf{the residual branch models orthogonal projection onto the normal spaces}.

\begin{postulate}[Recursive Application of Nonlinear Projections]
The nonlinear projection operator $P$ introduced in Postulate~1-2 may be applied 
recursively across multiple nested levels of a neural network.  At level $l$, 
the encoder $f_E^{(l)} \colon \mathbb{R}^{n_{l-1}} \to \mathbb{R}^{n_l}$ and decoder 
$f_D^{(l)} \colon \mathbb{R}^{n_l} \to \mathbb{R}^{n_{l-1}}$ implement a nonlinear 
projection onto a representation manifold $\mathcal{M}_D^{(l)} \subset 
\mathbb{R}^{n_{l-1}}$.  Under ideal learning, each such manifold admits a 
compact refinement into a union of lower-dimensional submanifolds,
\[
    \mathcal{M}_D^{(l)} 
    \xrightarrow{\;\text{Opt}\;} 
    \bigcup_{i=1}^{M_l} \mathcal{M}_{D_i}^{(l)},
\qquad
\dim(\mathcal{M}_{D_i}^{(l)}) < \dim(\mathcal{M}_D^{(l)}),
\]
mirroring the compactness principle of Postulate~1 at every level. The global encoder and decoder are then compositions of these nested operators,
\(
f_E = f_E^{(1)} \circ \cdots \circ f_E^{(L)}, \text{ and }
f_D = f_D^{(L)} \circ \cdots \circ f_D^{(1)},
\)
and the full projection operator is
\begin{equation}
    P = f_D^{(L)} \circ \cdots \circ f_D^{(0)} 
    \circ f_E^{(0)} \circ \cdots \circ f_E^{(L)}. \label{nestedEq0}
\end{equation}
\end{postulate}

Postulates~1–3 describe the geometric nature of the ideal representation learned 
by a neural network: the target set collapses to a union of lower-dimensional 
submanifolds, and the network acts as a nonlinear projection onto this compact 
structure (or onto its normal directions in the residual view).  What these 
postulates do \emph{not} explain is \emph{how} such a representation is 
progressively discovered during learning, nor why a deep network should 
spontaneously separate, prune, and refine submanifolds across layers.  
Postulate~4 addresses this missing mechanism. As seen from Eq.~\eqref{nestedEq0}, the emergence of the nested structure in Figure~\ref{fig:classicalvsPoSAnalysis} becomes clearer.

It proposes that each level $\ell$ applies its own nonlinear projection, 
refining the previous level’s representation as illustrated conceptually in 
Figure~\ref{fig:classicalvsPoSAnalysis}. Let $x^{(\ell-1)} \in \mathbb{R}^{N_{\ell-1}}$ 
denote the representation at level $\ell-1$, and $x^{(\ell)} = f_E^{(\ell)}(x^{(\ell-1)})$ 
the corresponding latent coordinates. The decoder $f_D^{(\ell)}$ maps $x^{(\ell)}$ 
back to the representation space of the previous level, yielding 
$\hat{x}^{(\ell-1)} = f_D^{(\ell)}(x^{(\ell)})$, and thereby implements a projection 
onto a refined manifold $\mathcal{M}_D^{(\ell)} \subset \mathbb{R}^{N_{\ell-1}}$. Across levels, these projections become increasingly selective: directions that are 
incoherent or orthogonal to the dominant geometric structure at level $\ell$ are 
suppressed, effectively eliminating incompatible components of the representation. 
As depth increases, the representation is progressively refined into a more compact 
union of submanifolds. In this view, each level operates on a representation space that itself admits a 
manifold structure, enabling the recursive application of the projection principle. 
Thus, a deep network is not merely a composition of functions, but a hierarchy of 
nested projection operators acting on successively refined coordinate spaces.


The four postulates establish a geometric and operational framework for understanding 
deep neural networks as hierarchical projection systems. While they are motivated by 
empirical observations, their primary value lies in the structural constraints they impose 
on learned representations and transformations.

In the following, we derive theoretical consequences of this framework. In particular, 
we show that the PoS formulation leads to fundamental implications for representation 
structure, sample complexity, and generalization. These results formalize how hierarchical 
projection and equivariance enable deep networks to efficiently represent complex data 
distributions.
\begin{wrapfigure}{r}{0.48\linewidth}
\vspace{-10pt}
\centering
\includegraphics[width=\linewidth]{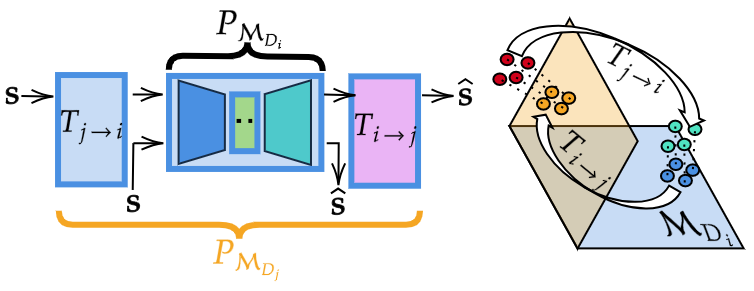}
\vspace{-14pt}
\caption{Projection on sub-manifold \(\mathcal{M}_{D_j}\) can be realized via its own encoder-decoder or via isometric mapping from \(\mathcal{M}_{D_j}\) to \(\mathcal{M}_{D_i}\).}
\label{fig:isometry_I}
\vspace{-10pt}
\end{wrapfigure}

\setcounter{remark}{0}
\begin{remark}[Isometry Invariance of Nonlinear Orthogonal Projection]
\label{Remark 1}
(See~\cite{dudek1994nonlinear}.)
Let $\mathcal{M}$ be a nonempty subset of the Euclidean space $\mathbb{R}^n$,
not necessarily a manifold; in particular, $\mathcal{M}$ may be a single smooth
submanifold or a union of such submanifolds.  
Let $T \colon \mathbb{R}^n \to \mathbb{R}^n$ be an isometry.
If $P \colon \mathbb{R}^n \to \mathcal{M}$ denotes the nonlinear orthogonal projection 
onto $\mathcal{M}$, then the conjugated mapping
\[
T \circ P \circ T^{-1}
\]
is the nonlinear orthogonal projection onto the isometric image $T(\mathcal{M})$.
\end{remark}

\begin{remark}[Projection Transfer Under Isometries]
\label{Remark 2}
Let $\mathcal{M}_{D_i},\,\mathcal{M}_{D_j} \subset \mathbb{R}^n$ be two submanifolds 
(or unions of submanifolds).  
Assume there exists an ambient isometry 
\[
T \colon \mathbb{R}^n \to \mathbb{R}^n
\qquad\text{such that}\qquad
T(\mathcal{M}_{D_i}) = \mathcal{M}_{D_j} .
\]
Define 
\[
T_{i\to j} \coloneq T|_{\mathcal{M}_{D_i}}, 
\qquad 
T_{j\to i} \coloneq T^{-1}|_{\mathcal{M}_{D_j}},
\]
so that $T_{j\to i}$ is the inverse of $T_{i\to j}$ on the submanifolds. Let $P_{\mathcal{M}_{D_i}}$ and $P_{\mathcal{M}_{D_j}}$ denote the nonlinear orthogonal 
projections onto $\mathcal{M}_{D_i}$ and $\mathcal{M}_{D_j}$, respectively.
Then, by the isometry invariance of nonlinear orthogonal projection,
\(
P_{\mathcal{M}_{D_j}}
= 
T \circ P_{\mathcal{M}_{D_i}} \circ T^{-1},
\)
and therefore on the submanifolds,
\[
P_{\mathcal{M}_{D_j}}
= 
T_{i\to j} \circ P_{\mathcal{M}_{D_i}} \circ T_{j\to i}.
\]
Thus, an ideal network representing $P_{\mathcal{M}_{D_i}}$ together with $T_{i\to j}$ 
can equivalently represent the projection onto $\mathcal{M}_{D_j}$.
\end{remark}

\begin{figure}[h]
\centering

\begin{minipage}{0.48\linewidth}
\begin{remark}[Isometry Action on Unions of Submanifolds]
\label{Remark 3}
Let $\mathcal{M} = \bigcup_{i=1}^L \mathcal{M}_{D_i} \subset \mathbb{R}^n$ be a 
finite union of embedded submanifolds, and let 
$g : \mathbb{R}^n \to \mathbb{R}^n$ be an isometry.  
Then
\[
g(\mathcal{M}) = \bigcup_{i=1}^L g(\mathcal{M}_{D_i}),
\]
and the nonlinear orthogonal projection satisfies the invariance identity
\[
P_{g(\mathcal{M})}
= 
g \circ P_{\mathcal{M}} \circ g^{-1}.
\]
Consequently, if $P_{\mathcal{M}}(x) \in \mathcal{M}_{D_i}$, then
\[
P_{g(\mathcal{M})}(g(x)) \in g(\mathcal{M}_{D_i}).
\]
Thus an isometry acts componentwise on the projection operator.
\end{remark}
\end{minipage}
\hfill
\begin{minipage}{0.48\linewidth}
\centering
\includegraphics[width=\linewidth]{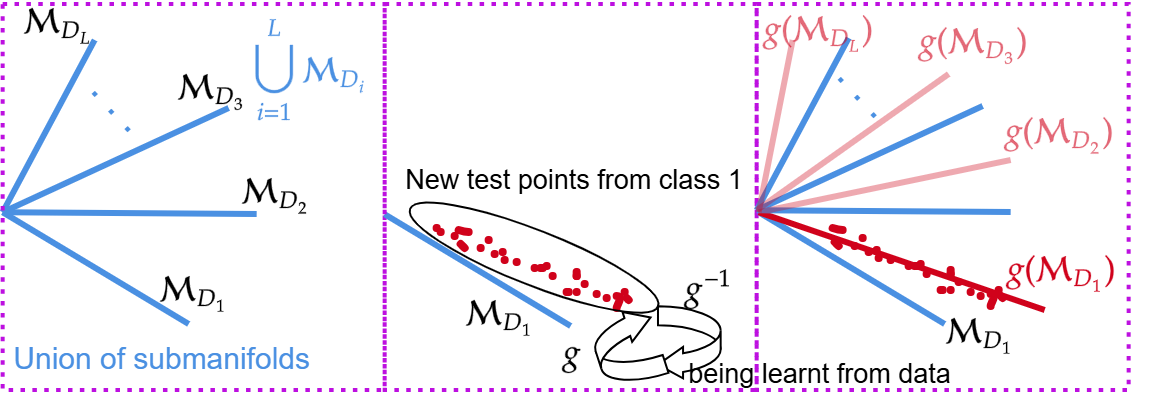}
\captionof{figure}{Isometry action on a union of submanifolds. 
Left: a finite union $\mathcal{M}=\bigcup_{i=1}^L \mathcal{M}_{D_i}$, where $\mathcal{M}_{D_1}$ is already learned as a canonical component. 
Middle: new samples (red points) do not lie on $\mathcal{M}_{D_1}$ but are assumed to belong to an isometric image $g(\mathcal{M}_{D_1})$. 
Learning $g^{-1}$ maps these samples back onto $\mathcal{M}_{D_1}$, implicitly identifying the transformed manifold $g(\mathcal{M}_{D_1})$. 
Right: by the invariance identity $P_{g(\mathcal{M})}=g\circ P_{\mathcal{M}}\circ g^{-1}$, the learned transformation acts componentwise.}
\label{fig:IsometryonSubmanifolds}
\end{minipage}

\end{figure}

A geometric interpretation of this invariance identity is depicted in Figure~~\ref{fig:IsometryonSubmanifolds}, where learning $g^{-1}$ from samples near one component induces the corresponding isometric images across the entire union.

\section{Geometric Consequences of PoS}

\subsection{Overview}
The postulates imply concrete geometric consequences beyond representation structure. 
In particular, when data are modeled as a union of submanifolds, the learned nonlinear 
projection enforces a transversal decomposition of tangent spaces, separating shared 
and residual components. Under this decomposition, orthogonality (or approximate 
incoherence) of residual directions naturally emerges as a sufficient condition for 
stable projection and submanifold selection.

This perspective provides a geometric explanation for why neural networks avoid collapsing 
representations onto the joint span and instead preserve union structure. Moreover, common 
mechanisms such as masking, noise injection, and dropout can be interpreted as enforcing 
selective annihilation of residual directions, promoting orthogonality and compactness. 
In the linearized setting, this connects to restricted isometry and mutual coherence, while 
in the nonlinear regime it extends to tangent spaces of submanifolds. 

In the following sections, we make these consequences explicit by analyzing disentanglement, 
orthogonality, and activation mechanisms such as ReLU under the PoS framework.

\subsection{Disentanglement of two Submanifolds within network}
\label{Disentanglement}

\begin{figure}[h]
\centering
\begin{minipage}{0.55\linewidth}
    \centering
    \includegraphics[width=\linewidth]{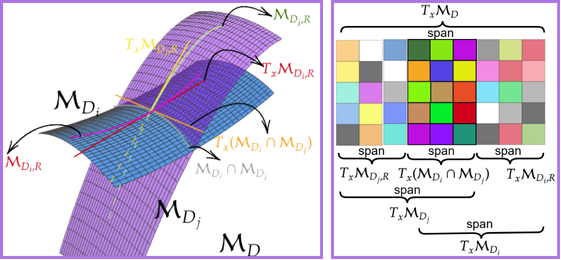}
\end{minipage}\hfill
\begin{minipage}{0.4\linewidth}
    \caption{
    Geometric disentanglement of two submanifolds. 
    Left: The tangent space at the intersection 
    $T_x(\mathcal{M}_{D_i}\cap\mathcal{M}_{D_j})$ splits into submanifold-specific 
    residual directions $T_x\mathcal{M}_{D_i,R}$ and $T_x\mathcal{M}_{D_j,R}$, which 
    a trained network pushes to be \emph{as orthogonal as possible} to ensure stable, 
    unambiguous projection. 
    Right: Linearized view where tangent spaces are approximated by subspace spans 
    $\operatorname{span}(D_i)$, linking the geometric decomposition to group sparsity models (Figure~\ref{fig:sparsity}).
    }
    \label{fig:disentanglement1}
\end{minipage}
\end{figure}
Postulates~1–3 describe the behavior of an ideal neural network after training:
the network converges to a projection operator whose range is a compact union of
low–dimensional submanifolds, realized through nonlinear projection and
residual–style decomposition.  
In the sequel, we introduce a conjecture explaining why
\emph{orthogonality of geometric components} naturally emerges from these
postulates.  
To motivate this, we begin by analyzing the geometric relationship between two
submanifolds representing, for instance, two different classes of the same
dataset and the manner in which a trained network learns to project data onto
their union.

Let $\mathcal{M}_{D_i},\,\mathcal{M}_{D_j}\subset\mathbb{R}^n$ be smooth embedded
submanifolds.  
We conjecture that, for an ideal network satisfying Postulates~1–3, the learned
projection onto $\mathcal{M}_{D_i}\cup \mathcal{M}_{D_j}$ implicitly enforces a
\emph{transversal decomposition} at every point
$x\in \mathcal{M}_{D_i}\cap\mathcal{M}_{D_j}$, namely
\[
\begin{aligned}
T_x\mathcal{M}_{D_i}
&=
T_x(\mathcal{M}_{D_i}\cap\mathcal{M}_{D_j})
\;\oplus\;
T_x\mathcal{M}_{D_i,R},
\\[4pt]
T_x\mathcal{M}_{D_j}
&=
T_x(\mathcal{M}_{D_i}\cap\mathcal{M}_{D_j})
\;\oplus\;
T_x\mathcal{M}_{D_j,R},
\end{aligned}
\]
where $T_x\mathcal{M}_{D_i,R}$ and $T_x\mathcal{M}_{D_j,R}$ denote the residual
(submanifold-specific) tangent subspaces at $x$, as illustrated in Figure~~\ref{fig:disentanglement1}. To capture this structure, consider a neural network whose encoder and decoder
are decomposed according to these geometric components:
\[
F_E = \bigl[\,F_{E,i}\;\; F_{E,j}\,\bigr], 
\qquad
F_D = 
\begin{bmatrix}
F_{D,i} \\[2pt] F_{D,j}
\end{bmatrix},
\]
where $F_E$ computes coordinate representations associated with each
submanifold, and $F_D$ reconstructs the ambient signal from these coordinates.
In an unsupervised setting, the ideal network disentangles the two submanifolds
precisely when the residual tangent directions are orthogonal: \(
T_x\mathcal{M}_{D_i,R}\;\perp\; T_x\mathcal{M}_{D_j,R}. \)
Orthogonality guarantees uniqueness of projection, masks irrelevant coordinate
components, and separates latent geometric factors. For example, if the manifolds are approximated by linear subspaces with a
dictionary \(
D = \bigl[\, D_{R_j}\;\; D_{\cap}\;\; D_{R_i} \bigr]^{\!\top}, \)
whose blocks are mutually orthogonal, then
\(
\operatorname{span}(D_{\cap}, D_{R_i})
~\text{and}~
\operatorname{span}(D_{\cap}, D_{R_j})
\)
provide linear models for $\mathcal{M}_{D_i}$ and $\mathcal{M}_{D_j}$,
respectively.  
In this setting, each sample activates only the residual block corresponding to its own submanifold, yielding clean geometric disentanglement. Nevertheless, such perfect orthogonality is rarely encountered in practice. However, approximately orthogonal representations in the residual components either in the Euclidean sense or with respect to a Riemannian metric naturally emerge in neural networks, as discussed in the sequel.

\subsection{Orthogonality, Restricted Isometry, and Masked Representation}
\label{Orthogonality}

While orthogonality constraints~\cite{orthogonality1} in deep networks have received increasing attention, most existing approaches are empirical and do not address the origin of structural orthogonality across layers. In contrast, the PoS framework offers a geometric explanation: mapping data to a union of submanifolds induces incoherent tangent representations. This perspective unifies dropout~\cite{dropout1}, noise injection~\cite{denoisedAutoencoders, noise-injection2}, masked learning~\cite{masked1, masked2}, distance-preserving mappings, and orthogonality~\cite{orthogonality1, orthogonality2} within a geometric framework. The classical RIP enforces pairwise separation, e.g., SR enforces near-orthogonality among atom subsets, and the submanifold model extends these principles to a hierarchical RIP condition on tangent bases, thereby providing a theoretical motivation for orthogonality in deep networks as we will be discussing in the following two sections.

Although an isometric representation is desirable, particularly along residual directions, guaranteeing this property or estimating the RIP constant is NP-hard, even in piecewise-linear settings such as unions of subspaces encountered in SR. Fortunately, SR theory employs a tractable surrogate: minimizing the mutual coherence \(\mu(\mathbf{D}) = \max_{i \neq j}
\frac{|\langle \mathbf{d}_i, \mathbf{d}_j \rangle|}{|\mathbf{d}_i|_2 |\mathbf{d}_j|_2},\)
where ${\mathbf{d}_i}$ denotes the atoms of the dictionary $\mathbf{D}$. Under the union-of-submanifolds assumption for neural networks, as formalized in Postulate~I, transversality between paired submanifolds leads to the natural emergence of incoherence in the tangent bases of residual components, as we will see in an illustrative example in the sequel.

Let us start from a simpler union-of-subspaces setting, in which the dictionary 
$\mathbf{D} \in \mathbb{R}^{n \times N}$ is partitioned into groups of atoms 
$\mathbf{D}_{\Lambda_i} \in \mathbb{R}^{n \times k_i}$, with $k_i < n$. 
The following analysis extends naturally to the union-of-subspaces where more than two subspaces exist, or even to the union-of-submanifolds setting, 
where each tangent space locally satisfies 
$T_x \mathcal{M}_{D_i} \approx \operatorname{span}(\mathbf{D}_{\Lambda_i})$. Suppose we are given a test sample $\mathbf{s}$ that lies in close proximity to class $i$, 
meaning that $\mathbf{s}$ is near the subspace 
$\operatorname{span}(\mathbf{D}_{\Lambda_i})$. Let $\mathbf{F}_{\Lambda_i} \in \mathbb{R}^{(n-k_i) \times n}$ 
denote the left annihilator (orthogonal complement) of 
$\mathbf{D}_{\Lambda_i}$, satisfying
\(
\mathbf{F}_{\Lambda_i} \mathbf{D}_{\Lambda_i} = \mathbf{0}.
\) Any signal $\mathbf{s}$ admits a decomposition into its components 
along the subspace $\operatorname{span}(\mathbf{D}_{\Lambda_i})$ 
and its orthogonal complement: \( \mathbf{s}
    =
    \mathbf{D}_{\Lambda_i} \mathbf{x}
    +
    \mathbf{F}_{\Lambda_i}^{\!\top} \mathbf{c} \), provided the rows of $\mathbf{F}_{\Lambda_i}$ are orthonormalized, where \( \mathbf{c} = 
    \mathbf{F}_{\Lambda_i} \mathbf{s}   \) is the coefficient vector of the component along left-nullspace, \(\operatorname{span}(\mathbf{F}_{\Lambda_i}^{\!\top})\). In the unsupervised context, when training data are drawn from a single class \(i\), an encoder–decoder architecture estimates the coefficient vector \(\widehat{\mathbf{x}}\) via the encoder \(f_E\), and the decoder \(f_D\) reconstructs the associated subspace representation \(\widehat{\mathbf{D}}_{\Lambda_i}\). Thus, the learning task reduces to identifying a single linear subspace. When data from an additional class \(j\) are included, a central challenge emerges: preventing the network from representations onto the joint span \(\operatorname{span}(\mathbf{D}_{\Lambda_i}, \mathbf{D}_{\Lambda_j})\) rather than preserving the union \(\operatorname{span}(\mathbf{D}_{\Lambda_i}) \cup \operatorname{span}(\mathbf{D}_{\Lambda_j})\). The key question is what mechanisms maintain separation between subspaces in the learned representation, thereby guaranteeing the union-of-subspaces structure.    

Let $\operatorname{span}(\mathbf{D}_{\Lambda_i})$ and 
$\operatorname{span}(\mathbf{D}_{\Lambda_j})$ be subspaces of 
$\mathbb{R}^n$ with nonzero mutual angle, and 
$\mathbf{F}_{\Lambda_i}, \mathbf{F}_{\Lambda_j}$ their respective left-nullspace operators satisfying 
$\mathbf{F}_{\Lambda_i}\mathbf{D}_{\Lambda_i}=0$ and 
$\mathbf{F}_{\Lambda_j}\mathbf{D}_{\Lambda_j}=0$. 
The orthogonal complements of the two subspaces are therefore 
$\operatorname{span}(\mathbf{F}_{\Lambda_i}^{\top})$ and 
$\operatorname{span}(\mathbf{F}_{\Lambda_j}^{\top})$. 
It is clear that the joint dictionary 
$\mathbf{D} = [\mathbf{D}_{\Lambda_i}, \mathbf{D}_{\Lambda_j}]$ 
has orthogonal complement $\operatorname{span}(\mathbf{F}_{\Lambda_{ij}}^{\top})$ satisfying
\(
\operatorname{span}(\mathbf{F}_{\Lambda_{ij}}^{\top}) 
=
\operatorname{span}(\mathbf{F}_{\Lambda_i}^{\top}) 
\cap 
\operatorname{span}(\mathbf{F}_{\Lambda_j}^{\top}).
\)
Therefore, collapsing the nullspaces to $\mathbf{F}_{\Lambda_{ij}}$ eliminates the union structure; preservation requires $\mathbf{F}_{\Lambda_i}$ and $\mathbf{F}_{\Lambda_j}$ to remain distinct. Reconstruction loss alone,
\(
\| \mathbf{s} - f_D(f_E(\mathbf{s})) \|_2,
\)
is insufficient to enforce preservation of class-specific complement directions as discussed in the sequel.

Assume the query signal $\mathbf{s}$ belongs to class \(i\), i.e., it lies in close proximity to the subspace $\operatorname{span}(\mathbf{D}_{\Lambda_i})$. Let $\mathbf{x}^*_{\Lambda_i} \in \mathbb{R}^{k_i}$ be the unique solution corresponding to the closest projection $\widetilde{\mathbf{s}} \in \mathbb{R}^{n}$ of $\mathbf{s}$ onto $\mathbf{D}_{\Lambda_i}$, i.e., $\widetilde{\mathbf{s}} = \mathbf{D}_{\Lambda_i}(\mathbf{D}_{\Lambda_i}^{\top}\mathbf{D}_{\Lambda_i})^{-1}\mathbf{x}^*_{\Lambda_i}$. Then the query $\mathbf{s}$ can be written as
\(
\mathbf{s}
=
\mathbf{D}_{\Lambda_i}\widehat{\mathbf{x}}
+
\mathbf{z}
=
\mathbf{D}_{\Lambda_i}
(\mathbf{D}_{\Lambda_i}^{\top}\mathbf{D}_{\Lambda_i})^{-1}
\mathbf{x}^*_{\Lambda_i}
+
\mathbf{F}_{\Lambda_{ij}}^{\top}\mathbf{c}
+
\mathbf{F}_{\Lambda_i,r}^{\top}\mathbf{c}_r,
\)
where
\(
\widehat{\mathbf{x}}
=
(\mathbf{D}_{\Lambda_i}^{\top}\mathbf{D}_{\Lambda_i})^{-1}\mathbf{x}^*_{\Lambda_i},
\)
and $\mathbf{z}$ lies in the orthogonal complement of $\operatorname{span}(\mathbf{D}_{\Lambda_i})$. The complement basis is decomposed as $\mathbf{F}_{\Lambda_i}=[\mathbf{F}_{\Lambda_{ij}};\mathbf{F}_{\Lambda_i,r}]$. When a first-level encoder
\(
\mathbf{E}
=
\bigl[\mathbf{D}_{\Lambda_i}\;\; \mathbf{D}_{\Lambda_j}\bigr]^{\top}
\)
is applied to $\mathbf{s}$, we obtain
\(
\widetilde{\mathbf{x}}
=
\begin{bmatrix}
\widetilde{\mathbf{x}}_{\Lambda_i} \\
\widetilde{\mathbf{x}}_{\Lambda_j}
\end{bmatrix},
\)
where
\(
\widetilde{\mathbf{x}}_{\Lambda_i}
=
\mathbf{x}^*_{\Lambda_i},
\)
and $\widetilde{\mathbf{x}}_{\Lambda_j}$ has reduced energy, i.e.,
\(
\|\widetilde{\mathbf{x}}_{\Lambda_j}\|_2
\le
\frac{\theta_{k_i,k_j}}{1-\delta_{\Lambda_i}}
\|\mathbf{x}^*_{\Lambda_i}\|_2
+
\sqrt{1-\theta_{k_i,k_j}^2}\,\|\mathbf{c}_r\|_2,
\)
where $\theta_{k_i,k_j}$ is the restricted orthogonality constant between the two subspaces (see Appendix~\ref{app:roc} for the definition). The proof of the first term is straightforward:
\(
\mathbf{D}_{\Lambda_i}^{\top}
\left(
\mathbf{D}_{\Lambda_i}
(\mathbf{D}_{\Lambda_i}^{\top}\mathbf{D}_{\Lambda_i})^{-1}
\mathbf{x}^*_{\Lambda_i}
+
\mathbf{F}_{\Lambda_i}^{\top}\mathbf{c}
\right)
=
\mathbf{x}^*_{\Lambda_i},
\)
since
\(
\mathbf{D}_{\Lambda_i}^{\top}\mathbf{F}_{\Lambda_i}^{\top}=0.
\)
(See Appendix~\ref{app:crossproof} for the proof of the second term.) It is clear that for case one, when the two subspaces are orthogonal, i.e.,
\(
\theta_{k_i,k_j}=0,
\)
and the component $\mathbf{F}_{\Lambda_i,r}^{\top}\mathbf{c}_r$ does not exist, the reconstructed signal
\(
\hat{\mathbf{s}}=\mathbf{D}\widetilde{\mathbf{x}}=\widetilde{\mathbf{s}}
\)
lies entirely in
\(
\operatorname{span}(\mathbf{D}_{\Lambda_i})
\)
instead of
\(
\operatorname{span}(\mathbf{D}).
\)
However, in general it is very rare in practice for data to exhibit such ideal orthogonality in Euclidean space. Moreover, there is no explicit mechanism that encourages the encoder to further annihilate directions in the remaining nullspace component
\(
\operatorname{span}(\mathbf{F}_{\Lambda_i}^{\top})
\setminus
\operatorname{span}(\mathbf{F}_{\Lambda_{ij}}^{\top}).
\)
In particular, the component
\(
\widetilde{\mathbf{x}}_{\Lambda_j}
\)
may still retain non-negligible energy, which yields a reconstructed output
\(
\hat{\mathbf{s}}
\in
\operatorname{span}(\mathbf{D}),
\)
e.g., a projection onto the entire plane rather than onto the union of the corresponding bases, as can be seen in the toy example (non-compact learning) in Figures~\ref{fig:toy} and 
\ref{fig:anomalies}.

\begin{figure}[h]
    \centering
 \includegraphics[width=0.99\linewidth]{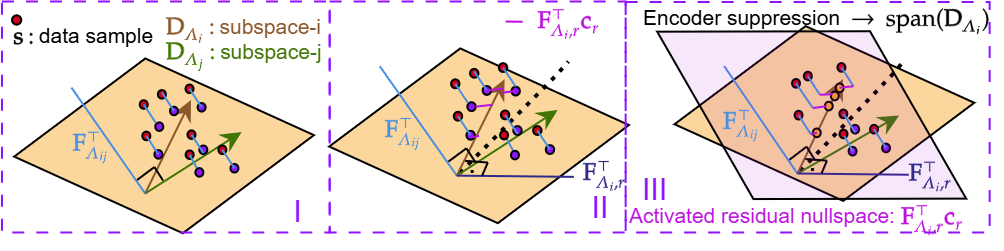}
    \vspace{-0.3cm}
    \caption{Geometric illustration of residual-nullspace interference and selective annihilation. \textbf{Structured degradation reveals hidden geometric directions that are otherwise suppressed, enabling the network to identify and eliminate residual nullspace components.}
(I) Projection onto the joint span $\operatorname{span}(\mathbf{D})$ removes only the shared nullspace 
$\operatorname{span}(\mathbf{F}_{\Lambda_{ij}}^{\top})$, leaving class-specific residual directions unresolved. 
(II) A sample from class $i$ generally contains components along both $\mathbf{F}_{\Lambda_{ij}}^{\top}$ and the residual nullspace 
$\mathbf{F}_{\Lambda_i,r}^{\top}$, causing coefficient leakage $\widetilde{\mathbf{x}}_{\Lambda_j}$. 
(III) Structured degradation activates residual nullspace directions, forcing the encoder to suppress them and thereby encouraging projection onto the correct subspace $\operatorname{span}(\mathbf{D}_{\Lambda_i})$.}
    \label{fig:maskedLearning1}
\end{figure}

Crucially, components that remain latent under standard reconstruction become observable once perturbed, either in the input space or within intermediate latent representations, forcing the encoder to explicitly suppress them across levels. To induce selective annihilation during unsupervised training, the input is degraded as 
$\mathbf{s}_{\mathrm{deg}}=\mathbf{s}+\mathbf{z}_{\mathrm{deg}}$ and the reconstruction loss 
$\|\mathbf{s}-f_D(f_E(\mathbf{s}_{\mathrm{deg}}))\|_2$ is minimized. Using the previous decomposition, we write
\(
\mathbf{s}_{\mathrm{deg}}
=
\mathbf{D}_{\Lambda_i}(\mathbf{D}_{\Lambda_i}^{\top}\mathbf{D}_{\Lambda_i})^{-1}\mathbf{x}^*_{\Lambda_i}
+
\mathbf{F}_{\Lambda_{ij}}^{\top}(\mathbf{c}+\mathbf{c}_{\mathrm{deg},ij})
+
\mathbf{F}_{\Lambda_i,r}^{\top}(\mathbf{c}_r+\mathbf{c}_{\mathrm{deg},r}),
\)
where $\mathbf{z}_{\mathrm{deg}}=\mathbf{F}_{\Lambda_{ij}}^{\top}\mathbf{c}_{\mathrm{deg},ij}+\mathbf{F}_{\Lambda_i,r}^{\top}\mathbf{c}_{\mathrm{deg},r}$ denotes the injected degradation component in the orthogonal complement of $\operatorname{span}(\mathbf{D}_{\Lambda_i})$. The reconstruction objective implicitly forces the encoder to suppress the degradation component, i.e.,
\(
\mathbf{E}\bigl(
\mathbf{F}_{\Lambda_i,r}^{\top}\mathbf{c}_{\mathrm{deg},r}
\bigr)\rightarrow 0,
\)
thereby activating annihilating directions inside the remaining nullspace
\(
\operatorname{span}(\mathbf{F}_{\Lambda_i}^{\top})
\setminus
\operatorname{span}(\mathbf{F}_{\Lambda_{ij}}^{\top}).
\)
When the degradation subspace is structured within this residual nullspace, the leakage term
\(
\widetilde{\mathbf{x}}_{\Lambda_j}
\)
is further attenuated, preventing mapping the input onto the joint span
\(
\operatorname{span}(\mathbf{D}_{\Lambda_i},\mathbf{D}_{\Lambda_j})
\)
and encouraging projection onto the desired union structure. Thus, structured degradation provides a geometric mechanism for enforcing compact representations consistent with the compact learning axiom. 
This degradation can be applied either in the input space (e.g., masked autoencoders) or across intermediate latent representations in a nested manner (as in Postulate~4), e.g., via dropout. 
\textbf{In summary, structured degradation reveals latent residual nullspace directions and forces their selective suppression.}

In practice, the degradation component $\mathbf{z}_{\mathrm{deg}}$ cannot be assumed to lie in $\operatorname{span}(\mathbf{F}_{\Lambda_i}^{\top})$, since both $\operatorname{span}(\mathbf{D}_{\Lambda_i})$ and its complement $\operatorname{span}(\mathbf{F}_{\Lambda_i}^{\top})$ are unknown and must be learned from data. Instead, degradation is sampled from a design subspace $\mathcal{U}$ such that $\mathbf{z}_{\mathrm{deg}}\in\mathcal{U}$, where $\mathcal{U}$ is chosen to contain residual directions with high probability while remaining separated from signal subspaces. For faithful recovery, the encoder must suppress the perturbation while preserving the signal component, i.e., $\mathbf{E}\mathbf{z}_{\mathrm{deg}}\approx\mathbf{0}$, which implicitly encourages $\mathcal{U}$ to align with the residual space, $\mathcal{U}\approx\operatorname{span}(\mathbf{F}_{\Lambda_i}^{\top})$ for samples from class $i$. The geometry of $\mathcal{U}$ is therefore critical: isotropic noise provides weak directional selectivity, while overly concentrated perturbations may suppress informative components. In contrast, structured degradations such as masking or band-limited corruption restrict $\mathcal{U}$ to controlled directions, promoting selective annihilation of residual subspaces without collapsing informative components. This perspective unifies operators such as dropout, structured noise injection, and masked autoencoding as mechanisms that induce controlled annihilation of residual directions; motivated by this view, in the Section~\ref{sec:ZeroShotECG} we advocate randomized multi-scale degradation, where varying masking supports increase the probability of alignment with true residual nullspaces while preserving signal-supporting directions.

\begin{Summary}{}{secondsummary}
\label{secondsummary}
\textbf{Although projection onto a union of subspaces is desirable, reconstruction objectives alone may collapse representations onto the joint span, failing to separate submanifold components,} i.e., residual nullspace directions induce leakage across subspaces. This ambiguity is resolved by structured degradation, which reveals and selectively suppresses these directions, thereby enforcing projection onto the correct subspace and enabling compact, disentangled representations.
\end{Summary}

\subsection{ReLU and Compact Representation via Angular Separation}
\label{app:CompactRepresentation}
We illustrate how compact representations emerge in a simple setting using toy data sampled around two one-dimensional subspaces with basis vectors $\mathbf{d}_1,\mathbf{d}_2 \subset \mathbb{R}^3$, as shown in Figure~~\ref{fig:toy}(a). Let a sample $\mathbf{s}$ from class $i$ admit the decomposition $\mathbf{s} = \mathbf{D}_{\Lambda_i}(\mathbf{D}_{\Lambda_i}^{\top}\mathbf{D}_{\Lambda_i})^{-1}\mathbf{x}^*_{\Lambda_i} + \mathbf{F}_{\Lambda_{ij}}^{\top}\mathbf{c} + \mathbf{F}_{\Lambda_i,r}^{\top}\mathbf{c}_r$, consisting of the component in the class subspace $\operatorname{span}(\mathbf{D}_{\Lambda_i})$, the shared nullspace component, and the class-specific residual nullspace. Applying the linear encoder $\mathbf{E}=[\mathbf{D}_{\Lambda_i}\;\mathbf{D}_{\Lambda_j}]^{\top}$ produces coefficients $\widetilde{\mathbf{x}}=[\widetilde{\mathbf{x}}_{\Lambda_i}^{\top}\;\widetilde{\mathbf{x}}_{\Lambda_j}^{\top}]^{\top}$. When the angle between the two subspaces exceeds $90^\circ$, the cross-subspace coefficients become negative, allowing ReLU to suppress them and enforce projection onto the correct subspace. \footnote{When the angle between the two subspaces exceeds $90^\circ$, the cross-subspace coefficient $\widetilde{\mathbf{x}}_{\Lambda_j}$ becomes negative provided that the residual term $\mathbf{F}_{\Lambda_i,r}^{\top}\mathbf{c}_r$ is sufficiently small, and the ReLU activation therefore annihilates this component.} A symmetric argument applies when the sample lies near $\operatorname{span}(D_{\Lambda_j})$, leading to suppression of $\tilde{x}_{\Lambda_i}$. Consequently, the encoder activates only the coefficients associated with the correct subspace, mapping samples onto the union of subspaces rather than their joint span, as illustrated in Figure~~\ref{fig:toy}(a-bottom row). \textbf{ReLU promotes compact representations precisely when subspaces are sufficiently separated, i.e., when their angle induces negative cross-coefficients that are suppressed to enforce projection onto the correct subspace.}

Consider now the case where the angle between the two subspaces is smaller than $90^\circ$, as illustrated in Figure~~\ref{fig:rotation_module}(a). In this regime the cross-subspace coefficients produced by the encoder $\mathbf{E}=[\mathbf{D}_{\Lambda_i}\;\mathbf{D}_{\Lambda_j}]^{\top}$ are no longer guaranteed to be negative. Consequently, the ReLU activation fails to annihilate $\widetilde{\mathbf{x}}_{\Lambda_j}$ for samples near $\operatorname{span}(\mathbf{D}_{\Lambda_i})$, and similarly fails to suppress $\widetilde{\mathbf{x}}_{\Lambda_i}$ for samples near $\operatorname{span}(\mathbf{D}_{\Lambda_j})$. As a result, both coefficient groups remain active and the encoder behaves effectively as a linear projection onto the joint span $\operatorname{span}(\mathbf{D}_{\Lambda_i},\mathbf{D}_{\Lambda_j})$, producing the non-compact representation shown in Figure~~\ref{fig:rotation_module}(a).

\begin{figure}[h]
    \centering
    \includegraphics[width=0.55\linewidth]{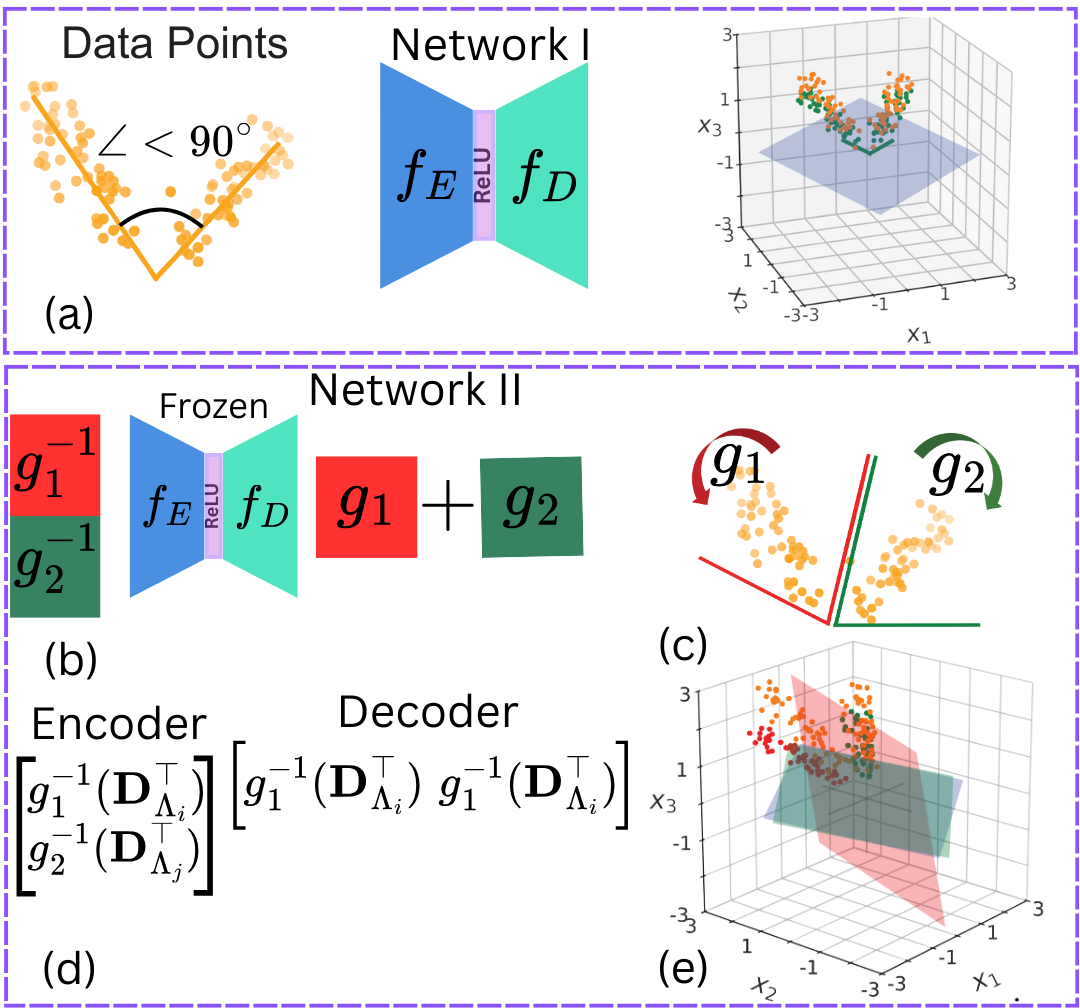}
    \vspace{-0.1cm}
    \caption{(a) When the angle between the two subspaces is $<90^\circ$, a ReLU autoencoder projects samples onto the joint span, yielding a non-compact representation. (b–d) Learnable transformations $g_1^{-1}, g_2^{-1}$ are introduced before the frozen network with corresponding adjoints in the decoder. (c–e) These transformations reorient the data so that the samples align with two distinct subspaces, enabling a compact representation.}
    \label{fig:rotation_module}
\end{figure}

Building on Remarks~1-3, which establish that nonlinear orthogonal projection is invariant under ambient isometries, we interpret the architecture in Figure~~\ref{fig:rotation_module}(b) as learning such isometric pre-transformations prior to projection. Specifically, two learnable $3\times3$ encoder blocks are placed in parallel in front of the pretrained network, while their adjoint mappings appear serially in the decoder through summation, consistent with the transversality relation in Eq.~(4). If the resulting $3\times6$ pre-network learns transformations approximating $g_1^{-1}$ and $g_2^{-1}$ that map samples close to the learned mother subspace, then by Remark~3 the overall effect is equivalent to projecting onto the isometric images $g_1(\operatorname{span}(\mathbf{D}_{\Lambda_i},\mathbf{D}_{\Lambda_j}))$ and $g_2(\operatorname{span}(\mathbf{D}_{\Lambda_i},\mathbf{D}_{\Lambda_j}))$. As illustrated in Figure~~\ref{fig:rotation_module}(c), these transformations rotate the coordinate system in different directions, producing four effective bases: two become closer and form the intersection of the transformed planes, while the remaining two separate beyond $90^\circ$ and constitute residual bases. ReLU therefore operates on these residual directions as described in the Disentanglement and Orthogonality sections, suppressing cross-subspace coefficients even when the original representation fails to realize a union-of-subspaces in the two-dimensional feature space. Subsequent feature expansion ($3\times6$ followed by $6\times4$) lifts the representation to a higher-dimensional space where the samples align with one of the two transformed subspaces, as shown in Figure~~\ref{fig:rotation_module}(e). During unsupervised training we additionally mask one of the three input coordinates for each sample, which activates the residual directions and encourages orthogonality between them, reinforcing separation of the learned subspaces.

\begin{Summary}{}{summary3}
\label{summary3}
\textbf{Neural networks can be understood as systems that enforce compact representations either by reorienting the data manifold or by expanding the coordinate system.}
Classical folding interpretations view deep networks as progressively folding affine regions of the input space. In the PoS framework, this behavior is unified through two complementary mechanisms governed by Postulate~4.

When encoder–decoder modules are added outside the learned network (e.g., Figure~\ref{fig:rotation_module}), they implement approximate isometries that reorient the data manifold prior to projection, increasing angular separation of residual directions. When modules are inserted inside the network (Figure~\ref{fig:classicalvsPoSAnalysis}), they act as coordinate transformations in latent space, expanding the representation into an over-complete system. In this regime, tangent bases become increasingly orthogonal, yielding separable submanifolds and structured sparsity (e.g., group sparsity as illustrated in Figure~~\ref{fig:sparsity}). These two perspectives are mathematically equivalent: what appears as manifold reorientation from an external viewpoint is precisely coordinate expansion and orthogonalization in the latent space.
\end{Summary}

\begin{figure}[t]
\centering
\includegraphics[width=0.48\linewidth]{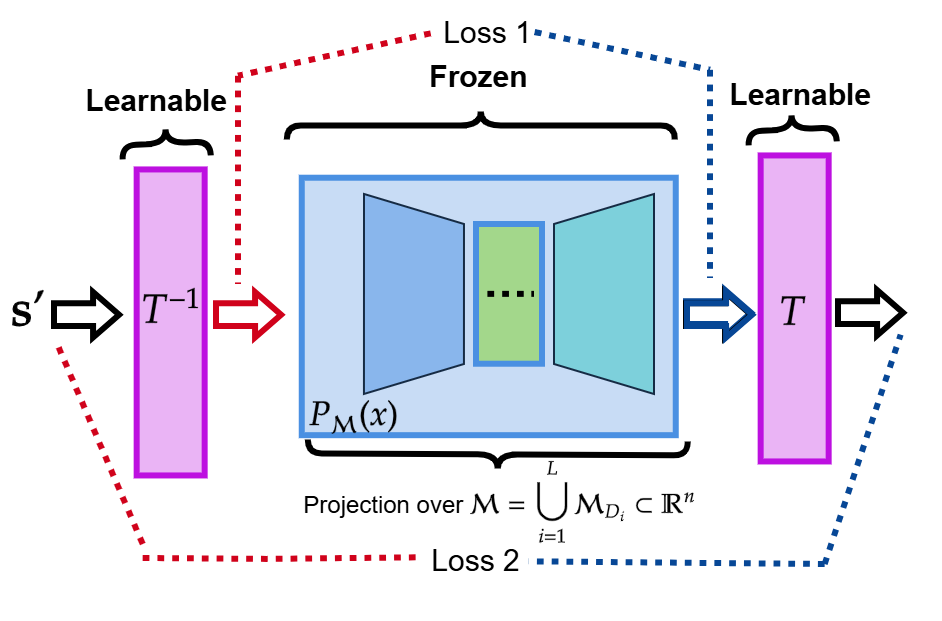}
\hfill
\includegraphics[width=0.48\linewidth]{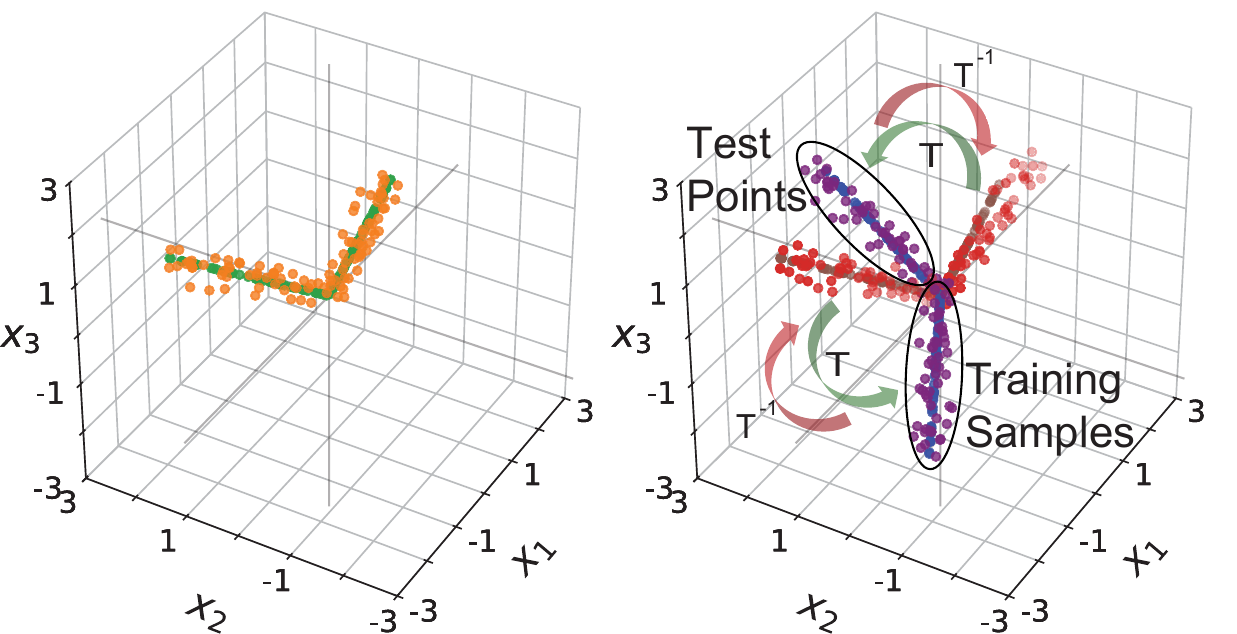}
\vspace{-0.2cm}
\caption{Isometric folding for out-of-union samples. A learnable transform $T^{-1}$ folds inputs toward the nearest submanifold before projection onto $\mathcal M$, while $T$ maps them back, yielding projection onto the transformed union $T(\mathcal M)$.}
\label{fig:IsometryUnion}
\end{figure}

The data and its mapped representation shown in Figure~~\ref{fig:rotation_module}(b) correspond to the compact network introduced in Figure~~\ref{fig:compactness_vs_ood}, where the encoder--decoder learns the projection onto a union of submanifolds. As illustrated in Figure~~\ref{fig:rotation_module}(c), when new test samples arrive that do not lie on this learned union, a nested architecture of the form shown in Figure~~\ref{fig:rotation_module}(a) can be employed to learn a transformation that folds the input data toward the closest submanifold before projection. Let $T^{-1}$ denote the learned transformation and $P_{\mathcal M}$ the projection onto $\mathcal{M}=\bigcup_i \mathcal{M}_{D_i}$. The transformation is trained using the folding loss $J_{\text{fold}}(\theta_T)=\|T^{-1}(s')-P_{\mathcal M} \circ T^{-1}(s')\|^2$, which encourages the transformed samples to lie on the learned union $\mathcal{M}$. In addition, a representation loss $J_{\text{rep}}(\theta_T)=\|s'-T \circ  P_{\mathcal M} \circ  T^{-1}(s')\|^2$ can be introduced to measure reconstruction error on the learned submanifolds. Although both objectives can be used jointly, in practice minimizing $J_{\text{fold}}$ alone is sufficient to enforce the desired folding behavior, as will be demonstrated in Section~\ref{Results}. The adjoint transformation $T$ maps the folded samples back to the original space, effectively learning projection onto the transformed union $T(\mathcal{M})$. Consequently, even if the additional samples arise only from variations near a single component $\mathcal{M}_{D_i}$, the learned transformation implicitly induces the corresponding images $T(\mathcal{M}_{D_j})$ for the remaining components, yielding $T(\mathcal{M})=\bigcup_i T(\mathcal{M}_{D_i})$, consistent with the isometry action described in Remark~3 and illustrated in Figure~~\ref{fig:IsometryonSubmanifolds}.

\section{From Union of Submanifolds to Locally Trivial Fibrations}
\label{F1}

\begin{figure}[t]
\centering
\begin{minipage}{0.55\linewidth}
    \centering
    \includegraphics[width=\linewidth]{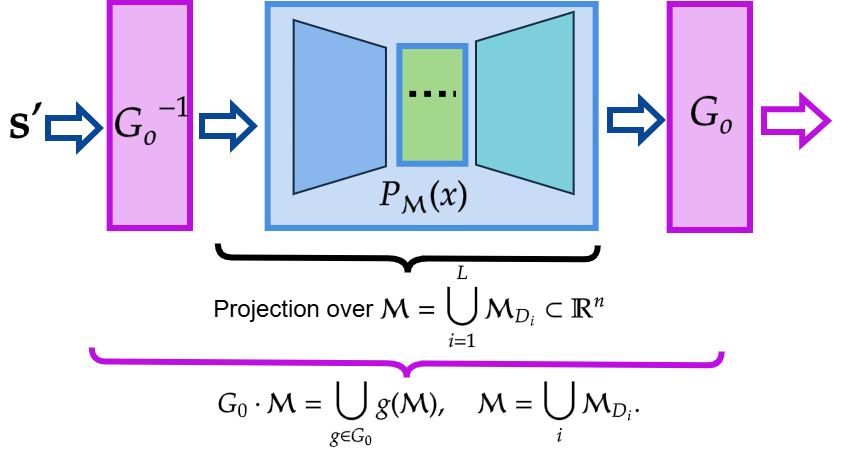}
\end{minipage}\hfill
\begin{minipage}{0.4\linewidth}
    \caption{
    Group-action view of the learned representation: projection onto the canonical manifold union $\mathcal{M}$ and its orbit $G\!\cdot\!\mathcal{M}$ generated by learnable transformations.
    }
    \label{fig:locallytrivial}
\end{minipage}
\end{figure}
Remark~3 shows that a learnable isometry $g\colon \mathbb{R}^n\!\to\!\mathbb{R}^n$ acts componentwise on a union of submanifolds 
\(
\mathcal{M}=\bigcup_{i=1}^{L}\mathcal{M}_{D_i}\subset\mathbb{R}^n,
\)
producing the transformed union
\(
g(\mathcal{M})=\bigcup_{i=1}^{L} g(\mathcal{M}_{D_i}),
\)
while the nonlinear projection satisfies the equivariance identity
\(
P_{g(\mathcal{M})}=g\circ P_{\mathcal{M}}\circ g^{-1}.
\)
In Section \ref{Orthogonality}, we further showed how multiple learned isometries 
$g_k : \mathbb{R}^n \to \mathbb{R}^n$ can act on a canonical “mother” manifold 
(or a union of manifolds) to generate additional components, as illustrated 
in Figure~~15. Each transformation effectively creates a new manifold component $g_k(\mathcal{M}_{D_i})$, corresponding to a new class in the representation space. In realistic settings, the number of such components may be extremely large or even effectively continuous, making explicit enumeration impractical. Instead, the structure induced by these transformations is better understood through symmetry. Rather than enumerating infinitely many submanifolds, this perspective shows that they arise from a canonical manifold (or canonical union) through symmetry transformations. Let $G$ be a group of learnable isometries acting on $\mathbb{R}^n$. The representation then naturally forms the orbit (see Appendix~\ref{app:group_action} for definitions)
\[
G\!\cdot\!\mathcal{M}
=
\bigcup_{g\in G} g(\mathcal{M})=
\bigcup_{g \in G} \; \bigcup_{i} g(\mathcal{M}_{D_i}),
\]
which can be interpreted as a collection of manifolds generated from a canonical component through group actions. Here, $G$ may represent a smooth transformation group acting continuously on the representation space, while in practical or finite settings it may be approximated by a discrete or discretized subset corresponding to a sampled symmetry structure. \textbf{The union-of-submanifolds model is a special case of a symmetry-generated orbit, where discrete components correspond to sampled instances of an underlying continuous family governed by group actions.}

This viewpoint admits a more compact geometric description. Recall that $\mathcal{M}=\bigcup_i \mathcal{M}_{D_i}$ denotes the learned union of submanifolds and let $G$ be a smooth group of transformations acting on the representation space. The family of manifolds generated by these transformations forms the orbit $G\!\cdot\!\mathcal{M}=\bigcup_{g\in G} g(\mathcal{M})$. Rather than treating each $g(\mathcal{M}_{D_i})$ as an independent component, the orbit can be organized through a fiber structure induced by the group action. In particular, if the group action is smooth (e.g., when $G$ is a Lie group), then for a neighborhood $U \subset G$ the orbit locally satisfies $\pi^{-1}(U) \simeq U \times \mathcal{M}$, where $\pi(gx)=g$ denotes the canonical projection. Thus the representation space locally decomposes into intrinsic coordinates on the canonical manifold $\mathcal{M}$ and transformation coordinates parameterized by $G$. This local product structure naturally leads to the notion of a \emph{locally trivial fibration} (see Appendix~\ref{localfibration}), providing a compact description of how symmetry-induced transformations generate the family of learned manifolds. 


\section{Fundamental theorem of deep learning}
\label{app:DeepLearningTheorem}

Building on the orbit–fibration structure introduced above, the same construction can be applied recursively. Let $P_{\mathcal M}$ denote the projection onto the canonical union of submanifolds $\mathcal M=\bigcup_i\mathcal M_{D_i}$. A family of learnable transformations $G_0$ acting on the representation space generates the orbit $G_0\!\cdot\!\mathcal M=\bigcup_{g\in G_0} g(\mathcal M)$, producing additional manifold components. Since the result is again a union of manifolds, the same mechanism can be applied iteratively with new transformation families $G_1,G_2,\ldots, G_L$ as shown in Fig~\ref{fig:DNNs}. Consequently, a deep architecture constructs manifolds in a hierarchical manner, where each stage expands the representation by generating new orbits of the previously obtained structure. In the locally trivial view developed in the previous subsection, the representation space therefore admits a multi-level product structure
\[
\pi^{-1}(U_L \times \cdots \times U_0)\;\simeq\;U_L\times\cdots\times U_0\times\mathcal M,
\]
where $U_\ell\subset G_\ell$ denotes a neighborhood in the transformation manifold at level $\ell$. Thus depth corresponds to progressively introducing new transformation coordinates, yielding a hierarchical family of manifolds generated from a canonical component.

\begin{figure}[h]
    \centering
    \includegraphics[width=0.6\linewidth]{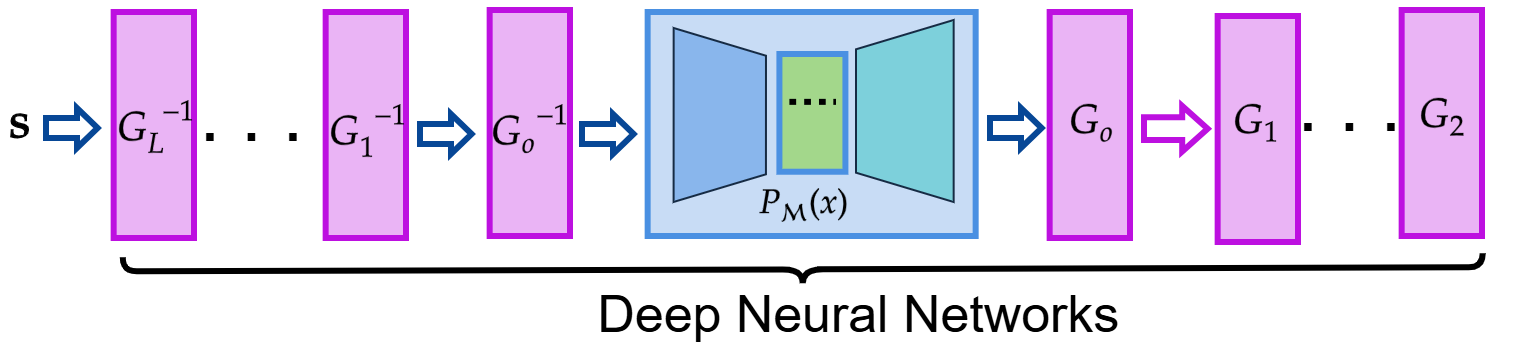}
    \vspace{-0.2cm}
    \caption{Deep networks as hierarchical manifold generators: successive transformation families generate new manifold components from a canonical union $\mathcal{M}$.}
    \label{fig:DNNs}
\end{figure}

To quantify the implications of the hierarchical manifold construction introduced above, we analyze the number of samples required to represent the generated family of manifolds.

\textbf{Deep neural networks reduce the sample complexity of learning a family of manifolds from multiplicative to additive scaling by exploiting hierarchical symmetry structure.} This reveals why deep architectures can generalize across exponentially many transformations while requiring only a linear growth in data.

The following theorem makes this reduction explicit; a complete derivation is provided in Appendix~\ref{appendix:samplecomplexity}.

\setcounter{theorem}{0}
\begin{theorem}[Fundamental Theorem of Deep Learning]
Let $\mathcal M=\bigcup_i \mathcal M_{D_i}$ be a union of compact $k$-dimensional submanifolds and let $G_0,G_1,\ldots,G_L$ denote successive families of learnable transformations acting on the representation space. Classical non-hierarchical learning methods that treat each transformed manifold independently require sampling all components of the generated family
\[
\mathcal M^{(L)}=
G_L\cdot(G_{L-1}\cdot(\cdots(G_0\cdot\mathcal M)\cdots)),
\]
leading to sample complexity proportional to
\[
N_{\mathrm{classical}}
\sim
C_\epsilon(\mathcal M)\prod_{\ell=0}^{L}|G_\ell|,
\]
where $C_\epsilon(\mathcal M)$ denotes the covering number of $\mathcal M$. In contrast, the hierarchical representation induced by deep networks learns transformations incrementally across layers. By the equivariance property described in Remark~4, it suffices to learn the action of each transformation $g\in G_\ell$ on a single representative manifold $\mathcal M_{D_i}$, which generalizes to all components. The required number of samples therefore scales as
\[
N_{\mathrm{DNN}}
\sim
C_\epsilon(\mathcal M)
+
\sum_{\ell=0}^{L}
C_\epsilon(G_\ell)\,C_\epsilon(\mathcal M_{D_i}),
\]
where $C_\epsilon(G_\ell)$ denotes the covering number of the transformation family (and reduces to $|G_\ell|$ in the finite case). Therefore, the sample complexity of hierarchical manifold learning grows additively across layers rather than multiplicatively.
\end{theorem}

Theorem shows that classical learning methods must sample every transformed manifold independently, leading to multiplicative growth in the number of required samples as new transformation families are introduced. In contrast, the hierarchical structure of deep networks separates intrinsic manifold geometry from transformation coordinates, allowing transformations to be learned incrementally across layers. Consequently the required number of samples grows additively across transformation families rather than multiplicatively. This result provides an explicit estimate of the number of samples required to represent the entire manifold family generated by the network, thereby explaining the improved generalization capacity of deep architectures. Specifically, by demonstrating why hierarchical representations can achieve substantially greater sample efficiency than classical non-hierarchical learning methods, the theorem offers a theoretical explanation for the research questions outlined in the Introduction section. These questions include why neural networks exhibit high efficiency and outperform traditional algorithms, as well as how they generalize data while challenging conventional paradigms relating model parameter size to the number of training samples~\cite{zhang2016understanding, belkin2019reconciling}.

We assume that a canonical union of submanifolds $\mathcal M$ has been learned with sample complexity $C_\epsilon(\mathcal M)$. The following analysis characterizes the additional samples required to generalize across transformation families. In practice, the learned representation may require fewer samples than the covering number, as the projection operator induces a compact representation of $\mathcal M$, so that $C_\epsilon(\mathcal M)$ serves as a conservative upper bound.

Furthermore, the above analysis provides an upper bound on transformation learning by assuming that each transformation $g \in G_\ell$ requires sampling the full submanifold $\mathcal M_{D_i}$. However, under the transversality structure described in Section~A and illustrated in Figure~~13, each transformed manifold $g(\mathcal M_{D_i})$ intersects the original manifold $\mathcal M_{D_i}$ along a shared subspace, and the corresponding tangent space at points $x \in \mathcal M_{D_i} \cap g(\mathcal M_{D_i})$ admits a decomposition into intrinsic and residual components. In particular, the tangent space splits as
\(
T_x \mathcal M_{D_i}
=
T_x(\mathcal M_{D_i} \cap g(\mathcal M_{D_i}))
\;\oplus\;
T_x \mathcal M_{D_i,R},
\)
where $T_x \mathcal M_{D_i,R}$ denotes the residual directions specific to $\mathcal M_{D_i}$, and the transformation $g$ acts primarily along these transversal directions. Consequently, learning the transformation $g$ reduces to learning its action along the residual subspace rather than over the entire manifold. Therefore, the effective sample complexity of learning each transformation is governed by the geometry of the residual submanifold $\mathcal M_{D_i,R}$, yielding a scaling proportional to $C_\epsilon(\mathcal M_{D_i,R})$, which is typically much smaller than $C_\epsilon(\mathcal M_{D_i})$. This further reduces the effective sample complexity of the hierarchical representation beyond the upper bounds derived above.

\section{PoS Modules and Local Nonlinear Decomposition}
\label{Appendix PoSModule}

\begin{wrapfigure}{l}{0.49\linewidth}
\vspace{-10pt}
\centering
\includegraphics[width=0.87\linewidth]{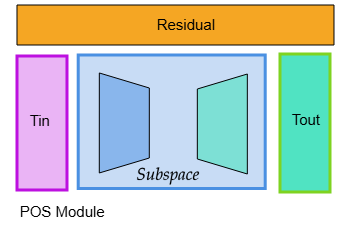}
\vspace{-10pt}
\caption{PoS module. The module consists of an input transformation (Tin), a projection onto a structured subspace, and an output transformation (Tout), together with a residual branch that captures components not explained by the current subspace. This structure serves as a fundamental building block for hierarchical composition in deep networks.}
\label{fig:PoS}
\vspace{-10pt}
\end{wrapfigure}
We now introduce the \emph{Pursuit–of–Subspaces (PoS)} module, illustrated in Figure~\ref{fig:PoS}, as a fundamental building block for deep neural networks. This module provides an operational view of deep architectures, where each component performs a structured projection while capturing unexplained variation through residual pathways. The module consists of three components: an input transformation (Tin), a projection onto a structured subspace, and an output transformation (Tout), together with a residual branch that captures directions not explained by the current subspace.

The key principle is that any neural network can be constructed through the \emph{nested composition} of PoS modules. In this recursive structure, the subspace component of one module may itself be realized by another PoS module, and similarly for the residual, input, or output branches. This allows increasingly complex representations to be built by iteratively refining and augmenting previously learned geometric structures.

Importantly, the notion of ``subspace'' in the PoS module is more general than a single manifold. It may represent:
(i) the intersection of multiple submanifolds (shared geometric structure),
(ii) a union of submanifolds (a learned representation space), or
(iii) a lower-dimensional embedding capturing common features across components.
Conversely, the residual branch aggregates directions that are not captured by the current subspace and can itself be interpreted as a union of residual components associated with individual submanifolds.

The geometric role of residual learning can be understood through a three-stage progression. First, under the assumption that the learned representation spans a sufficiently rich set of directions, orthogonality acts to narrow the target set by projecting onto the components aligned with the target manifold, suppressing irrelevant variations arising from overcomplete representations. Second, transversality refines this picture at the level of individual samples. When
\(
T_x \mathcal X \oplus T_x \mathcal Z = T_x \mathcal Y,
\) the decomposition ensures that residual directions are geometrically independent of the base representation. Consequently, for each input, only the residual directions complementary to the current representation contribute, preventing interference and enforcing local specialization. Finally, when the learned representation is insufficient and fails to span all required directions, the projection mechanism alone is no longer adequate. In this regime, residual branches act constructively by introducing the missing components of \(T_x \mathcal Y\), progressively completing the representation. This projection–completion interplay explains how deep architectures balance overparameterization with incremental expressivity.

The following lemma formalizes the local behavior of a PoS module by showing how projections decompose into a base (subspace) component and a residual correction.

\begin{lemma}[Local Nonlinear Decomposition of PoS Module]
\label{lemma:ResidualProj}
Let $x \in \mathcal M_B \cap \mathcal M$ be a point of transverse intersection, and let $s$ be sufficiently close to $x$. Then the projection of $s$ onto $\mathcal M$ admits the local decomposition
\begin{equation}
\mathcal{P}_{\mathcal M}(s)
\approx
\mathcal{P}_{\mathcal M_B}(s)
+
\Phi\!\left(
\mathcal{P}_{\mathcal M_R}(s)
-
\mathcal{P}_{\mathcal M_R}\bigl(\mathcal{P}_{\mathcal M_B}(s)\bigr)
\right),
\label{eq:residualprojection}
\end{equation}
where $\Phi : T_x\mathcal M_R \to N_x(\mathcal M_B)$ is a local nonlinear mapping induced by the transversal structure.
\end{lemma}

This decomposition shows that each PoS module separates representation into a shared subspace component and a residual correction, providing a geometric interpretation of residual connections in deep networks.

\begin{proof}[Proof sketch]
The result follows from a local tangent-space linearization at a transverse intersection point $x \in \mathcal M_B \cap \mathcal M$ and inversion of the projection operator restricted to the residual subspace. The full proof is deferred to Appendix~\ref{app:ResidualProjectionProof}.
\end{proof}

Observe that the second term in Eq.~\eqref{eq:residualprojection} is governed by the intersection geometry between $\mathcal M_B$ and $\mathcal M$, specifically by the angle between their tangent spaces at $x$. In the limiting case where the residual directions $T_x\mathcal M_R$ are orthogonal to $T_x\mathcal M_B$, the cross-projection term
\(
\mathcal{P}_{\mathcal M_R}\bigl(\mathcal{P}_{\mathcal M_B}(s)\bigr)
\)
vanishes, yielding a fully decoupled decomposition.

\section{Experimental validation}
\label{Results}

We experimentally validate our theoretical framework using three challenging tasks. 

\subsection{Zero-shot anomaly detection (ECG).}
\label{sec:ZeroShotECG}

As a case study, we consider a realistic personalized setting where only a few minutes (e.g., 5 minutes) of ECG recordings are available for each user, and no anomaly samples are observed during training~\cite{yamacc2022personalized}. The goal is to detect anomalies for a specific user without ever seeing anomalous signals from that user during training. This raises the following question: can a compact and efficient model, trained solely on healthy signals, outperform conventional deep models that rely on large-scale datasets including anomalies?

\paragraph{PoS predictions in this setting:}
The PoS framework yields concrete, testable predictions for personalized ECG anomaly detection. If user-specific signals lie on compact healthy manifolds related by approximate isometries, then learning only healthy manifolds should enable:
(i) zero-shot anomaly detection without anomaly samples,
(ii) transfer across users via manifold alignment,
and (iii) improved separation of anomalies as deviations from compact representations.
The experiments below directly evaluate these predictions (see Appendix~\ref{ECGDataset} for dataset preparation details).

The following experiments evaluate four core capabilities predicted by the PoS framework: detection, transfer, privacy, and explainability.

\textbf{Compact Learning:}
This setting directly implements the compact learning principle derived from the PoS framework. In this setting, we learn only the healthy signal manifold of a given user, denoted by $\mathcal{M}_{D_i} \subset \mathbb{R}^n$, where normal signals reside. During inference, anomaly detection is performed based on the reconstruction error of a lightweight, user-specific network. This network consists of a linear encoder $f_E:\mathbb{R}^n \rightarrow \mathbb{R}^N$ and a linear decoder $f_D:\mathbb{R}^N \rightarrow \mathbb{R}^n$, mapping between input and latent spaces.

\begin{figure*}[]
\centering
  \includegraphics[width=0.95\linewidth]{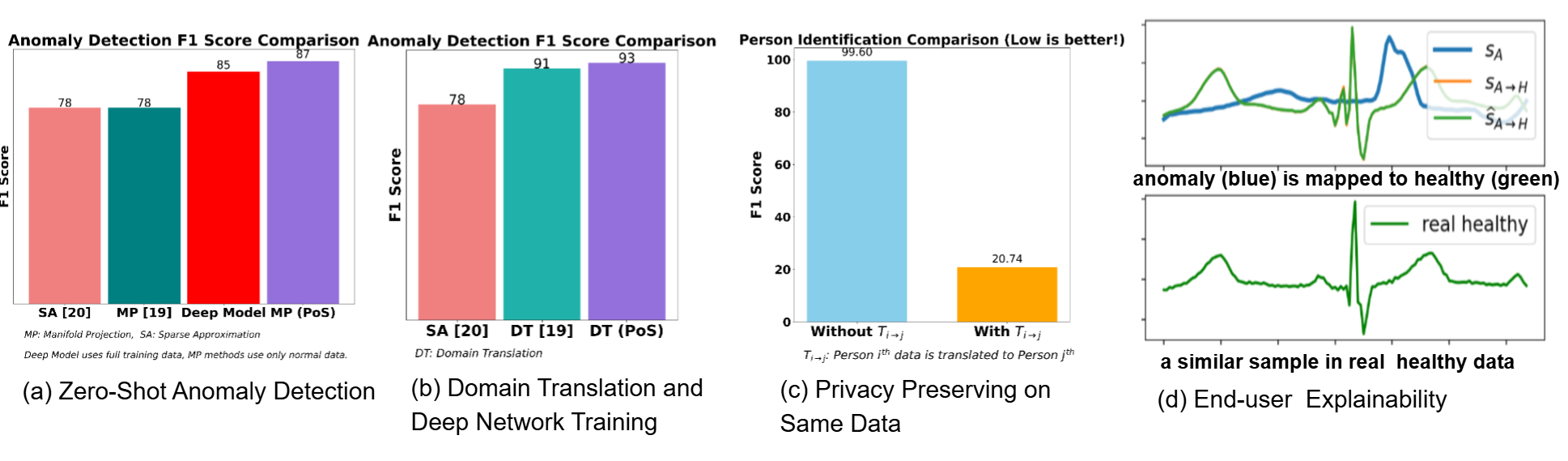}
  \caption{Zero-shot ECG anomaly detection via PoS. (a) Personalized projection learning on $\mathcal{M}_{D_i}$ achieves superior performance without anomaly samples. (b) Domain translation $T_{j \to i}$ generates synthetic data for user-specific training. (c) Identity-preserving mappings $T_{i \to k}$ enable privacy by aligning all users to a reference manifold. (d) On-the-fly alignment $T_{a \to i}$ maps anomalies to the healthy manifold, providing interpretable explanations.}
\label{fig:ECG-Anomaly}
\end{figure*}

When no nonlinear activation is used and $N \ll n$ (e.g., $N=10$, $n=128$), the network learns a projection onto a single linear subspace. To extend this representation from a single subspace (span) to a union of subspaces, ReLU is added in the latent space as nonlinear activation function. This transition directly reflects the theoretical prediction that nonlinear activations such as ReLU enable lifting representations from a single subspace to a union of subspaces. Additionally, as discussed in the orthogonality section, the input signal is masked with varying window sizes to encourage disentangled representations and promotes separation of submanifold components (see Appendix~\ref{app:ablationECG} for an ablation study). Such masking promotes separation of submanifold components by activating residual directions, as predicted by the geometric analysis in the Section~\ref{Orthogonality}.

We compare the proposed method with subspace projection-based anomaly detection (MP)~\cite{yamacc2022personalized}, sparse approximation (SA)-based method~\cite{carrera2016ecg}, and deep neural network~\cite{kiranyaz2016real} trained using data from multiple users, including anomaly samples. The proposed PoS-based manifold projection method with compact representation (MP (PoS)) achieves substantially higher performance despite relying only on healthy training data, as shown in Figure~\ref{fig:ECG-Anomaly}(a). Notably, despite its minimal computational cost and reliance only on healthy training data, the proposed approach outperforms deep models trained with access to both normal and anomalous data from multiple users.

\textbf{Signal Generation / Domain Translation: }
This experiment evaluates the PoS prediction that learned isometries enable transfer across manifolds via geometric alignment. A key limitation of conventional deep learning models is their strong dependence on training data distributions. When the data modality, acquisition device, or subject changes, a domain gap arises, often leading to significant performance degradation. In our setting, this issue is particularly pronounced: As shown in prior work, ECG patterns are highly individual-specific due to differences in cardiovascular structure and physiology. Consequently, models trained on one set of users do not generalize well to a new user, especially when no anomaly samples are available for that individual. This makes direct deployment of standard deep learning models in personalized settings highly unreliable.

To address this domain gap, we adopt the geometric folding-based domain translation framework illustrated in Figure~~\ref{fig:IsometryUnion}. Let $\mathcal{M}_{D_i}$ and $\mathcal{M}_{D_j}$ denote the signal manifolds of users $i$ and $j$, learned via encoder--decoder mappings $(f_E^i, f_D^i)$ and $(f_E^j, f_D^j)$. We define a learnable transformation $T_{j \to i}$ that maps samples from the domain of user $j$ toward the manifold $\mathcal{M}_{D_i}$. 

Given a sample $s^{j} \in \mathbb{R}^n$, the transformation is trained using the folding loss
\begin{equation}
J_{\mathrm{fold}}(\theta_{j \to i})
=
\left\|
T_{j \to i}(s^{j})
-
\mathcal{P}_{\mathcal{M}_{D_i}}\bigl(T_{j \to i}(s^{j})\bigr)
\right\|^2,
\label{eq:foldingloss}
\end{equation}
which encourages the transformed samples to lie on the learned manifold $\mathcal{M}_{D_i}$, where $\mathcal{P}_{\mathcal{M}_{D_i}} = f_D^i \circ f_E^i$ denotes the projection operator. Consequently, anomalous samples from user $j$ can be translated into the domain of user $i$, producing synthetic anomalies consistent with $\mathcal{M}_{D_i}$ consistent with the geometry of user $i$. By aggregating such translated samples across users and anomaly types, we construct a personalized training set for user $i$, enabling effective supervised learning of baseline DNN model~\cite{kiranyaz2016real}. As illustrated in Figure~~\ref{fig:ECG-Anomaly}(b), this projection-based translation mechanism supports SoTa anomaly detection.

\textbf{Privacy-Preserving via PoS Isometric Mappings:}
The learned transformation framework can also be used to conceal user identity through geometric alignment. Specifically, given a set of users, we select a reference user $k$ with manifold $\mathcal{M}_{D_k}$ and learn transformations $T_{i \to k}$ that map signals from user $i$ onto $\mathcal{M}_{D_k}$. For a signal $s^{i} \in \mathbb{R}^n$, the transformed signal $T_{i \to k}(s^{i})$ preserves the intrinsic structure of the signal while removing user-specific morphological characteristics.

This transformation effectively anonymizes the data: signals from different users become indistinguishable in the transformed domain. To evaluate this property, we train a malicious neural network designed to identify user identity from ECG signals. As shown in Figure~~~\ref{fig:ECG-Anomaly}(c), while the network achieves near-perfect accuracy on the original data (F1 score: $99.60$), its performance drops significantly after applying the PoS-based isometric mappings (F1 score: $20.74$), indicating that identity-specific information has been effectively suppressed. This demonstrates that PoS-based isometric mappings effectively remove identity-specific information while preserving task-relevant structure. This approach provides a principled mechanism for privacy-preserving data sharing without sacrificing downstream task performance.

\textbf{Explainability via Geometric Alignment:}
Complementary to the previous setting, we learn a transformation $T_{a \to i}$ that aligns anomalous signals with the healthy manifold $\mathcal{M}_{D_i}$ (more generally, a union of submanifolds) of user $i$. Given a signal $\mathbf{s} \in \mathbb{R}^n$ detected as anomalous, the transformation is learned on-the-fly for the given sample by minimizing
\[
\left\|
T_{a \to i}(\mathbf{s}) - \mathcal{P}_{\mathcal{M}_{D_i}}\bigl(T_{a \to i}(\mathbf{s})\bigr)
\right\|^2,
\]
which encourages the transformed signal $T_{a \to i}(\mathbf{s})$ to lie close to the union of subspaces representing the healthy manifold. In this sense, $T_{a \to i}$ performs a geometric alignment (e.g., rotation or deformation) that brings the anomalous signal into the vicinity of the healthy representation, while the projection operator enforces local consistency. For this mechanism to be effective, the learned manifold $\mathcal{M}_{D_i}$ must satisfy Postulate~I, i.e., it should represent a compact and minimal embedding of the healthy signal space. This provides a direct explanation of the model's decision: the original signal $\mathbf{s}$ (blue in Figure~~\ref{fig:ECG-Anomaly}(d)) is identified as anomalous, while its aligned version $T_{a \to i}(\mathbf{s})$ (green) represents its closest counterpart within the healthy manifold. Since the transformation already places the signal near the manifold, the projection induces only a small correction, indicating that the translated signal faithfully reflects how the input would appear under normal conditions. As illustrated in Figure~~\ref{fig:ECG-Anomaly}(d), the translated signal closely resembles real healthy patterns, enabling interpretable, user-level feedback grounded in the learned manifold structure.

\subsection{Geometric Regularization for Volumetric Reconstruction}
\label{appendix:microscopy}

We empirically validate the emergence of geometric disentanglement within neural networks through a practical application in computational imaging, namely the reconstruction of 3D volumetric sample data from Fourier Light Field microscopy (FLFM) \cite{FLFM} images. FLFM captures rapid 3D biological dynamics in a single snapshot, but the inherent spatio-angular resolution limits of its optical encoding make high-resolution volumetric reconstruction a severely ill-posed inverse problem. Consequently, computationally reconstructed 3D volumes frequently suffer from prominent blur and spatial degradation. We postulate that the true, high-fidelity biological structures reside on a smooth, low-dimensional data manifold $\mathcal{M}_D \subset \mathbb{R}^n$, while system-induced blur acts as a degradation operator that displaces samples off this manifold. 

\paragraph{PoS predictions in this setting:}
The PoS framework yields concrete predictions for volumetric reconstruction. 
(i) \emph{Compact learning}: high-quality microscopy images should form a compact union of submanifolds, which can be enforced by pushing degraded observations away from this set, rather than reconstructing clean images from degraded inputs (in contrast to standard masked autoencoding objectives). 
(ii) \emph{Projection refinement}: once a compact representation is learned, geometric alignment via the folding mechanism enables improved estimation of the transformation operator $T_{\text{in}}$, which naturally serves as the final deblurring mapping. 
The formulation below directly instantiates these predictions.

By explicitly forcing an autoencoder to separate the representations of clean data and their degraded counterparts, we can guide the primary reconstruction network to project its outputs onto $\mathcal{M}_D$. To accomplish this, we incorporate a masked autoencoder pretrained on a large image dataset \cite{MAEMicroscopy} and fine-tune it on our target dataset \cite{FVCD} to learn the geometric structure of $\mathcal{M}_D$. To ensure the network learns to isolate the residual tangent directions associated with true structure from those associated with optical degradation, we introduce a dynamic push-pull loss, described below.

Let $s \in \mathcal{M}_{D_i}$ be a clean microscopy slice and $s_b = B_\sigma(s)$ be its blurred version using a synthetic Gaussian blur, representing a point displaced from the manifold. The autoencoder comprises an encoder--decoder mapping $(f_E, f_D)$. We define the projection operator onto the learned manifold as $\mathcal{P}_{\mathcal{M}_{D_i}} = f_D \circ f_E$, with the corresponding projections $\hat{s} = \mathcal{P}_{\mathcal{M}_{D_i}}(s)$ and $\hat{s}_b = \mathcal{P}_{\mathcal{M}_{D_i}}(s_b)$. The fine-tuning objective is given by:
\begin{equation}
J_{AE} = \lambda_1\|\hat{s} - s\|_2^2 + \underbrace{\lambda_2\|\hat{s}_b - s\|_2^2 - \lambda_3\|x - x_b\|_2^2}_{\text{push-pull loss}},
\label{eq:AEMicLoss}
\end{equation}
where $x=f_E(s)$ and $x_b=f_E(s_b)$. Geometrically, the positive reconstruction terms tie the autoencoder to the true data manifold, pulling the reconstructions of both clean and degraded inputs toward $\mathcal{M}_{D_i}$. Concurrently, the negative latent term explicitly pushes the latent representations of the clean image and its blurred counterpart apart. This push-pull mechanism forces the nonlinear projector $\mathcal{P}_{\mathcal{M}_{D_i}}$ to become highly sensitive to degradation, isolating the latent geometric factors associated purely with high-quality structures.

\begin{figure}
    \centering
    \includegraphics[width=0.85\linewidth]{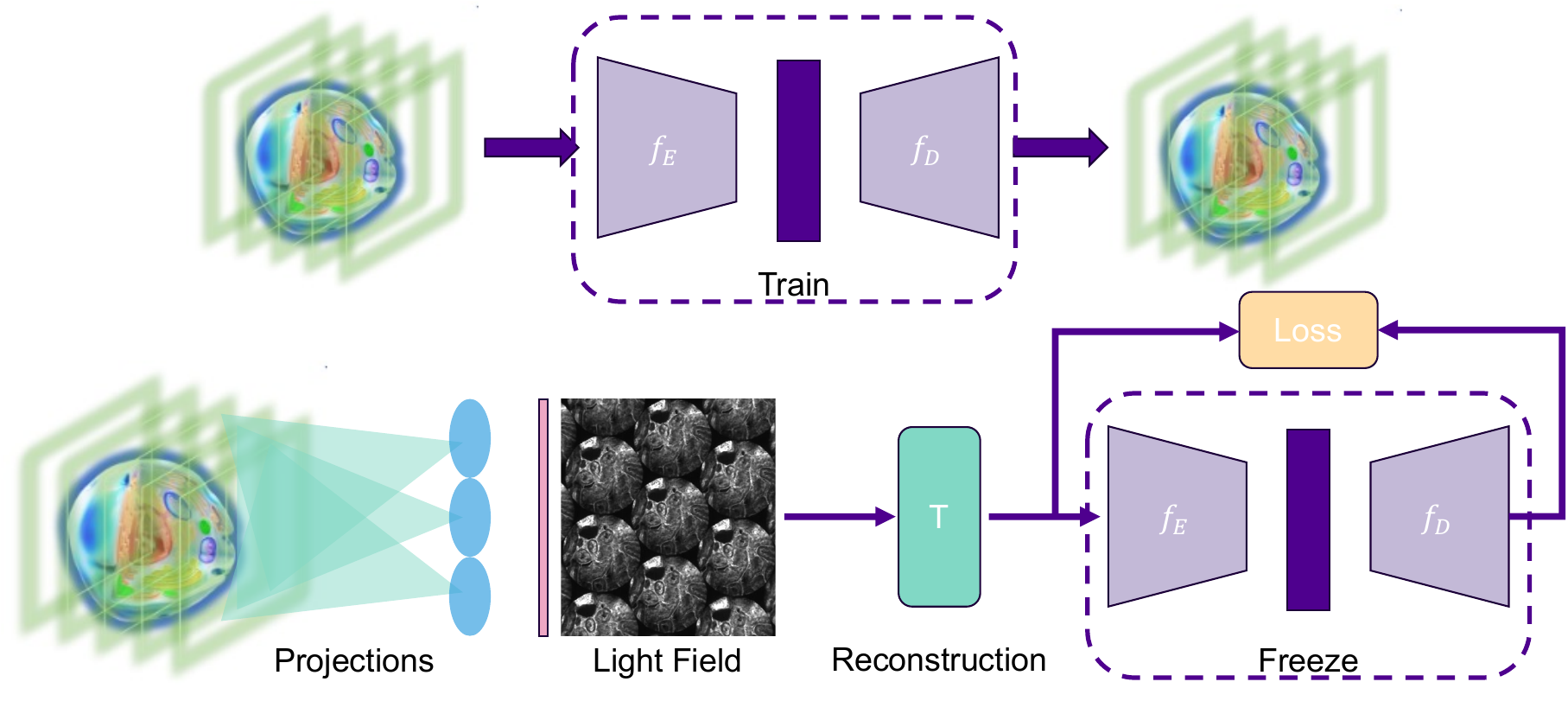}
    \caption{Manifold projection as a prior for 3D microscopy reconstruction, composed of two stages. Top: An autoencoder ($f_E, f_D$) is first trained on high-quality 3D volumes to learn the compact geometric manifold of clean biological structures. Bottom: A reconstruction network ($T$) is trained to estimate 3D volumes from the recorded light field. The previously learned autoencoder is frozen at this stage and used as a geometric regularizer. The projection loss explicitly forces the reconstructed volumes to lie on the learned manifold of clean 3D structures, mitigating optical degradation.}
    \label{fig:microscopy_setup}
\end{figure}

This sensitivity to blur enables us to leverage the fine-tuned projector as an auxiliary geometric projection loss for the primary FLFM volumetric reconstruction network. Defining $s^j \in \mathcal{M}_{D_j}$ as the input light field obtained from the sensor and $T_{j \to i}$ as the reconstruction network mapping the light field to 3D volumetric data, the folding loss in Eq.~\ref{eq:foldingloss} serves as an idempotent projection constraint alongside the main supervised loss. We demonstrate the efficacy of this geometric constraint through both quantitative metrics and qualitative analysis. Specifically, we adopt the state-of-the-art architecture from \cite{FVCD} for volumetric image reconstruction. To conduct a controlled ablation study, we disregard sensor noise and train solely the reconstruction module. The pretrained autoencoder from \cite{MAEMicroscopy} is fine-tuned on the dataset from \cite{FVCD} under two distinct schemes. First, using a standard $L_2$ reconstruction loss $J = \|\hat{s} - s\|_2^2$, and second, using the proposed push-pull loss (Eq.~\ref{eq:AEMicLoss}). The resulting autoencoders are subsequently integrated into the main network's training pipeline. We compare these models against a baseline trained without the folding loss, as well as a network regularized using the pretrained autoencoder without any fine-tuning. 
All methods are trained for 200 epochs using the AdamW optimizer with a learning rate of 1e-4. A shared combination of MSE and edge losses (with an edge loss coefficient of 0.1) is used as the base, with all configurations except the baseline additionally incorporating the folding loss.
We use a volume patch size of 480$\times$480 and 41 depth slices.
The training is performed on the LUMI supercomputing cluster utilizing AMD MI250X GPUs. A full training run for each model took approximately 80 GPU-hours.
The optimized models are evaluated on the test dataset of 125 high-quality 3D volumes \cite{FVCD} that are not introduced during training. 
The average peak signal-to-noise ratio (PSNR) and structural similarity (SSIM) values are reported in Table~\ref{tab:microscopy_results}, and visual results for a sample test data are provided in Figure~~\ref{fig:microscopy_results} for qualitative inspection.

\begin{figure}
    \centering
    \includegraphics[width=0.85\linewidth]{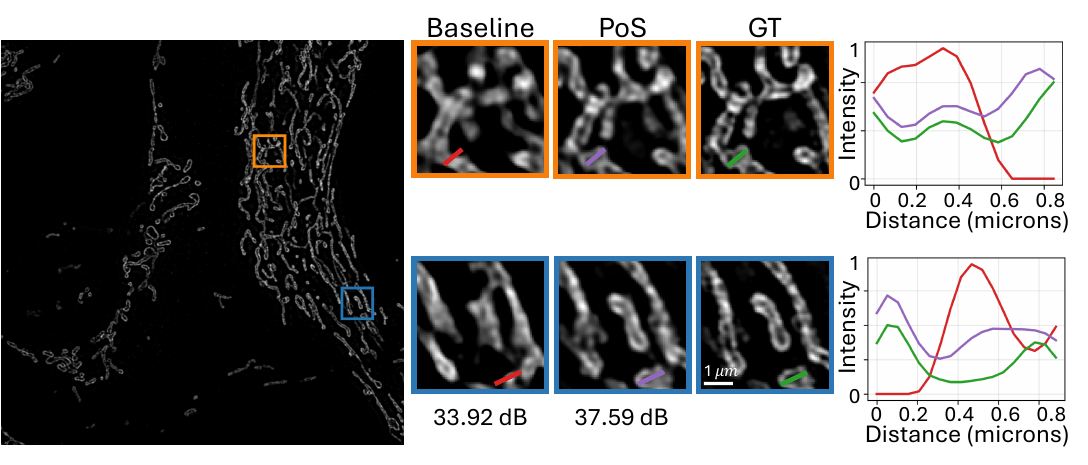}
    \caption{Qualitative inspection of the domain adaptation technique in volumetric reconstruction. Left: A two-dimensional slice of the ground-truth volume with two regions of interest (ROIs) marked. Middle: The close-up reconstruction results for the baseline method (architecture adopted from \cite{FVCD}) trained from scratch with L2-loss, in comparison with the proposed training based on the additional folding loss. Right: One-dimensional cross-sections selected within each ROI.}
    \label{fig:microscopy_results}
\end{figure}

\begin{table}[h]
\centering
\caption{Quantitative evaluation of 3D volumetric microscopy image reconstruction.}
\label{tab:microscopy_results}
\small
\begin{tabular}{lccc}
\toprule
\textbf{Model} & \textbf{PSNR (dB)} & \textbf{SSIM} \\
\midrule
Baseline \cite{FVCD} & 34.37 & 0.968 \\
MAE Pretrained \cite{MAEMicroscopy} & 34.01 & 0.965 \\
MAE Finetuned & 34.45 & 0.968 \\
MAE Push-Pull & $\mathbf{39.03}$ & $\mathbf{0.974}$ \\
\bottomrule
\end{tabular}
\end{table}

As reported in Table~\ref{tab:microscopy_results}, simply applying a pretrained or standard finetuned autoencoder yields negligible improvements over the baseline model, as the network fails to efficiently separate the underlying signal from the degradation. However, explicitly enforcing latent separation via the dynamic push-pull loss drastically improves the manifold projection. Specifically, the PoS-oriented training scheme achieves, on average, a 4.66 dB increase in PSNR compared to the baseline model. Visually, the proposed model reconstructs sharper Regions of Interest (ROIs) that closely track the ground truth structures, successfully mitigating the low-resolution degradation inherent to baseline LF microscopy systems.

\subsection{Attention, Transformer and PosFormer}
\label{appendix:attention}
\paragraph{PoS implications for this section: }
The geometric behavior illustrated in Figure~~\ref{fig:maskedLearning1} provides further intuition for the emergence of transformer-like architectures. As shown, projection onto the joint span first removes components in the shared nullspace, while residual nullspace directions remain unresolved. Subsequent refinement steps progressively eliminate these remaining components, effectively performing a sequence of structured null-space removals.

This perspective suggests that representation learning proceeds in stages: shared components are removed first, followed by increasingly refined residual components. When this process is implemented through parallel representations with cross-consistency constraints, the resulting step-by-step elimination naturally gives rise to attention-like mechanisms.

Moreover, this explains the effectiveness of PoSFormer. By explicitly enforcing orthogonality on predefined residual components, the model prevents leakage across submanifolds and ensures that residual directions are removed in a structured manner. This leads to improved separation of submanifolds and more stable representations.

\subsubsection{Intersection and Residual Learning}
\label{app:intersection_residual}

A central requirement of the PoS framework is to avoid collapsing representations onto the joint span of multiple submanifolds, and instead preserve their union structure. This is achieved by decomposing representations into shared (intersection) and submanifold-specific (residual) components, as illustrated in Figure~~\ref{fig:intersection_residual}.

Let $\mathcal{M}_{D_i}, \mathcal{M}_{D_j} \subset \mathbb{R}^n$ denote two submanifolds with corresponding projection operators $\mathcal{P}_{D_i}$ and $\mathcal{P}_{D_j}$. Given an input signal $s \in \mathbb{R}^n$, we first compute individual projections
\begin{equation}
s_i = \mathcal{P}_{D_i}(s), \qquad s_j = \mathcal{P}_{D_j}(s).
\end{equation}

\paragraph{Intersection via bidirectional orthogonal projection: }
In our implementation, the shared component between $s_i$ and $s_j$ is estimated through direct cross-projections rather than an independent iterative variable. Specifically, we compute
\begin{equation}
z_{i \leftarrow j} = \frac{s_i (s_i^\top s_j)}{s_i^\top s_i},
\qquad
z_{j \leftarrow i} = \frac{s_j (s_j^\top s_i)}{s_j^\top s_j}.
\label{eq:cross_projection}
\end{equation}
These correspond to projecting each representation onto the direction of the other.

\paragraph{Coupled iterative refinement: }
Rather than aggregating the cross-projections, the two aligned components are fed back into their respective projection operators and iteratively refined. Specifically, we initialize
\begin{equation}
z_i^{(0)} = s_i, 
\qquad 
z_j^{(0)} = s_j,
\end{equation}
and update them via coupled projections
\begin{equation}
z_i^{(t+1)} = \mathcal{P}_{D_i}\!\left(
\frac{z_j^{(t)} \, (z_j^{(t)\top} z_i^{(t)})}{z_j^{(t)\top} z_j^{(t)} + \varepsilon}
\right),
\qquad
z_j^{(t+1)} = \mathcal{P}_{D_j}\!\left(
\frac{z_i^{(t)} \, (z_i^{(t)\top} z_j^{(t)})}{z_i^{(t)\top} z_i^{(t)} + \varepsilon}
\right),
\label{eq:coupled_refinement}
\end{equation}
where $\mathcal{P}_{D_i}$ and $\mathcal{P}_{D_j}$ enforce that each update remains on the corresponding submanifold.

This process is repeated until convergence, yielding representations
\(
z_i^*, z_j^*
\)
that lie on their respective manifolds while becoming increasingly consistent under cross-projection. For a common input $s$, the coupled updates drive the representations toward agreement, so that $z_i^* \approx z_j^*$ and converge to a shared point $z^*$ that lies near (or provides a local approximation to) the intersection of the two manifolds. In this regime, the common representation can be taken as either $z_i^*$ or $z_j^*$, or equivalently expressed via a symmetric combination such as $z^* = \frac{1}{2}(z_i^* + z_j^*)$, which approximates the intersection of the two manifolds.

To enforce this behavior, we optimize a joint objective combining reconstruction fidelity and residual suppression. Let $\hat{s}$ denote the reconstructed signal, decomposed as
\(
\hat{s} = z^* + r_i + r_j,
\)
where $z^*$ approximates the shared component and $r_i, r_j$ denote residual components associated with the respective manifolds. We minimize \[
\mathcal{L} =
\|s - \hat{s}\|^2
+
\lambda \cdot
\begin{cases}
\|r_j\|^2, & s \in \mathcal{M}_{D_i}, \\
\|r_i\|^2, & s \in \mathcal{M}_{D_j}.
\end{cases}
\]
where the residual penalty enforces that components orthogonal to the correct manifold are driven to zero. This objective encourages the shared representation $z^*$ to capture the common structure while eliminating cross-manifold residuals, thereby aligning the iterative updates with the intersection geometry.
\begin{figure}
\includegraphics[width=0.48\linewidth]{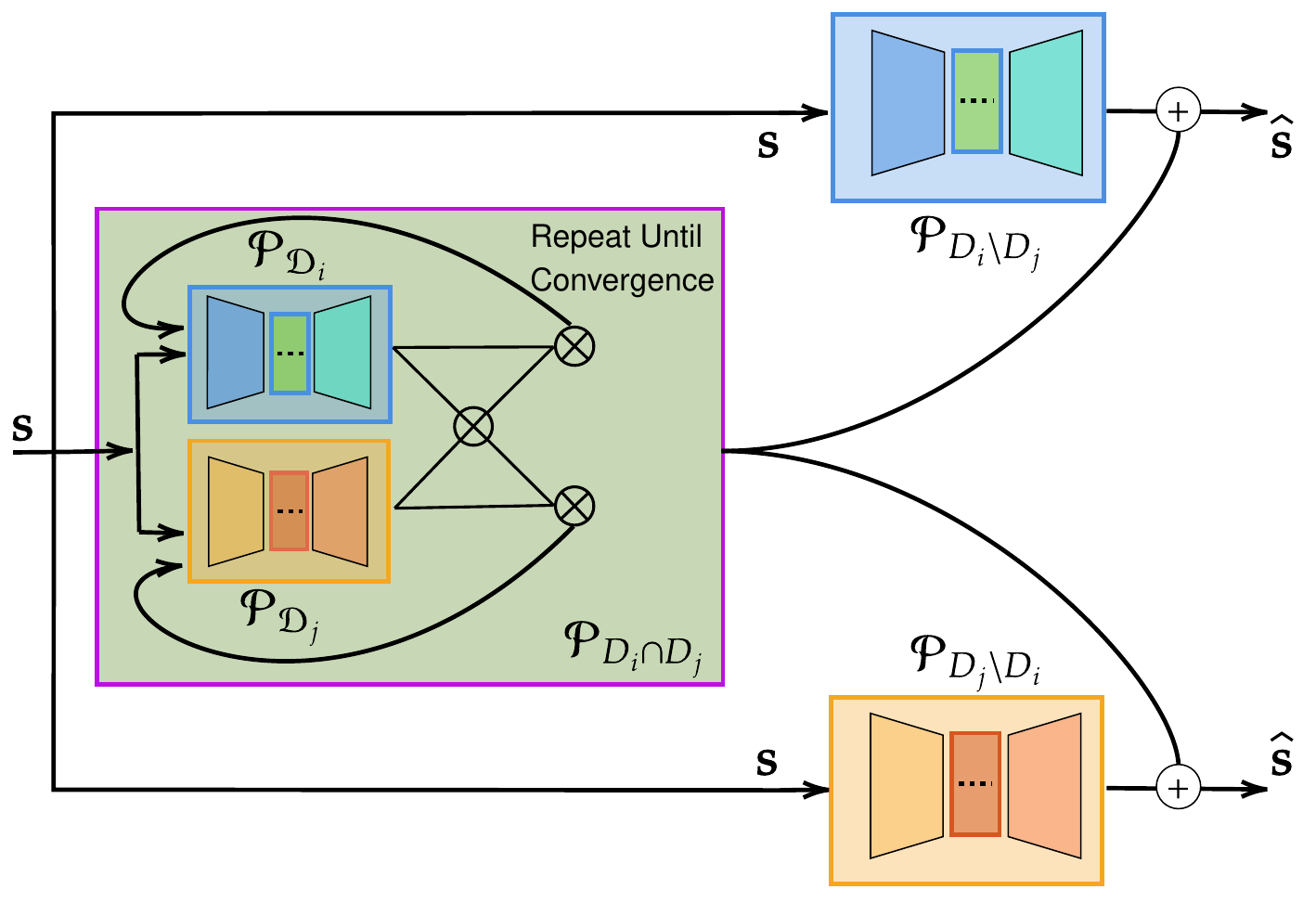}
\includegraphics[width=0.50\linewidth]{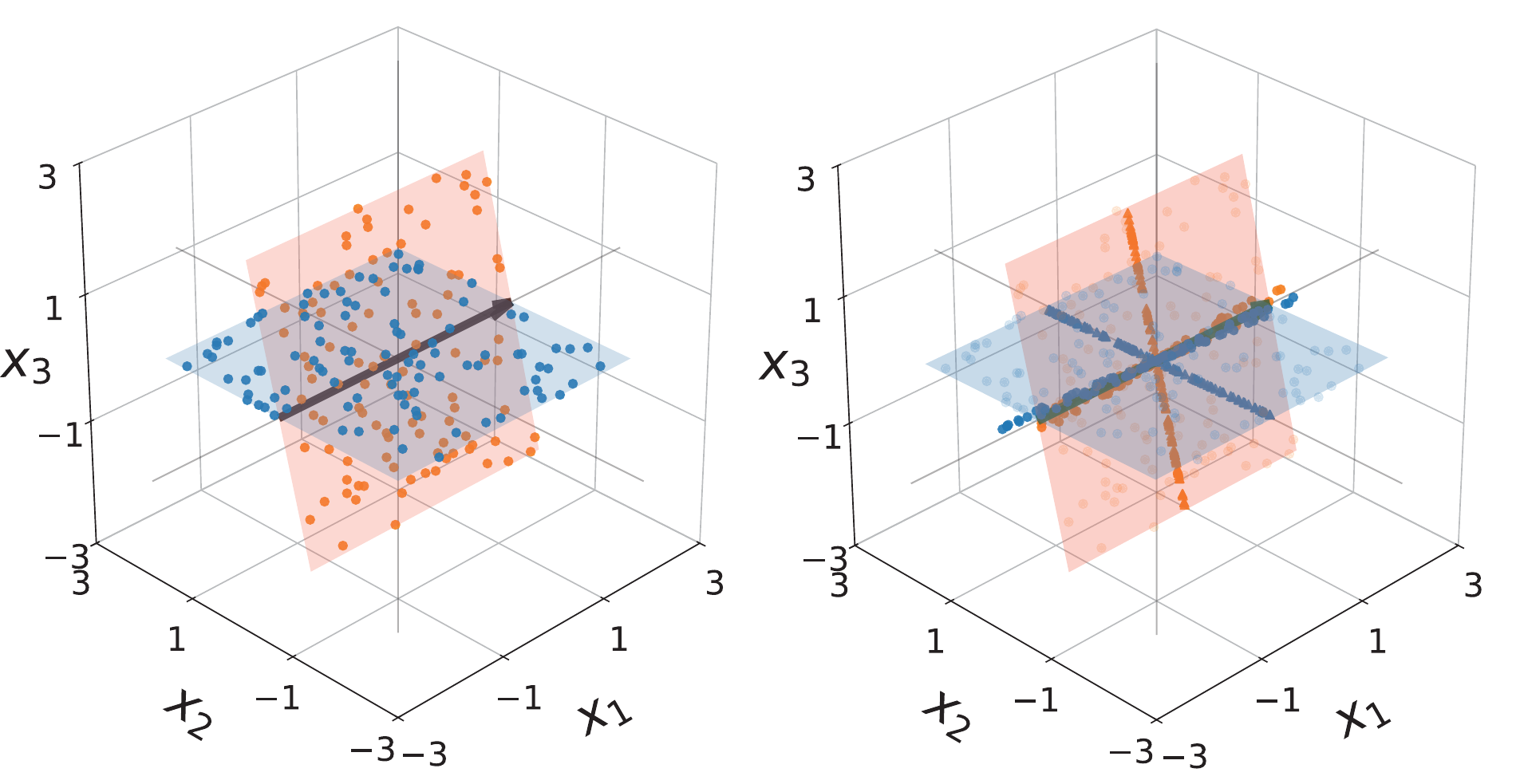}
\caption{ Intersection–residual learning via coupled cross-projections. 
Left: The architecture details. The input is projected onto $\mathcal{M}_{D_i}$ and $\mathcal{M}_{D_j}$, then iteratively aligned via bidirectional projections. 
The resulting representations approximate the shared component, while residuals capture submanifold-specific directions.
Right: The results with a toy dataset drawn from a union of two planar subspaces, their intersection being marked with a purple arrow. Projection of points towards the estimated intersection subspace, and corresponding residual projections are evident.}
\label{fig:intersection_residual}
\end{figure}

This intersection-residual mechanism illustrated in Figure~\ref{fig:intersection_residual} forms the basis for collaborative learning across multiple submanifolds and directly leads to attention-like operators in deep architectures.

\subsubsection{Collaborative Learning via Cross-Projection Consistency}
\label{app:collaborative_learning}

Building on the intersection--residual mechanism, we extend the framework to a collaborative learning setting in which multiple PoS modules are trained jointly. Rather than learning representations independently, each module refines its output by enforcing consistency with the representations produced by other modules.

Consider two projection operators $\mathcal{P}_{D_i}$ and $\mathcal{P}_{D_j}$ acting on an input $s \in \mathbb{R}^n$, producing
\[
s_i = \mathcal{P}_{D_i}(s), \qquad s_j = \mathcal{P}_{D_j}(s).
\]
In a standard setting, these representations would be learned independently. In contrast, we introduce a cross-projection consistency mechanism, in which each representation is evaluated relative to the other.

Specifically, we compute cross-aligned components
\[
\tilde{s}_i = \frac{s_i (s_i^\top s_j)}{s_i^\top s_i}, 
\qquad
\tilde{s}_j = \frac{s_j (s_j^\top s_i)}{s_j^\top s_j},
\]
which measure the extent to which each representation is supported by the other manifold.

These aligned components are not fed back into the same module, but instead passed forward to subsequent layers. Each layer applies a single cross-projection consistency step, refining the representations in a feedforward manner:
\begin{equation}
s_i^{(\ell+1)} = \mathcal{P}_{D_i}^{(\ell)}(\tilde{s}_j^{(\ell)}), 
\qquad
s_j^{(\ell+1)} = \mathcal{P}_{D_j}^{(\ell)}(\tilde{s}_i^{(\ell)}),
\end{equation}
where $\ell$ denotes the layer index and the projection operators may differ across layers.

This design yields a hierarchical, stage-wise refinement process: rather than performing a full projection in a single step, the network progressively improves representations across layers. At each stage, parallel branches perform cross-checks via orthogonal projections, ensuring that only mutually consistent components are preserved.

Each layer can therefore be viewed as: (i) computing candidate representations in parallel, (ii) performing a cross-projection consistency check between them, (iii) passing refined representations forward to the next layer.

This collaborative mechanism provides a standalone framework for enforcing cross-representation consistency, which can be applied in settings such as denoising and multi-view learning. When combined with residual decomposition as shown in the following subsection, it gives rise to attention-like aggregation in deep architectures.

\subsubsection{Transformers from the PoS Perspective}
\label{app:transformer_pos}

We now show how transformer architectures naturally arise from the geometric principles of the PoS hypothesis. The key idea is that projection onto a union of submanifolds can be decomposed into sequential residual-removal steps, and further enhanced through collaborative cross-consistency across parallel representations.

\paragraph{Step 1: Projection as residual removal.}
Consider a representation $s \in \mathbb{R}^n$ and a dictionary (or subspace) $D$. Classical projection can be written in two equivalent forms:
\begin{equation}
\hat{s} = D x, 
\qquad \text{or equivalently} \qquad 
\hat{s} = s - D^{\perp} x_{\perp},
\end{equation}
where $D$ spans the explainable directions and $D^{\perp}$ spans the complementary (residual) subspace.

The corresponding residual is
\begin{equation}
s_{\mathrm{res}} = s - \hat{s}.
\end{equation}

This shows that projection can be interpreted either as constructing the signal from its subspace components ($Dx$), or as removing residual directions ($s - D^{\perp}x_{\perp}$). In both cases, all explainable components are handled in a single step.

\paragraph{Step 2: Sequential decomposition of projection.}
Instead of removing all residual directions in a single step, the same projection can be decomposed into a sequence of partial residual removals:
\begin{equation}
\hat{s}
=
s - D^{\perp} x
=
s - D_1^{\perp} x_1 - D_2^{\perp} x_2 - \cdots,
\end{equation}
where each $D_\ell^{\perp}$ captures a subset of residual directions that are progressively removed across layers.

Each step removes only a portion of the residual, progressively refining the representation toward the projected signal. This shows that projection can be factorized into a sequence of simpler residual-removal operations.

This provides a geometric interpretation of depth: rather than computing a projection in a single step, deep networks progressively eliminate residual components across layers, yielding the same final representation through a sequence of structured refinements.

\begin{figure}[t]
\centering
\includegraphics[width=0.85\linewidth]{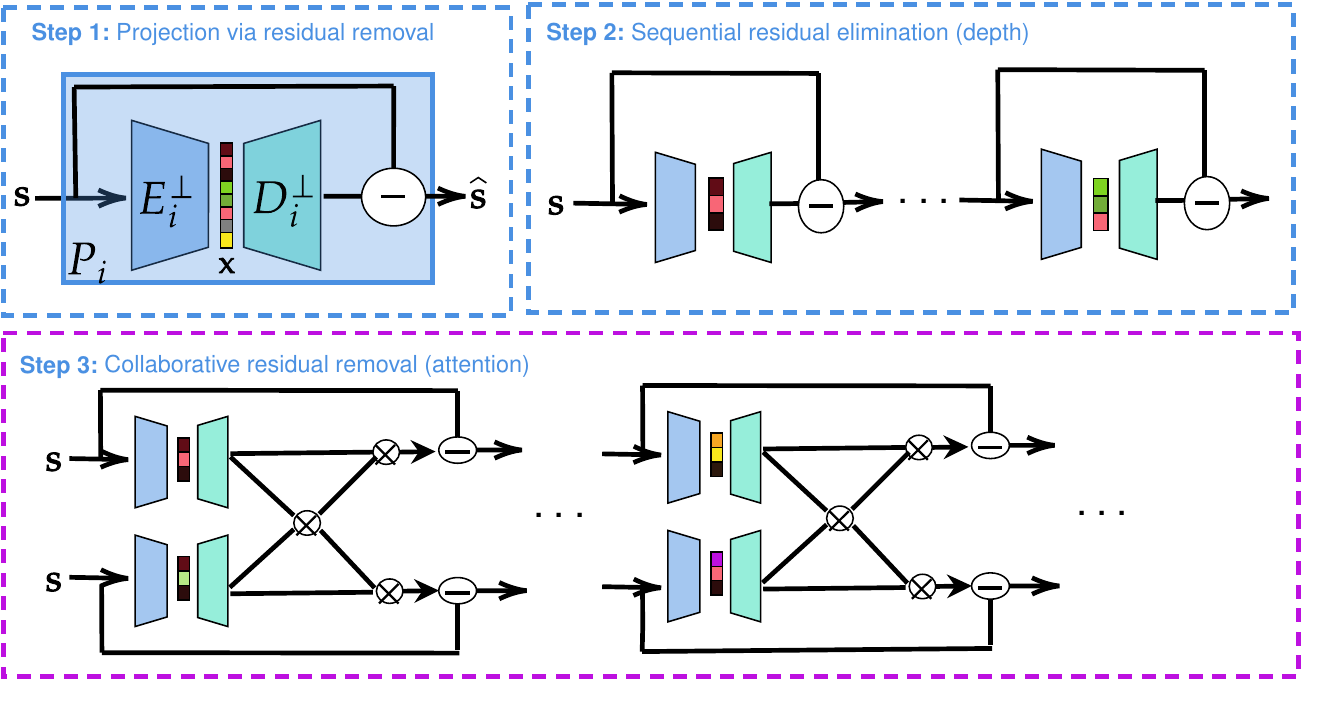}
\vspace{-0.2cm}
\caption{
From residual learning to transformers under the PoS framework.
\textbf{(Step 1)} Projection is interpreted as residual (null-space) removal, where components explained by a subspace are subtracted from the signal.
\textbf{(Step 2)} The same projection can be decomposed into a sequence of residual removals across layers, providing a geometric interpretation of depth.
\textbf{(Step 3)} Parallel modules perform collaborative residual removal via cross-projection, where components are removed from the signal in proportion to agreement.
Extending this mechanism to multiple parallel representations naturally yields attention maps in transformer architectures.
}
\label{fig:transformer_pos}
\end{figure}

\paragraph{Step 3: Collaborative residual removal via cross-projection.}

We now extend the residual-removal framework to multiple parallel representations. Rather than directly projecting onto candidate submanifolds, each module first estimates the residual (orthogonal) component of the input with respect to its associated subspace.
Let the representations in each branch be initialized as
\(
s_i^{(0)} = s_j^{(0)} = s.
\) At layer $\ell$, each branch maintains its own representation, and the residual components are computed as
\[
r_i^{(\ell)} = \mathcal{P}_{D_i^\perp}\!\big(s_i^{(\ell)}\big), 
\qquad 
r_j^{(\ell)} = \mathcal{P}_{D_j^\perp}\!\big(s_j^{(\ell)}\big),
\]
denoting the residual components at layer $\ell$. We compute cross-aligned residuals
\begin{equation}
\tilde{r}_i^{(\ell)} = \frac{r_i^{(\ell)} \big(r_i^{(\ell)\top} r_j^{(\ell)}\big)}{r_i^{(\ell)\top} r_i^{(\ell)}},
\qquad
\tilde{r}_j^{(\ell)} = \frac{r_j^{(\ell)} \big(r_j^{(\ell)\top} r_i^{(\ell)}\big)}{r_j^{(\ell)\top} r_j^{(\ell)}}.
\end{equation}

Each branch then removes the agreed residual component from its own input, yielding
\begin{equation}
s_i^{(\ell+1)} = s_i^{(\ell)} - \tilde{r}_i^{(\ell)}, 
\qquad
s_j^{(\ell+1)} = s_j^{(\ell)} - \tilde{r}_j^{(\ell)}.
\end{equation}

Thus, rather than performing independent projections, the network collaboratively identifies and removes residual components that are consistently supported across representations. This prevents premature removal of submanifold-specific structure while ensuring that shared residual directions are progressively eliminated.

By iterating this process across layers, the network performs a structured, collaborative residual removal, preserving the union-of-submanifolds structure while refining the signal toward its intrinsic representation.

\paragraph{From two branches to many: emergence of attention.}

The above construction naturally extends from two parallel representations to a collection of multiple representations. Let $\{s_t^{(\ell)}\}_{t=1}^T$ denote a set of parallel representations (e.g., tokens in a sequence), each evolving through residual removal. For clarity, we assume that all branches correspond to projections onto the same family of subspaces, so that query, key, and value roles are not yet distinguished.

At layer $\ell$, each branch computes its residual component
\[
r_t^{(\ell)} = \mathcal{P}_{D^\perp}\!\big(s_t^{(\ell)}\big),
\]
where $D$ denotes the shared subspace family.

To extend cross-projection consistency, we measure the agreement between the residual of a given branch $q$ and all other residuals via inner products
\[
\alpha_{qt}^{(\ell)} = r_q^{(\ell)\top} r_t^{(\ell)}.
\]

These alignment scores quantify how strongly residual directions are shared across representations. The shared residual component for branch $q$ is then estimated as the aggregated projection onto the span of residuals from other branches:
\begin{equation}
\tilde{r}_q^{(\ell)} = \sum_{t=1}^T 
\frac{r_t^{(\ell)} \big(r_t^{(\ell)\top} r_q^{(\ell)}\big)}{r_t^{(\ell)\top} r_t^{(\ell)}}.
\end{equation}

The updated representation is obtained by removing this shared residual component:
\begin{equation}
s_q^{(\ell+1)} = s_q^{(\ell)} - \tilde{r}_q^{(\ell)}.
\end{equation}

This mechanism generalizes pairwise cross-projection to multiple representations: rather than removing residuals independently, each branch removes the component that is consistently supported across all other branches. The resulting weights $\alpha_{qt}^{(\ell)}$ define an alignment-based weighting over representations.

From this perspective, attention is not a heuristic mechanism for routing information, but a geometric procedure that estimates which components should be removed from a representation based on agreement across multiple submanifolds. Each attention weight measures alignment between representations, and the resulting aggregation identifies shared components that are subtracted from the signal.

This establishes a direct connection between PoS principles and modern attention-based architectures.

\subsubsection{PoS-Former}

\begin{figure}[]
\centering

\includegraphics[width=0.90\linewidth]{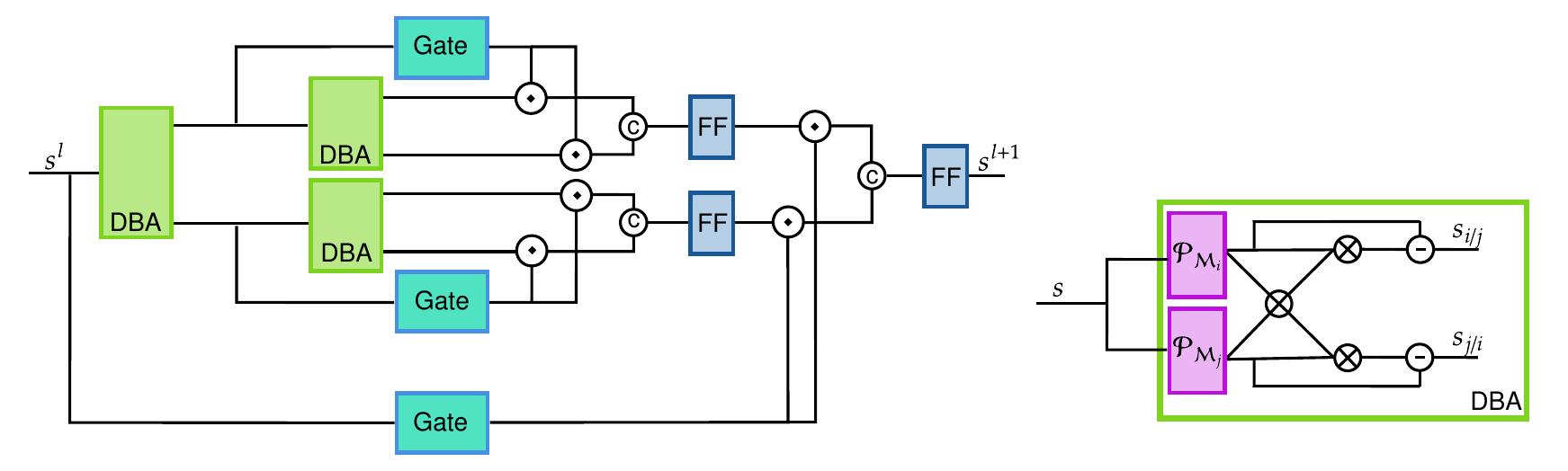}
\caption{Dual-branch attention (DBA) for subspace selection. Left: One layer of hierarchical nested decomposition based on DBA. Right: inner mechanism of a single DBA based on residual decomposition.}
\label{fig:posformer}
\end{figure}

Standard self-attention mechanisms often act as singular projection operators that struggle to represent mutually exclusive topological structures without catastrophic interference. Projecting a query signal onto a joint span often fails to resolve class-specific residual directions, leading to subspace overlap. To explicitly motivate representations consistent with the compact learning axiom, we introduce the \textit{Dual-Branch Orthogonal Flow}, a linear-time architectural mechanism designed to natively untangle intersecting submanifolds without relying on explicit data degradation.

\paragraph{Relation to transformer architectures:}
While transformer models aggregate information through attention-weighted combinations, the proposed PoSFormer follows a fundamentally different principle. Instead of mixing representations, it performs a \emph{hierarchical nested decomposition} of the input signal by progressively removing shared (normal-space) components and isolating submanifold-specific residuals. 

This decomposition is enforced explicitly through orthogonality constraints on residual submodules, ensuring that different branches capture geometrically independent directions. As a result, PoSFormer implements a structured \emph{unmixing} process, in contrast to the soft aggregation mechanism of standard attention.

The following analysis is framed within a linear subspace setting. Our practical implementation extends this into the nonlinear regime by enforcing submanifold separation in the feature domain, that is, after applying a nonlinear transformation to the input.
Given an input sequence $S \in \mathbb{R}^{T \times C}$ projected via local spatial depthwise convolutions, we map the features onto two distinct topological spaces. Let $\mathcal{M}_i, \mathcal{M}_j \in \mathbb{R}^{T \times C}$ denote these two distinct subspaces of the data manifold. To ensure non-negativity for flow normalization without the severe sparsity induced by Softmax, we apply a shifted exponential linear unit \cite{LinearAttention}:
\begin{equation}
\mathcal{M}_i = \text{ELU}(\text{Conv}_i(S)) + 1, \quad \mathcal{M}_j = \text{ELU}(\text{Conv}_j(S)) + 1
\end{equation}


For the dual feature projections $s_i \in \mathcal{M}_i$ and $s_j \in \mathcal{M}_j$ originating from the same input sequence, we isolate the tangential components belonging to their shared intersection subspace, denoted as $\mathcal{B} \approx \mathcal{M}_i \cap \mathcal{M}_j$. In a linear projection framework, the component of $s_i$ that is explainable by the shared intersection is found by projecting $s_i$ onto the basis of $s_j$, and vice versa. We approximate these shared intersection components ($s_{i,\mathcal{B}}$ and $s_{j,\mathcal{B}}$) via associative cross-attention:
\begin{equation}
s_{i,\mathcal{B}} = \frac{s_j(s_j^\top s_i)}{s_j(s_j^\top \mathbf{1})}, \quad s_{j,\mathcal{B}} = \frac{s_i(s_i^\top s_j)}{s_i(s_i^\top \mathbf{1})}
\label{eq:dba_intersection}
\end{equation}
where $\mathbf{1}$ is an all-ones vector for distributional normalization.
The corresponding components belonging to the residual submanifolds are then extracted via subtraction
\begin{equation}
s_{i,\mathcal{R}} \approx s_i - s_{i,\mathcal{B}}, \quad s_{j,\mathcal{R}} \approx s_j - s_{j,\mathcal{B}}
\label{eq:dba_residuals}
\end{equation}
These residuals encapsulate the disjoint, class-specific directions, corresponding to the orthogonal complements such that $s_{i,\mathcal{R}} \in \mathcal{M}_i \setminus \mathcal{M}_j$ and $s_{j,\mathcal{R}} \in \mathcal{M}_j \setminus \mathcal{M}_i$. 
To prevent the model from mapping inputs onto a collapsed joint plane, we apply a strict orthogonality constraint ($J_{orth}$) directly to these residuals:
\begin{equation}
J_{orth} = \mathbb{E} \left[ \left( \frac{\langle s_{i,\mathcal{R}}, s_{j,\mathcal{R}} \rangle}{\|s_{i,\mathcal{R}}\| \|s_{j,\mathcal{R}}\|} \right)^2 \right]
\label{eq:posformer_ortho_loss}
\end{equation}

Much like how structured degradation induces selective annihilation of residual nullspaces, this orthogonal penalty acts as an explicit geometric regularizer. It actively suppresses coefficient leakage between the dual branches, forcing the network to maintain geometrically independent bases for distinct fine-grained classes. This mathematically enforces the transversal decomposition requirement of the ideal network, ensuring $T_x\mathcal{M}_{D_i,R} \perp T_x\mathcal{M}_{D_j,R}$.

Finally, the orthogonally refined representations are filtered through a dynamic spatial gate $G(S) = \sigma(\text{Linear}(\text{DWConv}(S)))$ to ground the global flows in local manifold geometry before hierarchical integration
\begin{equation}
S^{\ell + 1} = \text{Norm}\Big(S^\ell + \text{FFN}\big([s^\ell_{i, \mathcal{R}} \odot G(S^\ell), \; s^\ell_{j, \mathcal{R}} \odot G(S^\ell)]\big)\Big)
\end{equation}

\begin{table}[h]
\centering
\caption{Performance comparison on CIFAR-10 and CIFAR-100 (trained from scratch).}
\label{tab:posformer_results}
\small
\begin{tabular}{lcccc}
\toprule
\textbf{Model} & \textbf{CIFAR-10 (\%)} & \textbf{CIFAR-100 Top-1 (\%)} & \textbf{CIFAR-100 Top-5 (\%)} \\ 
\midrule
 \cite{Cvt} (Baseline) & 92.50 & 69.35 & 89.09 \\
PoSFormer ($\lambda_o = 0.1$) & \textbf{93.40} & 71.61 & 90.11 \\
\textbf{PosFormer ($\lambda_o = 1.0$)} & \textbf{93.40} & \textbf{71.91} & \textbf{90.17} \\ 
\bottomrule
\end{tabular}
\end{table}

Figure~\ref{fig:posformer} illustrates the proposed architecture, where the DBA block describes Eq.~\ref{eq:dba_intersection}-\ref{eq:dba_residuals}. We compare this architecture with a vanilla transformer-based network replacing the gate-based, hierarchical framework depicted in Figure~\ref{fig:posformer}, top, with a standard self-attention mechanism. For a fair comparison, both architectures were designed with the same number of parameters ($\sim$2.6 M). We train each model using the AdamW optimizer with a base learning rate of $1\text{e-}3$ and a weight decay of $0.05$, modulated by a Cosine Annealing scheduler with warm restarts \cite{cosineannealing} ($T_0=50, T_{mult}=2$). The networks are trained for 150 epochs with a batch size of 128. We employ Cross-Entropy loss with a label smoothing factor of $0.1$ and apply gradient clipping (max norm $1.0$). Base data augmentations include Random Crop, Horizontal Flip, and RandAugment \cite{Randaugment} (2 ops, magnitude 9). Furthermore, at the batch level, we randomly apply either Mixup \cite{mixup} ($\alpha=0.2$), CutMix \cite{cutmix} ($\alpha=1.0$), or no blending, each with equal probability. The model checkpoint achieving the highest validation accuracy is used for final testing. The experiments were conducted using a single NVIDIA RTX 6000 GPU, requiring approximately 2 hours of training time per model.

Trained from scratch, the DBA achieved a test accuracy of 93.40$\%$ in CIFAR-10 \cite{Cifar}, compared to 92.50$\%$ accuracy of the vanilla transformer architecture. We also trained the models on the CIFAR-100 dataset. Table~\ref{tab:posformer_results} summarizes results. The ablation study on the effect of the orthogonality loss (Eq.~\ref{eq:posformer_ortho_loss}) is also performed by training the model with varying loss coefficients. As seen in the table, the PoS-based novel dual-branch attention architecture yields consistent improvement in training accuracy in both datasets.

\section{Discussion and Future Work}
\label{appendix:discussion}

The folding hypothesis formalizes deep networks as successively partitioning input space into exponentially many affine regions \cite{folding1, folding2}, later refined through combinatorial and geometric analyses \cite{folding3, folding4}. However, these perspectives remain inherently local and piecewise-linear, lacking a global geometric or axiomatic characterization.

A complementary geometric perspective interprets neural networks as sequences of coordinate transformations acting on a data manifold, with residual architectures corresponding to discretized dynamical systems and inducing layerwise transformations of the underlying Riemannian metric \cite{hauser2017principles}. While mathematically rich, such approaches remain focused on coordinate evolution rather than providing a unifying structural principle for learned representations.

In parallel, a growing body of work in neural population geometry studies the structure
of learned representations through the lens of high-dimensional geometry in both
artificial and biological systems \cite{kriegeskorte2021neural, chung2021neural, perich2025neural}.
These approaches characterize phenomena such as low-dimensional structure and clustering,
providing powerful empirical insights. However, they primarily analyze representations
post hoc, rather than deriving an underlying generative or axiomatic mechanism.

We introduced the PoS hypothesis as an axiomatic geometric framework for understanding the internal mechanisms of deep neural networks. By extending beyond the folding, the manifold hypothesis, and other geometric interpretations, PoS provides a unified geometric interpretation of deep networks. This perspective explains how networks organize, refine, and compose representations across layers, offering a mechanistic view of deep learning beyond purely empirical observations.

A central implication of the PoS framework is that many widely observed phenomena
in deep learning arise naturally from geometric principles. For instance, additive
Gaussian noise has been extensively studied in representation learning and
regularization, where it improves robustness, encourages smooth latent
representations, and enhances generalization through methods such as denoising
autoencoders, variational autoencoders, and stochastic latent regularization
techniques \cite{denoisedAutoencoders, chen2016variational}. It has also been
shown to improve adversarial robustness \cite{liu2018towards, cohen2019certified}.
In parallel, masking-based perturbation strategies, where portions of the input are
randomly removed, have been widely adopted to promote invariant and robust
representations \cite{he2022masked}. From the PoS perspective (Sections~\ref{Disentanglement} and~\ref{Orthogonality}),
these seemingly disparate techniques admit a unified geometric interpretation:
they act as mechanisms that induce controlled annihilation of residual subspaces
through hierarchical composition. Rather than viewing them as heuristic regularizers,
we interpret them as emergent consequences of the PoS axioms.

PoS provides a unifying geometric framework connecting the folding perspective,
the manifold hypothesis, and Riemannian viewpoints. In particular,
Remarks~\ref{Remark 1}--\ref{Remark 3} establish a structural invariance property:
once a network module implements a nonlinear orthogonal projection onto a union
of submanifolds, composing it with surrounding transformations that approximate
isometries preserves this form. Specifically, if a module realizes a projection operator $P_{\mathcal{M}}$ and
adjacent layers act as approximate isometries $g$, then the composed mapping
$g \circ P_{\mathcal{M}} \circ g^{-1}$ remains a projection onto the transformed
union $g(\mathcal{M})$. Consequently, in the ideal learning regime, the overall
network retains the structure of a nonlinear orthogonal projection onto a union
of submanifolds. This provides a mathematically grounded explanation for how hierarchical
composition preserves compact representations, and yields geometric intuition
for generalization, as illustrated in Figure~~\ref{fig:IsometryonSubmanifolds}.

Although the above invariance result is derived under idealized assumptions,
the toy examples provide concrete geometric insight into how these principles
manifest in practice. As illustrated in Figure~~\ref{fig:compactness_vs_ood}, when the underlying subspaces
are sufficiently separated (e.g., angles exceeding $90^\circ$), a ReLU-based
network can directly realize a compact representation by projecting onto a
union of submanifolds. In contrast, Figure~~\ref{fig:rotation_module}(a) shows that when this geometric
condition is violated, the same architecture behaves effectively as a linear
operator, collapsing the representation onto the joint span and failing to
achieve compactness. The constructions in Figure~~\ref{fig:rotation_module}(b--e) demonstrate that this limitation can be
systematically overcome through hierarchical composition. By introducing
learnable transformations that reorient the data via folding, the network
increases the effective separation between residual directions in a higher-
dimensional space. This enables subsequent nonlinear projections to recover
a union-of-submanifolds structure. These observations provide a geometric explanation of when ReLU networks succeed or fail, and show that depth acts as a mechanism for progressively constructing representations that satisfy the compactness requirements of the PoS framework. 

Additionally, this geometric perspective suggests a fundamental limitation
of standard ReLU-based architectures and raises an important open question.
Consider training a tied autoencoder with residual connections, where a ReLU
nonlinearity is applied in the hidden representation (bottleneck), on masked
inputs to learn a compact representation of a structured domain (e.g., images
of a single identity). What representation does the network converge to when
using conventional additive residual connections? How does this behavior
change if the skip connection is modified from summation $(+)$ to subtraction
$(-)$, aligning with the null-space removal principle? From the PoS viewpoint, these two designs correspond to fundamentally
different geometric operations,
and may lead to qualitatively different representation structures. We leave
a detailed theoretical and empirical analysis of this question as an important
direction for future work.

The postulates also suggest several directions for future investigation.
In particular, compact representations provide a natural explanation for both
the stability of in-distribution predictions and the failure modes observed
under out-of-distribution inputs. From the PoS perspective, such failures arise
when inputs are projected onto incorrect components of the learned union of
submanifolds, leading to ambiguous or misleading representations. This viewpoint also offers a potential geometric interpretation of the
\textbf{double-descent} phenomenon \cite{belkin2019reconciling}. As training progresses, internal modules may undergo
a transition from non-compact representations (approximating a joint span) to
compact representations (approximating a union of submanifolds). The
generalization behavior of the network may critically depend on this transition,
since the geometric mechanisms underlying invariance and transfer (as illustrated
in Figure~~\ref{fig:IsometryonSubmanifolds}) become effective only when the
representation aligns with a union-of-submanifolds structure rather than their
span (used here in a loose sense to denote the ambient manifold). This transition may also be related to the recently observed \textbf{grokking
phenomenon} \cite{power2022grokking}, where networks abruptly shift from
memorization to generalization during training, potentially reflecting
the emergence of compact representations.

Residual connections emerge as a necessary mechanism for modeling complementary (orthogonal) directions, enabling networks to progressively refine representations without destroying previously learned structure as explained in Figure \ref{fig:maskedLearning1}. Furthermore, disentanglement and class-specific organization arise as consequences of projection and transversality, rather than requiring explicit architectural constraints. We also discuss in Section~\ref{appendix:attention} that transformers' self attention can be interpreted as direct consequences of this phenomenon and we also introduce an improved variant, PoSFormer, as an application of the PoS perspective.

The PoS hypothesis suggests that, in classification settings, the decisive information is carried by residual components rather than the shared (intersection) structure. This is explicitly illustrated in Figure~~\ref{fig:intersection_residual}, where the intersection of submanifolds captures common structure, while the residual directions encode discriminative variations between classes. Consequently, classification naturally operates on these residual components. This observation is consistent with recent empirical findings by \cite{balestriero2024learning}, who show that reconstruction-based models primarily learn the top-variance subspace of the data, whereas features relevant for perception lie in the complementary low-variance subspace. From the PoS perspective, these two regimes admit a geometric interpretation: the top subspace corresponds to the intersection (shared structure), while the bottom subspace aligns with residual directions that separate submanifolds. We emphasize that this connection is not explicitly made in prior work; rather, it emerges naturally as a consequence of the PoS framework. A more detailed theoretical and empirical investigation of this relationship remains an important direction for future work.

An important direction is the connection between the PoS framework and generative models, particularly denoising diffusion models. In diffusion models, generation proceeds through a sequence of transformations that progressively remove noise from a corrupted signal. From the PoS perspective, this process admits a geometric interpretation consistent with Figure~~\ref{fig:DNNs}. The sequence of transformations can be viewed as a trajectory of the signal across a hierarchy of submanifolds, where each step maps the representation from one submanifold to another through structured operators. In the idealized regime described by Remarks~\ref{Remark 1}--\ref{Remark 3}, these transformations approximate isometries composed with nonlinear orthogonal projections onto unions of submanifolds. Consequently, diffusion dynamics can be interpreted as a progressive unfolding and re-alignment of the signal within a structured geometric space, rather than merely a stochastic denoising process. We believe this connection provides a promising foundation for a principled geometric understanding of generative modeling, which we leave for future investigation.

The PoS framework also suggests intriguing connections to theories in cognitive science and neuroscience, particularly the free-energy principle and the concept of Markov blankets . In these models, a system is partitioned into internal states, external states, and a Markov blanket composed of sensory and active states \cite{ororbia2023mortal}, which mediate all interactions between the system and its environment. From the PoS perspective, the central module (cf. Figure~~\ref{fig:DNNs}) naturally plays the role of internal states, while the surrounding transformations act analogously to sensory and action interfaces. In this sense, the capsule-like structure induced by PoS can be interpreted as a geometric realization of a Markov blanket, where information exchange is constrained through structured mappings between submanifolds.

The PoS framework also suggests an alternative interpretation of representation learning that does not explicitly rely on reconstruction. In particular, when the decoder is viewed as the adjoint of the encoder, projection can be interpreted as the removal of residual (null-space) components. From this perspective, learning can be understood as progressively suppressing these residual directions, driving representations toward compact submanifolds. This viewpoint hints at a connection between PoS and energy-based or inference-driven formulations of learning. In particular, the magnitude of residual components can be interpreted as a form of “energy” that measures deviation from the learned structure, drawing a conceptual link to the free-energy principle in neuroscience \cite{friston2005theory, friston2010free}. A full development of this perspective, including its implications for decoder-free architectures and biological learning models, is left for future work.

Several important questions also remain open. First, while the PoS postulates characterize the structure of ideal representations, understanding how such compact learning given in Postulate I can be guaranteed. We propose two alternative approach: masked or degraded representation learning (e.g., as in ECG anomaly detection study); or new loss functions, such as push-pull keeping degraded images projection away from the clean images projection while learning clean high quality image representation, i.e., non-linear orthogonal projection over union-of-submanifolds. Both have their own advantages and limitations, more advanced techniques are needed to be proposed in future works. Among the other topics, it is less investigated one. Second, extending the framework to broader classes of models, including state-space models, sequence models, and large-scale generative architectures, is an important direction. Establishing whether the projection-based interpretation holds across these settings may lead to a unified geometric theory of modern machine learning. Finally, the connection between PoS and symmetry, invariance, and group actions offers a promising avenue for future work. By explicitly modeling transformation structure, it may be possible to design architectures that more efficiently capture the geometry of data, leading to improved generalization, robustness, and interpretability.

Overall, the PoS framework provides a foundation for understanding deep learning systems through geometry and structure, and opens new directions for both theoretical analysis and architectural design.

\bibliographystyle{plain}
\bibliography{references}

\clearpage   
\appendix

\section{Notation}
\label{app:Notation}

In this work, we consider the $\ell_p$--norm of a vector 
$\mathbf{x} \in \mathbb{R}^n$, defined by
\(
\|\mathbf{x}\|_{p}
    = \left( \sum_{i=1}^n |x_i|^p \right)^{1/p}   
\) with $p \ge 1$. 
The $\ell_0$ ``norm'' is given by
\(
\|\mathbf{x}\|_{0}
    = \lim_{p \to 0} \sum_{i=1}^n |x_i|^p
\), which counts the number of nonzero entries in $\mathbf{x}$.  
We also use the $\ell_\infty$--norm,
\(
\|\mathbf{x}\|_{\infty}
    = \max_{1 \le i \le n} |x_i|.
\) A signal $\mathbf{s}$ is called \emph{strictly $k$-sparse} if it can be 
represented with at most $k$ nonzero coefficients in a basis 
$\boldsymbol{\Phi}$, that is,
\(\mathbf{s} = \boldsymbol{\Phi}\,\mathbf{x}\),
\(
\|\mathbf{x}\|_{0} \le k.
\) The \textbf{support} of $\mathbf{x}$ is the set
\(
\Lambda = \{ i \mid x_i \neq 0 \}
    \subset \{1,2,\ldots,n\},
\)
which identifies the indices of the nonzero components.  
Its complement with respect to $\{1,\dots,n\}$ is
\(
\Lambda^c = \{1,2,\ldots,n\} \setminus \Lambda.
\) The vector $\mathbf{x}_\Lambda \in \mathbb{R}^{|\Lambda|}$ denotes the 
restriction of $\mathbf{x}$ to the indices in $\Lambda$.
Similarly, for a matrix $\mathbf{M} \in \mathbb{R}^{m \times n}$, the submatrix 
$\mathbf{M}_\Lambda \in \mathbb{R}^{m \times |\Lambda|}$ is formed by selecting 
the columns indexed by $\Lambda$.

\section{Zero-Shot, Explainable Anomaly Detection: ECG Data and Preprocessing}
\label{ECGDataset}

We evaluate our method on the MIT-BIH Arrhythmia Database \cite{moody2001impact,goldberger2000physiobank}, which contains two-channel ECG recordings from 48 subjects, each approximately 30 minutes long with annotated heartbeat labels.

Each heartbeat is represented using a fixed-length signal of 128 samples obtained via resampling. Following standard practice, we consider two representations centered at the R-peak: (i) a \emph{single-beat} segment, extracted by locating adjacent R-peaks and selecting an inward window (10\% toward the center), and (ii) a \emph{beat-trio} segment, constructed by extending the window outward (10\%), thereby including neighboring beats to capture temporal morphology. R-peak locations are obtained from annotations; when unavailable, standard QRS detection methods may be used \cite{pantompkins, li1995detection, robustpeakkiranyaz}.

We adopt the AAMI standard \cite{AAMI}, grouping beats into five categories (N, V, S, F, Q), where N is treated as normal and all others as anomalous. Consistent with prior personalized ECG studies \cite{kiranyaz2017personalized}, we use 34 records, excluding patients with pacemakers or highly irregular signals.

For each target subject, training is performed using only normal beats from the first five minutes of recording, while abnormal beats are reserved for evaluation. To enrich the representation space, this subject-specific data is combined with normal beats from the remaining subjects.

\section{Zero-Shot, Explainable Anomaly Detection: Ablation Study}
\label{app:ablationECG}

\begin{table}[H]
    \centering
    \resizebox{0.5\textwidth}{!}{
    \begin{tabular}{l|c|c|c|c|c}
    \hline
    \rowcolor{gray!30} \textbf{\#} & \textbf{\# Atoms} & \textbf{Enc/Dec} & \textbf{Nonlinearity} & \textbf{Masked} & \textbf{F1} \\
    \hline
    1 & 10 & Tied & Linear & \ding{55} & 0.821 \\
    2 & 30 & Tied & Linear & \ding{55} & 0.787 \\
    3 & 10 & Tied & ReLU & \ding{55} & 0.836 \\
    4 & 30 & Tied & ReLU & \ding{55} & 0.860 \\
    5 & 10 & Tied & Linear & \ding{51} & 0.825 \\
    6 & 30 & Tied & Linear & \ding{51} & 0.806 \\
    7 & 10 & Tied & ReLU & \ding{51} & 0.838 \\
    8 & 30 & Tied & ReLU & \ding{51} & 0.861 \\
    9 & 10 & Untied & ReLU & \ding{55} & - \\
    10 & 30 & Untied & ReLU & \ding{55} & 0.856 \\
    11 & 10 & Untied & Linear & \ding{51} & - \\
    12 & 30 & Untied & Linear & \ding{51} & 0.832 \\
    13 & 10 & Untied & ReLU & \ding{51} & - \\
    14 & 30 & Untied & ReLU & \ding{51} & 0.872 \\
    \end{tabular}
    }
    \caption{Ablation study on compact personalized models for zero-shot anomaly detection.}
    \label{tab:ablationECG}
\end{table}

This section presents an ablation study of the personalized zero-shot anomaly detection framework introduced in Section~\ref{sec:ZeroShotECG}. Specifically, we analyze the effect of architectural and training design choices on learning the healthy submanifold $\mathcal{M}_{D_i}$ for each user. Each personalized model consists of a compact encoder–decoder network, where a linear encoder $f_E:\mathbb{R}^n \rightarrow \mathbb{R}^N$ maps input signals to a low-dimensional latent space, and a decoder $f_D:\mathbb{R}^N \rightarrow \mathbb{R}^n$ reconstructs the signal. We systematically vary the latent dimensionality ($N$), the use of tied versus untied encoder–decoder weights, the presence of nonlinear activation functions, and input masking. The results in Table~\ref{tab:ablationECG} demonstrate how these components influence the transition from simple subspace projections to more expressive union-of-subspaces representations, and their impact on anomaly detection performance.

As discussed in Figure~~\ref{fig:compactness_vs_ood}, when the network does not include nonlinear activation functions and operates in a low-dimensional regime ($N \ll n$, e.g., $N=10$, $n=128$), it effectively learns a projection onto a single linear subspace. To extend this representation from a single subspace (i.e., the span of directions) to a union of subspaces, nonlinear activations such as ReLU are required.

Furthermore, as discussed in the orthogonality section, applying input masking with varying window sizes encourages the separation of submanifold components and promotes a union-of-subspaces structure. However, when masking is applied in a highly compressed regime (e.g., $N=10$), no performance improvement is observed, and in some cases performance degrades. This can be explained by the increase in the effective dimensionality of the representation. Specifically, for a union of subspaces $\{\mathcal{S}_i\}$, define
\[
k_{\max} = \max_{i \neq j} \dim\bigl(\mathrm{span}(\mathcal{S}_i \cup \mathcal{S}_j)\bigr).
\]
In general, $k_{\max}$ exceeds the dimensionality of individual subspaces. As discussed in the Restricted Isometric Embedding framework, stable and unique representations require the embedding dimension to satisfy $N \geq k_{\max}$. When $N$ is too small, the model cannot faithfully represent the union structure, leading to degraded performance.

This explains the observed improvement when the latent dimension is increased (e.g., $N=30$). In this regime, the combination of nonlinear activation and masking enables the personalized mini-networks to learn a compact yet expressive representation of the subject-specific healthy manifold, resulting in improved anomaly detection performance.

This transition from a single subspace to a union of subspaces directly reflects the geometric decomposition discussed in the main text.

\section{Proofs}
\label{app:Proofs}
\subsection{Restricted Orthogonality Constant}
\label{app:roc}

\begin{definition}[Restricted Orthogonality Constant]
Let $\mathbf{D}_{\Lambda_i} \in \mathbb{R}^{n \times k_i}$ 
and $\mathbf{D}_{\Lambda_j} \in \mathbb{R}^{n \times k_j}$ 
be matrices whose columns span the 
$k_i$- and $k_j$-dimensional subspaces 
$\operatorname{span}(\mathbf{D}_{\Lambda_i})$ 
and 
$\operatorname{span}(\mathbf{D}_{\Lambda_j})$, respectively. The $(k_i,k_j)$-restricted orthogonality constant 
$\theta_{k_i,k_j}$ is defined as the smallest quantity such that
\[
\big|
\langle 
\mathbf{D}_{\Lambda_i}\mathbf{x}, 
\mathbf{D}_{\Lambda_j}\mathbf{y}
\rangle
\big|
\;\le\;
\theta_{k_i,k_j}
\,
\|\mathbf{x}\|_2
\|\mathbf{y}\|_2
\]
for all $\mathbf{x}\in\mathbb{R}^{k_i}$ and 
$\mathbf{y}\in\mathbb{R}^{k_j}$.
\end{definition}

\subsection{Proof of the Bounds on Inter-Subspace Interference}
\label{app:crossproof}

The following argument follows the standard adjoint manipulation and restricted-orthogonality/RIP bounds used in the dual-certificate analysis for sparse recovery; see, e.g., \cite{candes2005decoding}.
Let the projected signal be defined as
\(
\tilde{\mathbf{s}}
=
\mathbf{D}_{\Lambda_i}
(\mathbf{D}_{\Lambda_i}^{\top}\mathbf{D}_{\Lambda_i})^{-1}
\mathbf{x}^*_{\Lambda_i},
\)
where \(\mathbf{x}^*_{\Lambda_i}\in\mathbb{R}^{k_i}\) is the coefficient vector. It follows directly that
\(
\mathbf{D}_{\Lambda_i}^{\top}\tilde{\mathbf{s}}
=
\mathbf{D}_{\Lambda_i}^{\top}
\mathbf{D}_{\Lambda_i}
(\mathbf{D}_{\Lambda_i}^{\top}\mathbf{D}_{\Lambda_i})^{-1}
\mathbf{x}^*_{\Lambda_i}
=
\mathbf{x}^*_{\Lambda_i}.
\) Let \(\Lambda_j\) be a disjoint index set from \(\Lambda_i\) and let 
\(\mathbf{x}_{\Lambda_j}\in\mathbb{R}^{k_j}\).
We seek to bound the interaction term
\(
\left|
\left\langle
\mathbf{D}_{\Lambda_j}^{\top}\tilde{\mathbf{s}},
\mathbf{x}_{\Lambda_j}
\right\rangle
\right|.
\) Using the definition of the adjoint operator,
\(
\left\langle
\mathbf{D}_{\Lambda_j}^{\top}\tilde{\mathbf{s}},
\mathbf{x}_{\Lambda_j}
\right\rangle
=
\left\langle
\tilde{\mathbf{s}},
\mathbf{D}_{\Lambda_j}\mathbf{x}_{\Lambda_j}
\right\rangle.
\) Substituting the expression for \(\tilde{\mathbf{s}}\) yields
\[
\left\langle
\tilde{\mathbf{s}},
\mathbf{D}_{\Lambda_j}\mathbf{x}_{\Lambda_j}
\right\rangle
=
\left\langle
\mathbf{D}_{\Lambda_i}
(\mathbf{D}_{\Lambda_i}^{\top}\mathbf{D}_{\Lambda_i})^{-1}
\mathbf{x}^*_{\Lambda_i},
\mathbf{D}_{\Lambda_j}\mathbf{x}_{\Lambda_j}
\right\rangle.
\]

Define
\(
\mathbf{a}
=
(\mathbf{D}_{\Lambda_i}^{\top}\mathbf{D}_{\Lambda_i})^{-1}
\mathbf{x}^*_{\Lambda_i},
\quad
\mathbf{b}
=
\mathbf{x}_{\Lambda_j}.
\) Then the previous expression becomes
\(
\left|
\left\langle
\mathbf{D}_{\Lambda_i}\mathbf{a},
\mathbf{D}_{\Lambda_j}\mathbf{b}
\right\rangle
\right|.
\) By the definition of the restricted orthogonality constant
(Definition~5), we have
\(
\left|
\left\langle
\mathbf{D}_{\Lambda_i}\mathbf{a},
\mathbf{D}_{\Lambda_j}\mathbf{b}
\right\rangle
\right|
\le
\theta_{k_i,k_j}
\|\mathbf{a}\|_2
\|\mathbf{b}\|_2 .
\)Substituting the definitions of \(\mathbf{a}\) and \(\mathbf{b}\) gives
\[
\left|
\left\langle
\mathbf{D}_{\Lambda_j}^{\top}\tilde{\mathbf{s}},
\mathbf{x}_{\Lambda_j}
\right\rangle
\right|
\le
\theta_{k_i,k_j}
\|
(\mathbf{D}_{\Lambda_i}^{\top}\mathbf{D}_{\Lambda_i})^{-1}
\mathbf{x}^*_{\Lambda_i}
\|_2
\|\mathbf{x}_{\Lambda_j}\|_2 .
\]

From the RIP bound on the Gram matrix,
\(
\|(\mathbf{D}_{\Lambda_i}^{\top}\mathbf{D}_{\Lambda_i})^{-1}\|
\le
(1-\delta_{k_i})^{-1},
\)
which implies
\(
\|
(\mathbf{D}_{\Lambda_i}^{\top}\mathbf{D}_{\Lambda_i})^{-1}
\mathbf{x}^*_{\Lambda_i}
\|_2
\le
\frac{1}{(1-\delta_{k_i})}
\|\mathbf{x}^*_{\Lambda_i}\|_2 .
\) Therefore
\[
\left|
\left\langle
\mathbf{D}_{\Lambda_j}^{\top}\tilde{\mathbf{s}},
\mathbf{x}_{\Lambda_j}
\right\rangle
\right|
\le
\frac{\theta_{k_i,k_j}}{1-\delta_{k_i}}
\|\mathbf{x}^*_{\Lambda_i}\|_2
\|\mathbf{x}_{\Lambda_j}\|_2 .
\]

Since $\mathbf{x}_{\Lambda_j}\in\mathbb{R}^{k_j}$ is arbitrary, we may invoke the variational characterization of the $\ell_2$ norm,
\(
\|\mathbf{u}\|_2=\sup_{\mathbf{v}\neq 0}\frac{|\langle \mathbf{u},\mathbf{v}\rangle|}{\|\mathbf{v}\|_2}.
\)
Letting $\mathbf{u}=\mathbf{D}_{\Lambda_j}^{\top}\tilde{\mathbf{s}}$ and $\mathbf{v}=\mathbf{x}_{\Lambda_j}$, and using the bound established above,
\[
\left|\left\langle \mathbf{D}_{\Lambda_j}^{\top}\tilde{\mathbf{s}},\,\mathbf{x}_{\Lambda_j}\right\rangle\right|
\le
\frac{\theta_{k_i,k_j}}{1-\delta_{k_i}}\,
\|\mathbf{x}^*_{\Lambda_i}\|_2\,
\|\mathbf{x}_{\Lambda_j}\|_2,
\]
we obtain
\[
\|\mathbf{D}_{\Lambda_j}^{\top}\tilde{\mathbf{s}}\|_2
=
\sup_{\mathbf{x}_{\Lambda_j}\neq 0}
\frac{\left|\left\langle \mathbf{D}_{\Lambda_j}^{\top}\tilde{\mathbf{s}},\,\mathbf{x}_{\Lambda_j}\right\rangle\right|}{\|\mathbf{x}_{\Lambda_j}\|_2}
\le
\frac{\theta_{k_i,k_j}}{1-\delta_{k_i}}\,
\|\mathbf{x}^*_{\Lambda_i}\|_2.
\]
This yields the desired $\ell_2$ bound on the cross-residual energy.

To bound the second term, note that $\mathbf{F}_{\Lambda_i,r}^{\top}\mathbf{c}_r\in 
\operatorname{span}(\mathbf{D}_{\Lambda_i})^{\perp}$ while 
$\mathbf{D}_{\Lambda_j}\mathbf{x}_{\Lambda_j}\in 
\operatorname{span}(\mathbf{D}_{\Lambda_j})$. 
Let $\alpha_{\min}$ denote the smallest principal angle between 
$\operatorname{span}(\mathbf{D}_{\Lambda_i})$ and 
$\operatorname{span}(\mathbf{D}_{\Lambda_j})$. 
By Definition~5, the restricted orthogonality constant satisfies
\[
\theta_{k_i,k_j}
=
\sup_{x,y}
\frac{|\langle \mathbf{D}_{\Lambda_i}x,\mathbf{D}_{\Lambda_j}y\rangle|}
{\|x\|_2\|y\|_2}
\le
\cos(\alpha_{\min}),
\]
and therefore
\[
\sin(\alpha_{\min}) \le \sqrt{1-\theta_{k_i,k_j}^2}.
\]
Using the variational characterization and the principal-angle bound for cross projections, we obtain
\[
\|\mathbf{D}_{\Lambda_j}^{\top}\mathbf{F}_{\Lambda_i,r}^{\top}\mathbf{c}_r\|_2
\le
\sin(\alpha_{\min})\,\|\mathbf{c}_r\|_2
\le
\sqrt{1-\theta_{k_i,k_j}^2}\,\|\mathbf{c}_r\|_2 .
\]
Combining this with the previous estimate on 
$\|\mathbf{D}_{\Lambda_j}^{\top}\widetilde{\mathbf{s}}\|_2$ yields
\[
\|\widetilde{\mathbf{x}}_{\Lambda_j}\|_2
=
\|\mathbf{D}_{\Lambda_j}^{\top}\mathbf{s}\|_2
\le
\frac{\theta_{k_i,k_j}}{1-\delta_{k_i}}\|\mathbf{x}^*_{\Lambda_i}\|_2
+
\sqrt{1-\theta_{k_i,k_j}^2}\,\|\mathbf{c}_r\|_2 .
\]

\subsection{Group Actions and Orbits}
\label{app:group_action}

For completeness, we briefly recall the notions of group actions and orbits used in the main text. Standard references include \cite{lee2000introduction, nakahara2018geometry}.

\paragraph{Group:}
A group $(G,\circ)$ is a set $G$ equipped with a binary operation $\circ : G\times G\rightarrow G$ satisfying the following properties:
(i) associativity $(g_1\circ g_2)\circ g_3=g_1\circ(g_2\circ g_3)$,
(ii) existence of an identity element $e\in G$ such that $e\circ g=g\circ e=g$ for all $g\in G$, and
(iii) existence of an inverse $g^{-1}\in G$ for every $g\in G$ such that $g\circ g^{-1}=g^{-1}\circ g=e$.

\paragraph{Group action.}
Let $G$ be a group and $X$ a set. A (left) group action of $G$ on $X$ is a map
\[
G\times X \to X, \qquad (g,x)\mapsto g\cdot x,
\]
such that
\[
e\cdot x = x, \qquad
(g_1 g_2)\cdot x = g_1\cdot (g_2\cdot x)
\]
for all $g_1,g_2\in G$ and $x\in X$.

\paragraph{Orbit.}
Given $x\in X$, the orbit of $x$ under the group action is
\[
G\cdot x = \{\, g\cdot x \mid g\in G \,\}.
\]
More generally, if $\mathcal{M}\subset X$ is a subset, the orbit of $\mathcal{M}$ is
\[
G\cdot \mathcal{M} = \bigcup_{g\in G} g(\mathcal{M}).
\]

In the context of this work, the representation space is denoted by $X$, which may correspond to the input space or an intermediate latent space within the network. The group $G$ consists of learnable transformations (i.e., isometries)  acting on $X$. The orbit $G\cdot\mathcal{M}$ therefore represents the family of manifolds obtained by applying these transformations to a canonical manifold (or union of manifolds) $\mathcal{M}\subset X$.

\subsection{Locally Trivial Fibrations}
\label{localfibration}

For completeness, we briefly recall the notion of a locally trivial fibration used in the main text. Let $E$ and $B$ be smooth manifolds and let $\pi : E \rightarrow B$ be a smooth map. We say that $\pi$ defines a \emph{locally trivial fibration} with fiber $F$ if for every point $b \in B$ there exists an open neighborhood $U \subset B$ such that
\[
\pi^{-1}(U) \cong U \times F,
\]
where $\cong$ denotes a diffeomorphism compatible with the projection onto $U$. Intuitively, this means that although the global space $E$ may have a complicated structure, it locally decomposes into a product between a base coordinate $U$ and a fiber $F$.

A simple example is the torus $T^2$, which can be written as the product $S^1 \times S^1$. Viewing the projection $\pi : T^2 \rightarrow S^1$ onto the first circle, the preimage of any open arc $U \subset S^1$ satisfies
\[
\pi^{-1}(U) \cong U \times S^1,
\]
so locally the torus looks like a cylinder. In this case the bundle is trivial both locally and globally.

More intricate fibrations arise when the product structure holds only locally. A classical example is the Hopf fibration $S^3 \rightarrow S^2$, where the fibers are circles $S^1$. Around any small neighborhood $U \subset S^2$, the preimage satisfies $\pi^{-1}(U) \cong U \times S^1$, but globally the total space $S^3$ is not the product $S^2 \times S^1$. Such structures appear naturally in geometry and physics; see \cite[Sec.~1.3]{dundas2018short} for a detailed discussion and examples involving quantum state spaces.

\subsection{Sample Complexity of Hierarchical Manifold Learning}
\label{appendix:samplecomplexity}

Our proof is a natural extension of the manifold sampling theory of Niyogi et al.~\cite{niyogi2008finding}, which provides bounds on the number of samples required to cover a compact submanifold with high probability. Let $\mathcal M_{D_i}\subset\mathbb R^n$ be a compact $k$-dimensional submanifold and let $\mathcal M=\bigcup_i\mathcal M_{D_i}$ denote the union of learned submanifolds introduced in the main text.

Following Niyogi et al.~\cite{niyogi2008finding}, the number of samples required to recover the geometry of a compact $k$-dimensional manifold depends on three factors: the intrinsic dimension $k$, the volume of the manifold, and a geometric regularity parameter known as the condition number.

A sample set $\{x_j\}$ is said to be $\epsilon$-dense in $\mathcal M_{D_i}$ if every point $p\in\mathcal M_{D_i}$ lies within distance $\epsilon$ of at least one sample, i.e.,
\[
\forall p\in\mathcal M_{D_i}, \quad \exists\, x_j \;\text{such that}\; \|p-x_j\|<\epsilon .
\]
Intuitively, this means the samples cover the manifold at spatial resolution $\epsilon$.

The geometric regularity of the manifold is characterized by its condition number $1/\tau$, where $\tau$ is the \emph{reach} of the manifold. The reach is defined as the largest radius $\tau$ for which the tubular neighborhood
\[
\{x \in \mathbb{R}^n : d(x,\mathcal M_{D_i}) < \tau\}
\]
admits a unique nearest-point projection onto $\mathcal M_{D_i}$. Geometrically, $\tau$ bounds both the curvature of the manifold and the minimum separation between distinct components.

In our framework, Postulate~II assumes that the nonlinear projection operator is well-defined within a neighborhood of radius $\gamma$. This projection radius must lie inside the reach of the manifold, so that
\[
\epsilon < \gamma \le \tau ,
\]
ensuring both stable projection and sufficiently dense sampling.

Under these conditions, the $\epsilon$-covering number of $\mathcal M_{D_i}$ satisfies
\[
C_\epsilon(\mathcal M_{D_i})
=
O\!\left(
\frac{\mathrm{vol}(\mathcal M_{D_i})}
{\cos^k(\theta)\,\mathrm{vol}(B_\epsilon^k)}
\right),
\]
where $\theta=\arcsin(\epsilon/8\tau)$ and $B_\epsilon^k$ denotes the $k$-dimensional Euclidean ball of radius $\epsilon$.

In particular, the dominant scaling with resolution is
\[
C_\epsilon(\mathcal M_{D_i}) = O(\epsilon^{-k}).
\]

For the union of manifolds $\mathcal M=\bigcup_i\mathcal M_{D_i}$, the covering number satisfies the general bound

\[
C_\epsilon(\mathcal M)
\le
\sum_i C_\epsilon(\mathcal M_{D_i}).
\]

In our framework the manifold components may intersect, as transformations can map points between components through transversal directions. In this case intersection regions are shared between manifolds and therefore covered only once. Consequently, the summation above provides a valid (and generally loose) upper bound on the covering number of the union.

Now consider a family of learnable transformations $G_0$ acting on the representation space. As described in the main text, the action generates the orbit
\[
G_0\!\cdot\!\mathcal M = \bigcup_{g\in G_0} g(\mathcal M),
\]
which produces new manifold components. Classical learning methods that do not exploit this structure must sample each transformed manifold independently. If $|G_0|$ denotes the number of transformations, the required number of samples scales as
\[
N_{\mathrm{classical}} \sim |G_0|\,C_\epsilon(\mathcal M).
\]

More generally, when successive transformation families 
$G_0,G_1,\ldots,G_L$ act recursively as described in Sec.~E, the generated
manifold family becomes

\[
\mathcal{M}^{(L)} =
G_L\!\cdot\!\Big(G_{L-1}\!\cdot\!\big(\cdots(G_0\!\cdot\!\mathcal{M})\cdots\big)\Big).
\]

If $|\mathcal{M}|$ denotes the number of canonical components
$\mathcal{M}_{D_i}$, the total number of generated manifolds satisfies

\[
|\mathcal{M}^{(L)}|
=
|\mathcal{M}|
\prod_{\ell=0}^{L}|G_\ell|, 
\]

where \(|G_\ell|\) is the number of elements in  \(G_\ell\). Therefore classical non-hierarchical learning requires sampling all generated components, leading to the multiplicative complexity
\[
N_{\mathrm{classical}} \sim C_\epsilon(\mathcal M)\prod_{\ell=0}^{L}|G_\ell|.
\]

In contrast, the hierarchical structure induced by the deep architecture allows transformations to be learned incrementally. At each level $\ell$, a transformation family $G_\ell$ generates new components from the canonical manifolds $\mathcal M_{D_i}$. By the equivariance property described in Remark~4, it suffices to learn the action of each transformation $g \in G_\ell$ on a single representative manifold $\mathcal M_{D_i}$, which then generalizes to all components without requiring independent resampling. This assumes that the transformation family acts transitively across the manifold components. For example, in vision tasks, $G_\ell$ may represent pose, illumination, or deformation transformations, where a transformation learned on one object instance transfers across all instances.

To unify discrete and continuous transformation families, we measure the complexity of $G_\ell$ through its cardinality in the finite case and through its covering number when $G_\ell$ is a continuous transformation manifold. If $G_\ell$ is finite, the number of transformations is $|G_\ell|$, and the number of samples required at level $\ell$ scales as $|G_\ell|\,C_\epsilon(\mathcal M_{D_i})$. If $G_\ell$ is a Lie group of dimension $d_{G_\ell}$, its effective complexity is governed by its covering number, which scales as $C_\epsilon(G_\ell)=O(\epsilon^{-d_{G_\ell}})$. In this case, the sampling requirement at level $\ell$ scales as $C_\epsilon(G_\ell)\,C_\epsilon(\mathcal M_{D_i})$.

Thus the sampling complexity decomposes into intrinsic manifold sampling and transformation sampling at each level. Consequently, the total number of samples required by the hierarchical representation satisfies
\[
N_{\mathrm{DNN}} 
\sim 
C_\epsilon(\mathcal M) 
+ 
\sum_{\ell=0}^{L} 
C_\epsilon(G_\ell)\,C_\epsilon(\mathcal M_{D_i}),
\]
where $C_\epsilon(G_\ell)$ should be interpreted as $|G_\ell|$ in the finite case.

Thus, while classical learning must sample every transformed manifold independently, resulting in multiplicative growth in sample complexity, the hierarchical representation separates intrinsic manifold geometry from transformation coordinates, leading to additive growth across layers.

\section{Proof of Lemma~\ref{lemma:ResidualProj}}
\label{app:ResidualProjectionProof}

Let $x \in \mathcal M_B \cap \mathcal M$ be a transverse intersection point close to $s$. 
Define
\(
x_B = \mathcal{P}_{\mathcal M_B}(s),
\qquad
x_R = \mathcal{P}_{\mathcal M_R}(s),
\)
as the nonlinear projections of $s$ onto the base component and the residual component, respectively. 
Let $P_{T_x\mathcal M_B}$ and $P_{T_x\mathcal M_R}$ denote linear projections onto the corresponding tangent subspaces. For $\mathcal{P}_{\mathcal M}(s)$ sufficiently close to $x$, we have the local linearizations
\[
x_B \approx x + P_{T_x\mathcal M_B} v,
\qquad
x_R \approx x + P_{T_x\mathcal M_R} v,
\]
where $v := \mathcal{P}_{\mathcal M}(s) - x$ is the displacement vector.

Decomposing $v$ with respect to $\mathcal M_B$,
\[
v = P_{T_x\mathcal M_B} v + (I - P_{T_x\mathcal M_B})v
= (x_B - x) + P_{N_x\mathcal M_B} v,
\]
where $P_{N_x\mathcal M_B}$ denotes projection onto the normal space of $\mathcal M_B$. Substituting into the residual projection,
\[
x_R \approx x + P_{T_x\mathcal M_R}\bigl((x_B - x) + P_{N_x\mathcal M_B} v\bigr).
\]

Since $x, x_R$ lie in the residual direction, we have
\[
P_{T_x\mathcal M_R}(x_R - x) \approx x_R - x,
\]
and therefore
\[
P_{T_x\mathcal M_R}(x_R - x_B)
\approx
P_{T_x\mathcal M_R} P_{N_x\mathcal M_B} v.
\]

Because $x_B$ lies close to the residual component, the left-hand side approximates the difference between projections:
\[
x_R - \mathcal{P}_{\mathcal M_R}(x_B)
\approx
P_{T_x\mathcal M_R} P_{N_x\mathcal M_B} v.
\]

By transversality, the restriction of $P_{T_x\mathcal M_R}$ to $N_x\mathcal M_B$ is invertible. 
Define
\[
A : N_x\mathcal M_B \to T_x\mathcal M_R,
\qquad
A(u) = P_{T_x\mathcal M_R} u.
\]
Then $A$ is an isomorphism onto its image, and we can solve
\[
P_{N_x\mathcal M_B} v
\approx
A^{-1}\bigl(x_R - \mathcal{P}_{\mathcal M_R}(x_B)\bigr).
\]

Finally, reconstructing the full projection,
\[
\mathcal{P}_{\mathcal M}(s)
= x + v
= x_B + P_{N_x\mathcal M_B} v.
\]

Substituting $x_B = \mathcal{P}_{\mathcal M_B}(s)$ and $x_R = \mathcal{P}_{\mathcal M_R}(s)$ yields
\[
\mathcal{P}_{\mathcal M}(s)
\approx
\mathcal{P}_{\mathcal M_B}(s)
+
\Phi\bigl(
\mathcal{P}_{\mathcal M_R}(s)
-
\mathcal{P}_{\mathcal M_R}(\mathcal{P}_{\mathcal M_B}(s))
\bigr),
\]
where $\Phi$ is the nonlinear mapping locally corresponding to $A^{-1}$.

\section{Asset Licenses and Terms of Use}
\label{app:Licence}

All datasets and models utilized in this work are open-source and properly credited. The CIFAR-10 and CIFAR-100 datasets \cite{Cifar} are distributed under the MIT License. The pre-trained Masked Autoencoder for microscopy \cite{MAEMicroscopy} is released under the MIT License. The FVCD volumetric microscopy dataset and baseline reconstruction models \cite{FVCD} are provided under the Creative Commons Attribution 4.0 International License (CC BY 4.0). The ECG datasets \cite{ECGDataset1} utilize standard open-access medical databases (e.g., PhysioNet) and are used in accordance with the Open Data Commons Attribution License (ODC-BY).


\newpage

\end{document}